\documentclass[a4paper,11pt]{article}
\PassOptionsToPackage{capitalize}{cleveref}
\usepackage{arXiv/arxiv_custom_commands}
\usepackage[sort,square,numbers,longnamesfirst]{natbib}
\usepackage[symbol]{footmisc}

\allowdisplaybreaks

\usepackage{titling}
\title{Faster Adaptive Optimization via Expected Gradient Outer Product Reparameterization}

\usepackage{authblk}
\author[1,2,*]{Adela DePavia}
\author[3,$\dag$]{Jose Cruzado}
\author[2,$\dag$]{Jiayou Liang}
\author[4,5]{Vasileios Charisopoulos}
\author[1,2,3,5,6]{Rebecca Willett}
\affil[1]{Committee on Computational and Applied Mathematics, University of Chicago}
\affil[2]{Data Science Institute, University of Chicago}
\affil[3]{Department of Statistics, University of Chicago}
\affil[4]{Department of Electrical \& Computer Engineering, University of Washington}
\affil[5]{NSF-Simons National Institute for Theory and Mathematics in Biology}
\affil[6]{Department of Computer Science, University of Chicago}
\affil[*]{Corresponding author: adepavia@uchicago.edu}
\affil[$\dag$]{Denotes equal contribution}

\usepackage{algorithm}
\usepackage{algpseudocode}
\algrenewcommand\algorithmiccomment[1]{\texttt{\textbackslash\textbackslash\ }#1}
\usepackage[margin=1.1in]{geometry}

\makeatletter
\newcommand{\leqnomode}{\tagsleft@true}
\newcommand{\reqnomode}{\tagsleft@false}
\makeatother

\begin{document}
\maketitle

\begin{abstract}
Adaptive gradient methods---such as Adagrad, Adam, and their variants---have found wide\-spread use in machine learning, signal processing and many other settings. Several methods in this family are not rotationally equivariant, meaning that simple reparameterizations  (i.e. change of basis) can drastically affect their convergence. However, their sensitivity to the choice of parameterization has not been systematically studied;  it is not clear how to identify a ``favorable'' change of basis in which these methods perform best. In this paper we propose a reparameterization method and demonstrate both theoretically and empirically its potential to improve the convergence of adaptive algorithms. Our method is an orthonormal transformation based on the \textit{expected gradient outer product} (EGOP) matrix, which can be approximated using either full-batch or stochastic gradient oracles. Our theoretical results show that in the neighborhoods of local minima, the sensitivity of adaptive algorithms to choice-of-basis is influenced by spectral decay in the EGOP matrix. We illustrate the potential impact of EGOP reparameterization by presenting empirical evidence and theoretical arguments that common machine learning tasks with ``natural'' data exhibit strong EGOP spectral decay.
\end{abstract}

\section{Introduction}\label{sec:intro}

Adaptive gradient methods are popular optimization algorithms in modern machine learning \cite{Crew_2020}. Optimizers in this family include the seminal Adagrad and Adam algorithms as well as variants such as AdamW, Adadelta, and Adamax \cite{kingma2017adammethodstochasticoptimization, loshchilov2019decoupledweightdecayregularization,reddi2019convergenceadam, zeiler2012adadeltaadaptivelearningrate}. These algorithms are termed ``adaptive'' because they maintain and update different learning rates for each coordinate in parameter space.\footnote{We focus on adaptive algorithms with diagonal preconditioners, corresponding to the standard implementation of Adagrad, Adam, and variants in popular software packages. Full-matrix Adagrad, proposed by Duchi et al. \cite{duchi2011adaptive}, is equivariant but computes a matrix root on every iteration and is thus rarely used in practice.} Despite their popularity, fully understanding the impact these adaptive learning rates have on convergence remains an area of ongoing research \cite{jiang2024convergence, liu2024adagrad, maes2024understanding}. Notably, coordinate-wise learning rates make these methods sensitive to orthonormal reparameterization, distinguishing them from equivariant methods like gradient descent, whose performance does not change under orthonormal reparameterization.

Orthonormal reparameterizations correspond to full-dimensional changes of basis, and include seemingly benign transformations such as rotations of the loss landscape about the origin. The sensitivity of adaptive algorithms to rotation means that changes of basis can affect the rate of convergence to local minima, and even impact the generalization properties of the obtained solutions in non-convex settings. Given the sensitivity of these ubiquitous optimization methods to choice of basis, we pose the research question: 
\begin{quote}
    \vspace{.075cm}
    \centering \textbf{\emph{When using an adaptive algorithm, in what settings and to what extent can change-of-basis improve optimization?}} 
    \vspace{.075cm}
\end{quote}

We address this question by identifying geometric properties of loss functions that govern the sensitivity of adaptive algorithms to change-of-basis. We propose a reparameterization procedure based on the \emph{expected gradient outer product} (EGOP) matrix and show that this reparameterization can improve the convergence of adaptive methods. The geometric properties identified in this work---namely, strong decay of the EGOP eigenvalues---have been observed in a variety of machine learning objectives \cite{ papyan2018full,sagun2017empirical, zhang2024transformers}. We include both empirical evidence and theoretical arguments suggesting  this property commonly arises when using natural data.

\paragraph{Contributions}
We show that for a large class of objectives, the proposed reparameterization procedure can improve the convergence of adaptive optimization algorithms. Our analysis of EGOP spectral decay and low-rank structure yields a novel and specific hypothesis for why adaptive gradient methods are particularly sensitive to changes of basis in machine learning settings. Our main contributions are as follows: \textbf{(1)} We characterize a class of objective functions for which EGOP reparameterization can reduce the number of iterations required for adaptive algorithms to converge to minima and first-order stationary points. \textbf{(2)} For these functions, we identify a choice of basis in which adaptive algorithms will perform well, and we propose an approximation procedure that only requires access to a (stochastic) gradient oracle, rather than analytical knowledge about the loss function. This procedure is defined in Section~\ref{sec:EGOP-defn}. \textbf{(3)} We develop theory that proves that in neighborhoods of local minima, the proposed reparameterization endows adaptive algorithms with improved convergence guarantees, quantified in terms of the spectrum of the EGOP matrix. Our main results are discussed in Section~\ref{sec:convergence-analysis}. \textbf{(4)} We empirically investigate this procedure and find that the proposed reparameterization improves the convergence of adaptive algorithms on a diverse suite of optimization problems. These experiments include results with large-scale contemporary neural network architectures. We present empirical results in Section~\ref{sec:experimental-results}.

\subsection{Related Work}\label{ssec:related-work}

\begin{figure}
    \centering
    \begin{subfigure}[c]{0.58\textwidth}
        \centering
        \includegraphics[width=\linewidth]{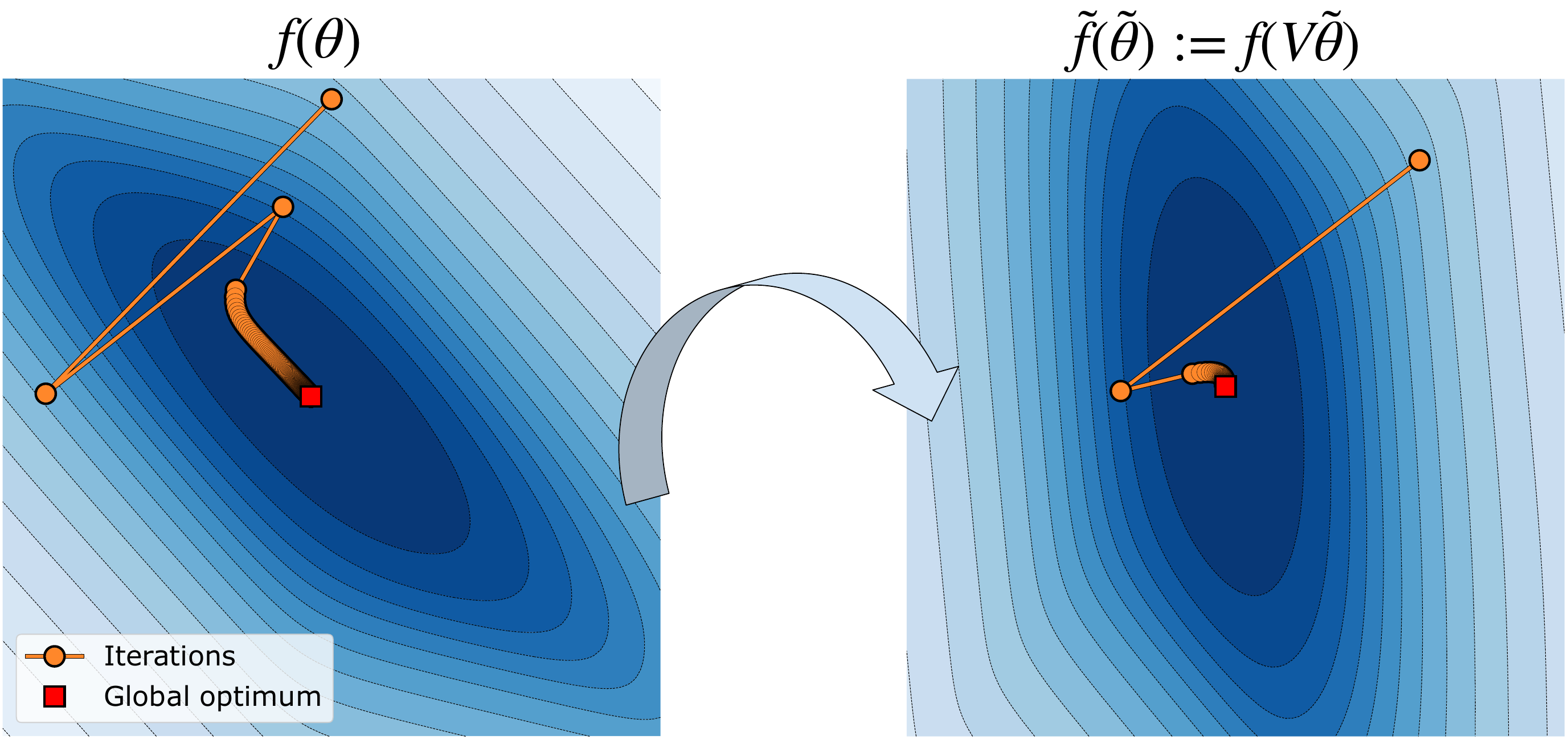}
    \end{subfigure}
    \hspace{0.01\textwidth} 
    \begin{subfigure}[c]{0.38\textwidth} 
        \centering
        \includegraphics[width=\linewidth]{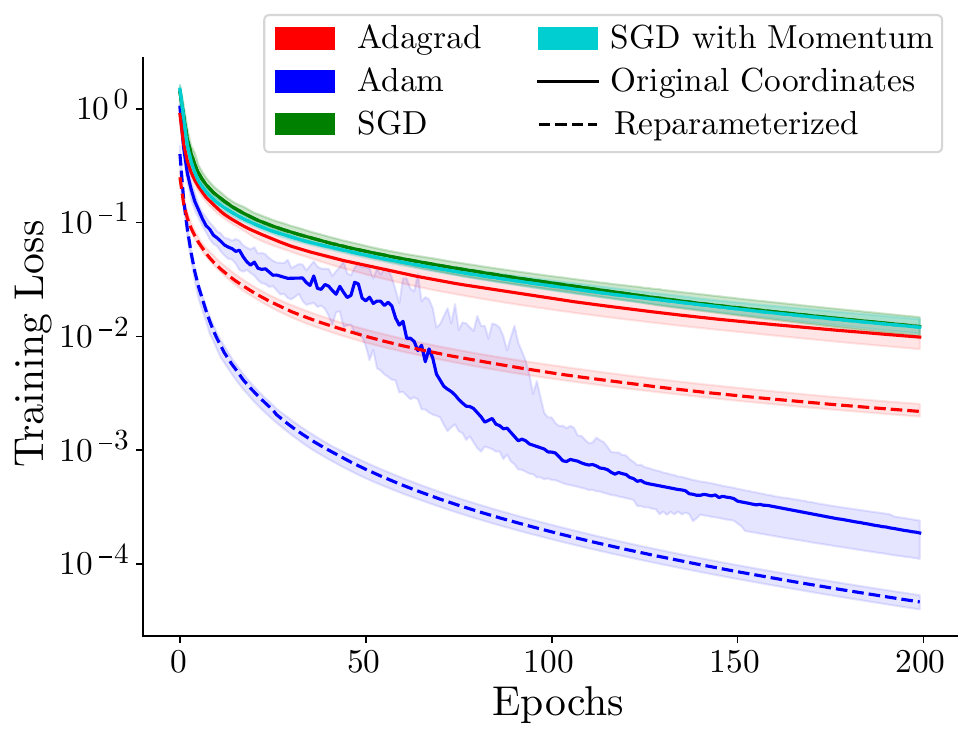}
    \end{subfigure}
    \caption{(Left) Using Adagrad in original coordinates and under EGOP reprameterization to optimize a two-dimensional log-sum-exp objective. The iterates produced by Adagrad are \textit{not} equivalent up to rotation, despite equivalent initializations. In the EGOP eigenbasis, the primary directions of function variation are axis-aligned. Details in Section~\ref{ssec:details-for-opener-cartoon}. (Right) Cross-entropy loss from a 2-layer ReLU network in 2.4k dimensions trained on image classification using Adam, Adagrad, SGD, and SGD with momentum, in both original coordinates and under reparameterization. Equivariant methods (e.g. SGD) exhibit no change under reparameterization. Discussion and details in Section~\ref{sec:experimental-results}.}\label{fig:opener-cartoon}
\end{figure}

Our work intersects with research on guarantees for adaptive algorithms, the role of orthogonal rotations in algorithmic performance, and geometry of machine learning objectives with natural data. Here we concisely survey related works, and we include an expanded discussion in Section~\ref{sec:extended-related-works}. 

\paragraph{Geometric Sensitivity of Adaptive Methods}
Recently, there has been renewed interest in distinguishing the properties of adaptive algorithms versus stochastic gradient descent (SGD), arising in part from several empirical studies suggesting that adaptive methods outperform SGD when training transformer models \cite{kunstner2024heavy,zhang2024transformers}. Traditional analyses of adaptive algorithms establish regret bounds for online convex optimization \cite{duchi2011adaptive, hazan2016introduction}, and more recent work establishes convergence rates for smooth, non-convex objectives \cite{defossez2020simple,ward2020adagrad}. However, because SGD is known to have optimal convergence rates in these settings, these theoretical results only show that adaptive algorithms achieve rates matching those of SGD.

In order to understand when adaptive algorithms enjoy provably stronger guarantees than SGD, recent work studies convergence under refined geometric assumptions, with particular emphasis on assumptions that are \textit{not} rotationally invariant \cite{jiang2024convergence, liu2024adagrad, xie2024adamexploitsellinftygeometryloss}. Xie et al. \cite{xie2024adamexploitsellinftygeometryloss} establish convergence guarantees in terms of the $\ell_{\infty}$ smoothness constant and show experimentally that rotationally invariant geometric assumptions do not suffice to capture settings when Adam outperforms SGD.  Jiang et al. \cite{jiang2024convergence} and Liu et al. \cite{liu2024adagrad} study convergence of Adagrad on objectives that are \textit{coordinate-wise smooth}, defined in Section~\ref{ssec:notation}. Our analysis builds on these results; Jiang et al. \cite{jiang2024convergence} and Liu et al. \cite{liu2024adagrad} show that the sum of the coordinate-wise smoothness constants governs Adagrad convergence, and we prove that EGOP reparameterization decreases this value by a factor as large as $1/d$, for $d$ the problem dimension.

\begin{figure*}
    \centering
    \includegraphics[width=\linewidth]{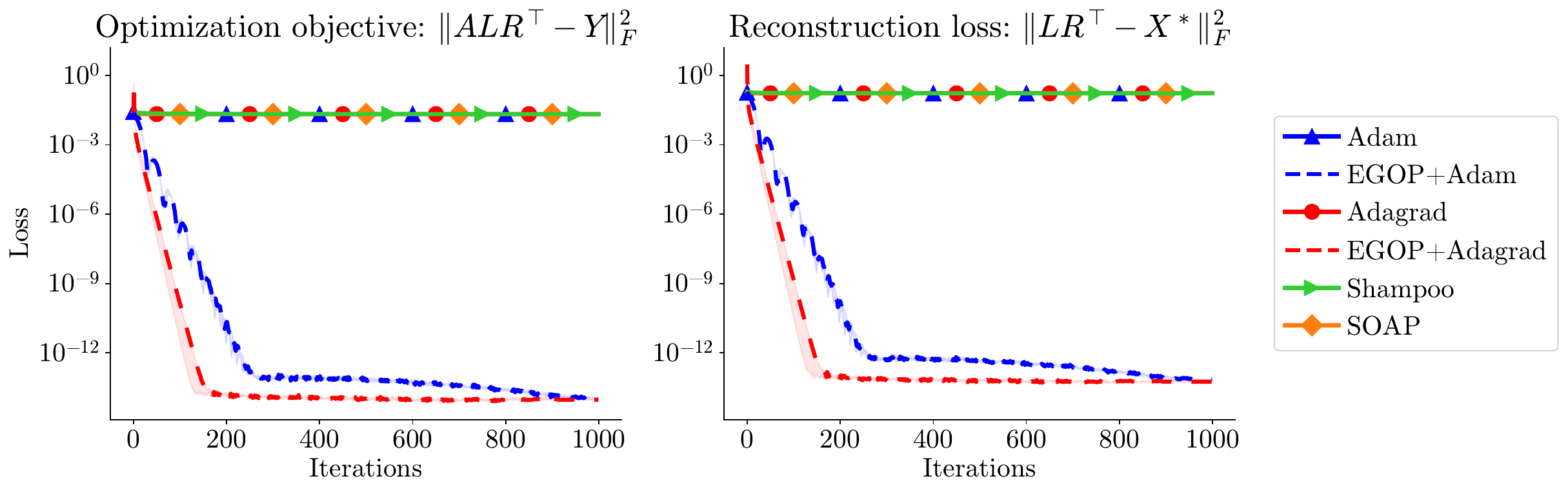}
    \caption{Comparing optimization algorithms on an instance of low-rank matrix factorization (defined in Section~\ref{sssec:LRMF}). Base Adam, base Adagrad, Shampoo, and SOAP all have similar performance, and thus \textbf{their traces are superimposed}. On the right, we validate that the improved optimization exhibited by EGOP-reparameterized Adam and Adagrad selects a good minimizer, and does not overfit to noise. For this objective, the Hessian blocks $\nabla_{L,L}^2 f(\theta)$  and $\nabla_{R,R}^2 f(\theta)$ always admit exact Kronecker product factorizations, making this setting well-suited to the constraints imposed by SOAP and Shampoo. Discussion and details in Section~\ref{sec:experimental-results}.}
    \label{fig:LRMF}
\end{figure*}

\paragraph{Change-of-Basis for Adaptive Algorithms} 
Recent works propose that using different orthonormal transformations with Adam and its variants can reduce computational costs and improve performance in neural networks \cite{maes2024understanding,vyas2024soap,  zhao2024galore}. Vyas et al. \cite{vyas2024soap} proposed a method called SOAP, which computes an orthonormal reparameterization based on the singular vectors of the matrix-valued network gradients and performs optimization in this basis. Vyas et al. \cite{vyas2024soap} empirically examine the performance of SOAP and find that in LLM pretraining it outperforms both Adam and Shampoo, a related preconditioning method \cite{gupta2018shampoo}. Zhao et al. \cite{zhao2024galore} propose GaLore, a method that simultaneously performs reparameterization and dimensionality reduction. GaLore computes a similar orthogonal basis to that used in SOAP, but instead of a full-dimensional change-of-basis GaLore retains only leading basis vectors in order to reduce dimension \cite{zhao2024galore}. Maes et al. \cite{maes2024understanding} empirically study Adam's rotational sensitivity and propose an orthonormal reparameterization, similar to those used by SOAP and GaLore; they show empirically that this can improve Adam's performance \cite{maes2024understanding}. 

EGOP reparameterization is related to SOAP, Shampoo, and GaLore because all three methods use spectral information of some notion of gradient covariance, but the properties of the EGOP reparameterization proposed in this work are fundamentally distinct from those of SOAP/Shampoo/GaLore. As we discuss in Section~\ref{sec:experimental-results}, the change of basis leveraged by those methods is constrained to have a Kronecker product structure, while our approach does not impose this constraint; in Figure~\ref{fig:EIV}, we demonstrate empirically that as a result, EGOP reparameterization outperforms SOAP and Shampoo in settings where the Hessian is not well-approximated by a Kronecker product. Furthermore, as shown in Figure~\ref{fig:LRMF},  we observe empirically that EGOP reparameterization can outperform SOAP and Shampoo even in settings where the Hessian \emph{is} well-approximated by a Kronecker product, suggesting that EGOP spectral decay is a powerful tool even when additional structure is present.

Outside of training neural networks, several works have considered data-driven dimensionality reduction methods for optimizing more general objectives with low-rank EGOP matrices \cite{cartis2024learning, cosson2023low}. These procedures target objectives with exact low-rank structure, while our method can improve convergence of adaptive algorithms even when the EGOP matrix has strong spectral decay but is full-rank.

\paragraph{EGOP Structure in Machine Learning}
    Increasing empirical evidence across a wide range of applications, including language modeling and image classification, suggests that empirical loss functions used for training machine learning models are approximately low-rank---meaning these functions vary strongly in only a small subset of directions in parameter space \cite{papyan2018full,sagun2017empirical,  zhang2024transformers}. This approximate low-rank structure can be detected and analyzed using \textit{active subspace methods}, which often leverage the \textit{EGOP matrix}, defined in Section~\ref{sec:EGOP-defn}  \cite{constantine2015active}. Recently, there has been growing interest in the EGOP matrix in machine learning research \cite{cui2020active, mallinar2024emergence,radhakrishnan2022mechanism,zhu2023catapults}. Radhakrishnan et al. \cite{radhakrishnan2022mechanism} provide theoretical and empirical evidence that the weights of neural networks trained with gradient descent correlate strongly with a kind of EGOP matrix.\footnote{In Mallinar et al. \cite{mallinar2024emergence} Radhakrishnan et al. \cite{radhakrishnan2022mechanism}, and Zhu et al. \cite{zhu2023catapults}, the EGOP is defined using the gradient with respect to the \emph{input data} instead of the optimization parameters.}

\subsection{Notation}\label{ssec:notation}

Given a vector $\theta \in \R^d$, we denote its $i$\ts{th} entry by $\theta(i)$. We denote the inner product on $\R^d$ by $\langle \cdot, \cdot \rangle$, and let $\shortnorm{\cdot}_p$ denote the vector $p$-norm on $\R^d$, with $\shortnorm{\theta}_{\infty} \defeq \max_{i}\abs{\theta(i)}$. Given a matrix $A \in \R^{m \times n}$, we write $\frobnorm{A}$ for its \emph{Frobenius} norm and $\opnorm{A} \defeq \sup_{\shortnorm{\theta}_2=1} \shortnorm{A\theta}_2$ for its \emph{operator} norm. Given a PSD matrix $H\in \R^{d\times d}$, we denote the norm $\shortnorm{\theta}_{H} \defeq \sqrt{\ip{\theta, H\theta}}$. For a matrix $A \in \R^{n\times m}$, we denote the vectorization of $A$ by  $\operatorname{vec}(A) \in \R^{nm}$.

We obtain guarantees in terms of the \textit{coordinate-wise smoothness constants} of the objective $f(\cdot)$. Following Jiang et al. \cite{jiang2024convergence} and Liu et al. \cite{liu2024adagrad}, we say that a function $f$ is \textit{coordinate-wise smooth} within a set $\Theta\subseteq \R^d$ with respect to constants $L_1,\dots,L_d > 0$ if 
$\forall \theta_1, \theta_2 \in \Theta$: 
\begin{equation}\label{eq:def-coordinate-wise-smoothness}
    \abs{f(\theta_1)-f(\theta_2) - \ip{\nabla f(\theta_2), \theta_1 - \theta_2}} \leq
    \frac{1}{2} \norm{\theta_1 - \theta_2}_{L}^2,
\end{equation}
where $L = \diag(L_1, \dots, L_{d})$.
\section{EGOP Reparameterization}\label{sec:EGOP-defn}

    Given a function $f:\R^d\rightarrow \R$ and a sampling distribution $\rho$, the expected gradient outer product of $f(\cdot)$ with respect to $\rho$ is defined as
    \begin{equation}\label{eq:def-EGOP}
        \EGOP(f) \defeq \mathbb{E}_{\theta\sim \rho}\left[\nabla f(\theta)\nabla f(\theta)^{\T}\right].
    \end{equation}
    As $f:\R^d\rightarrow \R$, the EGOP is a $d\times d$ symmetric matrix, and it is positive semidefinite because it is an expectation over PSD matrices. We denote its eigendecomposition by $\EGOP(f) = V \Lambda V^{\T}$ where $V\in \R^{d\times d}$ is an orthonormal matrix. The EGOP eigenbasis captures key geometric properties of the function $f(\cdot)$. When the sampling distribution $\rho$ is isotropic, the leading eigenvectors of the EGOP matrix capture the directions of greatest variation in $f(\cdot)$, whereas the eigenspaces of the smallest eigenvalues are directions along which $f(\cdot)$ does not vary strongly: this is reflected by the fact that for any eigenpair
    $(\lambda_i, v_i)$ of $\EGOP(f)$,
    $
        \lambda_i = \mathbb{E}_{\theta\sim\rho}[\langle \nabla f(\theta), v_i\rangle^2 ].
    $

    The EGOP matrix is defined with respect to a user-specified sampling distribution $\rho$.
    Our guarantees in  Section~\ref{sec:convergence-analysis} require that $\rho$ be isotropic and that its scale is sufficiently large with respect to the norm of some local minimum of $f(\cdot)$. In Section~\ref{sec:experimental-results}, we present empirical results when the EGOP is estimated with $\rho$ a standard Gaussian, and when $\rho$ is defined by common neural network initialization distributions \cite{glorot2010understanding}.
    
    In this work, we compare how adaptive optimization algorithms perform when optimizing $f(\cdot)$ versus the reparameterized function $\tf:\R^d\rightarrow \R$ defined $\tf(\ttheta)\defeq f(V\ttheta)$. For any $f$, the objective $\tf(\cdot)$ can be approximated by empirically estimating the EGOP via Monte Carlo sampling and forming the eigenvectors of the empirical EGOP matrix, as summarized in~\cref{alg:meta-algorithm-block}. Note that any solution $\ttheta$ obtained by optimizing $\tf(\cdot)$ can be transformed into a solution in original coordinates as $\theta \defeq
    V \ttheta$, which satisfies $f(\theta) = \tf(\ttheta)$.
    
    \begin{algorithm}
    \caption{Reparameterization by EGOP Eigenbasis}
    \label{alg:meta-algorithm-block}
    \begin{algorithmic}[1]
        \State {\bfseries Input:} $M$ number of gradient samples, distribution $\rho$.
        \State Generate $\{\theta_i\}_{i=1}^M \sim \rho$ i.i.d.
        \State Form empirical EGOP $\hat{P} = \frac{1}{M}\sum_{i=1}^M \nabla f(\theta_i)\nabla f(\theta_i)^\T$
        \State Form eigendecomposition $V\Lambda V^\T = \hat{P}$
        \State Define function $\tilde{f}(\cdot) = f\circ V$
        \State Optimize $\tf(\cdot)$ with adaptive algorithm of choice.
        \State {\bfseries Return:} $\ttheta \in \R^d$ the result of optimizing $\tf(\cdot)$.
    \end{algorithmic}
    \end{algorithm}

\section{Convergence Guarantees under EGOP Reparameterization}\label{sec:convergence-analysis}

In this section, we show that for objectives with strong EGOP spectral decay and Lipschitz Hessians, there is some neighborhood around each local minima such that EGOP reparameterization improves Adagrad's convergence guarantees within that neighborhood. Our results show that the radius of this neighborhood scales inversely with the Lipschitz constant of the Hessian. We analyze Adagrad's convergence in both convex and nonconvex settings, focusing
on the case of exact gradients for simplicity (see~\cref{alg:Adagrad} for a precise statement). In the convex setting, we study Adagrad constrained to set $\Theta$, which uses the following update rule:
\begin{equation}
    \label{eq:adagrad}
    \tag{\texttt{Adagrad}}
    \theta_{t+1} = \Pi_{\Theta}^{H_t}(\theta_t - \eta H_{t}^{-1} \grad f(\theta_t)), \;\;
    H_t := \diag\Big(
        \epsilon^2 I_{d} + \sum_{j \leq t} \grad f(\theta_j) \grad f(\theta_j)^{\T}
    \Big)^{1/2},
\end{equation}
where $\Pi_{\Theta}^{H}$ denotes the orthogonal projection onto set $\Theta$ under the
metric induced by $H \succ 0$.

Our results relate the improvement obtained through EGOP reparameterization to the \textit{stable rank}\footnote{We refer to this quantity the stable rank because of the connection to the stable rank considered in numerical linear algebra; the ratio in (\ref{eq:def-stable-rank}) is related to $\norm{G}_*/\opnorm{G}$, where $G$ denotes the empirical gradient bundle matrix $[\nabla f(\theta_1),\dots,\nabla f(\theta_M)]\in \R^{d\times M}$ and $\norm{\cdot}_*$ denotes the nuclear norm. This ratio is often referred to in literature as the \textit{stable} or \textit{effective rank} of $G$ \cite{chou2024gradient,rudelson2007sampling}.} of $f$:
\begin{equation}\label{eq:def-stable-rank}
    {\sr}(f) \defeq \frac{\sum_{i=1}^d \sqrt{\lambda_i(\EGOP(f))}}{\sqrt{\lambda_{\max}(\EGOP(f))}}.
\end{equation}
Functions with strong EGOP spectral decay will have $\sr(f)\ll d$, with $\sr(f)$ tending towards $1$ as spectral decay increases.

We now introduce the main assumptions for our theoretical results. Setting the stage,
we consider a twice-differentiable objective $f:\R^d\rightarrow \R$, a sampling distribution $\rho(\cdot)$ and fix a local minimum $\theta^*$. We consider constraint set $\Theta\subseteq \R^d$ and denote its $\ell_2$ diameter by
$
    \diam(\Theta).
$

\begin{assumption}\label{assumption:rho-and-Theta}
    The set $\Theta$ is convex, and the sampling distribution $\rho(\cdot)$ has support contained in $\Theta$ and is an isotropic distribution centered at some $\theta_c\in \Theta$ with scale $c$. This implies $\mathbb{E}_\rho[\theta] = \theta_c$ and 
    \[
        \mathbb{E}_{\rho}[(\theta-\theta_c)(\theta-\theta_c)^\T] = c^2\mathbb{I}.
    \]
    Moreover we assume $\Theta$ contains $\theta^*$ some local minimum of $f(\cdot)$ such that $\norm{\theta_c-\theta^*}_2 \leq c$.
\end{assumption}

Following the definitions introduced in Section~\ref{sec:EGOP-defn}, we let $V$ denote the eigenbasis of the EGOP matrix and let $\tf\defeq f\circ V$. Let $\tTheta$ denote the corresponding transformation of set $\Theta$: $\tTheta\defeq \{V^\T \theta \mid \theta\in \Theta\}$. Let $\{L_i\}_{i=1}^d$ denote the values for which $f(\cdot)$ is coordinate-wise smooth within $\Theta$, following the definition in Eq.~\ref{eq:def-coordinate-wise-smoothness}. We denote the vector of these coordinate-wise smoothness constants by $\vec{L}\in \R^d$.  Let $\{\tilde{L}_i\}_{i=1}^d$ denote the analogous coordinate-wise smoothness constants of $\tf(\cdot)$ within $\tTheta$. 

To measure the density of the eigenvectors of $\EGOP(f)$, we introduce the values $\beta_k$: for $v_k$ the $k$-th eigenvector of $\EGOP(f)$, let $\beta_k \defeq \norm{v_k}^2_1/d.$ Note that $\forall k\in [d]$, $\beta_k \in [1/d, 1]$ with $\beta_k=1/d$ if $v_k$ is one-hot, and $\beta_k=1$ if $v_k$ is uniformly dense.

Our second assumption requires that the Hessian of $f(\cdot)$ be Lipschitz within $\Theta$ and that its Lipschitz constant be suitably bounded.
\begin{assumption}\label{assumption:H-Lipschitz}
    The Hessian of $f(\cdot)$ is $H$-Lipschitz within $\Theta$: $\exists H\in \R$ such that
    $\forall \theta_1, \theta_2 \in \Theta$, $\opnorm{\nabla^2 f(\theta_1)-\nabla^2 f(\theta_2)} \leq H \norm{\theta_1-\theta_2}_2 \ 
    $
    and $\exists \delta \in [0,\beta^2_1)$ such that $H$ satisfies
    \begin{equation}\label{eq:H-bound}
        H \leq \frac{\sqrt{\delta \lambda_1(\EGOP) + \diam(\Theta)^2 \lambda^2_1(\nabla^2 f(\theta^*))} - \diam(\Theta) \lambda_1(\nabla^2 f(\theta^*))}{\diam(\Theta)^2}
    \end{equation}
    
\end{assumption}

Assumption~\ref{assumption:H-Lipschitz} quantifies how small the Lipschitz constant of the Hessian must be, i.e. how slowly the curvature of $f(\cdot)$ can change within $\Theta$, in order for the results in \cref{lemma:convex-cvgnce-ball-version} to apply, and its particular form derives from the analysis. 
Critically, $H$ is not an input to the algorithm, and Adagrad in the new coordinate system is guaranteed to converge even when Eq.~\ref{eq:H-bound} does not hold, albeit without the improved convergence bound from \cref{lemma:convex-cvgnce-ball-version}.

Under the above assumptions, we show that EGOP spectral decay governs the improvements conferred
by EGOP reparameterization. We first consider the convex setting. For simplicity, we instantiate these guarantees for constraint set $\Theta$ a ball; in practice, adaptive algorithms are often deployed with weight decay, which bounds the norm of the optimization parameters and effectively constrains $\theta$ to some ball.

In Section~\ref{ssec:cvx-cvgnce}, we state guarantees for generic convex constraint sets.
\begin{theorem}[Convergence for convex objectives with noise-free gradients]\label{lemma:convex-cvgnce-ball-version}
    Consider $f(\cdot)$ a convex objective, constraint set $\Theta$ a ball, and sampling distribution $\rho(\cdot)$ satisfying Assumptions \ref{assumption:rho-and-Theta} and \ref{assumption:H-Lipschitz}, and $\tf(\cdot)$, $\tTheta$ the EGOP-reparameterized objective and constraint set, respectively.  Running Algorithm~\ref{alg:Adagrad} with constrained updates and inputs $(\tf(\cdot), \nabla \tf(\cdot),\ttheta_0,T, \epsilon, \tTheta)$ for any initial condition $\ttheta_0\in \tTheta$ produces iterates $\{\ttheta_t\}_{t=1}^T$ satisfying
    \begin{equation}
        \begin{aligned}
            \frac{1}{T}\sum_{t=0}^{T-1}\tf(\ttheta_t) - \tf(\ttheta^*) &\\
            \lesssim
            \bigg(\eta + &\frac{\diam(\Theta)^2}{\eta}\bigg)^2 \frac{\shortnorm{\vec{L}}_1}{T}\left(\frac{\sr(f)\sqrt{1+\delta}}{d(\beta_1-\delta)} + \frac{\sqrt{\delta(1+\delta)}}{\beta_1-\delta} \right)
         + \frac{\epsilon\diam(\Theta)^2}{\eta T}.
        \end{aligned}
        \label{eq:cvx-reparam-bounds-ball-version}
    \end{equation}
\end{theorem}

We compare the above bound with the convergence guarantee for Adagrad in original coordinates initialized at any $\theta_0 \in \Theta$, established by Liu et al.\footnote{The inequality in \eqref{eq:cvx-og-coors-bound} is an intermediate result in the proof of Theorem 4.1 in Liu et al. \cite{liu2024adagrad}.} \cite{liu2024adagrad}:
\begin{equation}\label{eq:cvx-og-coors-bound}
    \frac{1}{T}\sum_{t=0}^{T-1}f(\theta_t) - f(\theta^*))\leq \left(\eta + \frac{\diam(\Theta)^2}{\eta}\right)^2\cdot \frac{\shortnorm{\vec{L}}_1}{T} + \frac{\epsilon \diam(\Theta)^2}{\eta T}.
\end{equation}
When $f(\cdot)$ has strong EGOP spectral decay and dense leading eigenvectors, \cref{lemma:convex-cvgnce-ball-version} demonstrates EGOP reparameterization can significantly improve Adagrad's regret bounds. In these settings, $\sr(f)/(\beta_1 d) = O(1/d)$, implying stronger guarantees over original coordinates by up to a factor of $d$; the term $\shortnorm{\vec{L}}_1$ appearing in both Eq.~\ref{eq:cvx-reparam-bounds-ball-version} and \ref{eq:cvx-og-coors-bound} is the same, and denotes the sum of the coordinate-wise smoothness constants of $f(\cdot)$ in original coordinates. We defer the proof of \cref{lemma:convex-cvgnce-ball-version} to Section~\ref{sec:deferred-proofs}.

In Section~\ref{sec:EGOP-spectral-decay}, we show that the two properties of the EGOP emphasized in this result, namely low stable rank and dense leading eigenvectors, are satisfied empirically for benchmark objectives. The factor $\beta_1$ tends towards $1$ as the leading EGOP eigenvector gets denser. For the guarantee to reflect a benefit from reparameterization, it suffices to have $\beta_1 \gg 1/d$ and small stable rank. We note that the density condition $\beta_1 \gg 1/d$ is satisfied with high probability for random unit vectors in high dimensions: in particular, for $v\sim\textrm{Unif}(S^{d-1})$ where $S^{d-1}$ is the unit sphere in $\R^d$, with high probability $\norm{v}^2_1/d > 0.6$.   For simplicity, the above guarantee only considers the density of $v_1$, but generalized guarantees in terms of $\beta_k$ can be obtained. 
Similarly, one can generalize \cref{lemma:convex-cvgnce-ball-version} for approximate versions of the EGOP eigenbasis; such guarantees will scale with the subspace distance between the EGOP eigenbasis and its approximation. We note that numerical stability parameter $\epsilon$ is typically chosen small enough that the second term, $\epsilon\diam(\Theta)^2/(\eta T)$, is not the dominant term.

Many naturally-motivated objectives in machine learning have locally Lipschitz Hessians, including loss functions used in logistic regression, over-parameterized matrix factorization, and training of multilayer linear networks; see Section~\ref{ssec:ML-lipschitz-Hessians} for examples.

In the non-convex setting, we study the convergence of unconstrained Adagrad (Algorithm~\ref{alg:Adagrad} with unconstrained updates). We show there is some neighborhood around each local minima such that while iterates remain in this neighborhood, EGOP-reparameterized Adagrad enjoys local convergence bounds that are stronger than those in original coordinates by a factor of $\sr(f)/d\beta_1$ (see \cref{lemma:nonconvex-cvgnce}). Similarly to the convex setting, the radius of these neighborhoods grows as the Lipschitz constant of the Hessian decreases.

\section{EGOP Spectral Decay in Machine Learning}\label{sec:EGOP-spectral-decay}

\begin{figure}[t!]
  \begin{center}
    \includegraphics[width=0.6\linewidth]{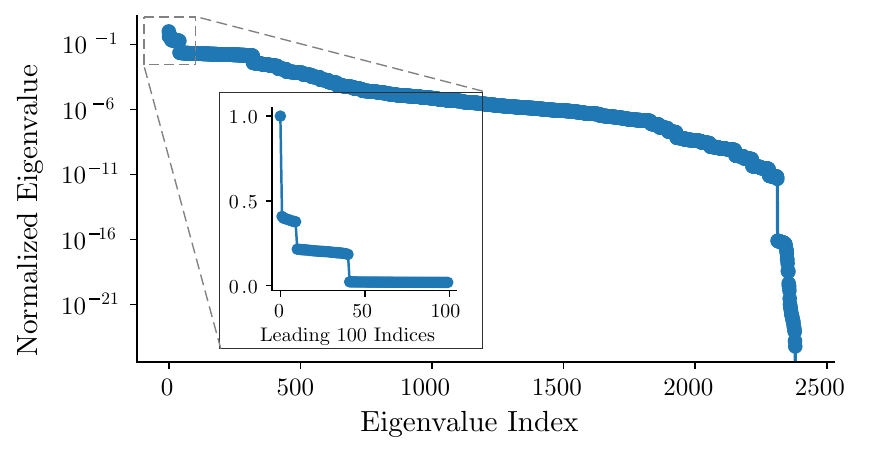}
  \end{center}
  \caption{The EGOP eigenspectrum of a 2-layer ReLU network on the UCI handwritten digits dataset. Plot shows ratio $\lambda_k/\lambda_1$ as a function of eigenvalue index $k$, indexed in decreasing order. The EGOP eigenvalues decay sharply. In Section~\ref{ssec:sup-figs-spectral-decay} of the supplementary material, we demonstrate that these results are robust to the choice of sampling distribution $\rho$.}\label{fig:tinyMNIST-global-spectral-decay}
\end{figure}

Our analysis in Section~\ref{sec:convergence-analysis} shows EGOP spectral decay and dense leading EGOP eigenvectors are sufficient for EGOP reparameterization to improve convergence guarantees for Adagrad. Here we present empirical evidence that these conditions occur in benchmark machine learning objectives and discuss why natural data may produce EGOP spectral decay in real-world problems.

Figure~\ref{fig:tinyMNIST-global-spectral-decay} shows the EGOP eigenspectrum for the objective function $f(\cdot)$ from an image classification problem. We use a 2-layer ReLU neural network to predict 10-class probabilities for handwritten digit images from the UCI ML digits dataset \cite{optical_recognition_of_handwritten_digits_80}. Here $f(\cdot)$ denotes the negative log-likelihood loss on the training data. Figure~\ref{fig:tinyMNIST-global-spectral-decay} shows strong EGOP spectral decay. We plot $\lambda_k/\lambda_1$ for all EGOP eigenvalues $\lambda_k$ on a logarithmic scale, while the inset zooms in on the leading 100 indices on a linear scale. These plots illustrate roughly exponential eigenvalue decay. In Section~\ref{ssec:sup-figs-spectral-decay}, we show these trends are robust to the choice of sampling distribution $\rho$, and that decay occurs in other benchmark datasets. In Section~\ref{ssec:sup-figs-spectral-decay} we also examine the density of the leading EGOP eigenvectors for a 2-layer ReLU network on the UCI digits dataset and show that $\beta_k \gg1/d$ for the leading eigenvectors.

    \subsection{Natural Data Induces EGOP Spectral Decay}

    In addition to empirical evidence, simple gradient calculations suggest natural data may induce EGOP spectral decay in machine learning problems. Many common objectives  can be expressed as $f(\theta) = h(A\theta)$, where $h(\cdot)$ is a loss function and $A \in \R^{n\times d}$ is a data matrix whose rows comprise samples $a_i \in \R^d$. By the chain rule, the EGOP for such objectives satisfies
    \begin{equation}
        \label{eq:EGOP-chain-rule-main}
        \EGOP(f) = A^\T \mathbb{E}_{\theta\sim \rho}\left[\nabla_\theta h(A\theta) \nabla_\theta h(A\theta)^\T\right] A
    \end{equation}
    where $\nabla_{\theta}h(A\theta)$ denotes the gradient $\nabla h(\cdot)$ evaluated at $A\theta$. This expression shows the EGOP is the transformation of some PSD matrix $M$ by $A^\T M A$. It suggests that strong spectral decay in $A$ may induce eigenvalue decay in the EGOP matrix. For many naturally occurring data distributions, the singular values of $A$ exhibit strong decay \cite{udell2019big}. The inner PSD matrix in the right hand side of Eq.~\ref{eq:EGOP-chain-rule-main} depends on both $A$ and $h(\cdot)$, so without further assumptions on $h(\cdot)$ it is difficult to precisely characterize the spectral decay induced by the composition with $A^\T$ and $A$, but this can be quantified for specific choices of $h(\cdot)$; see Section~\ref{ssec:spectral-decay-proofs} for examples.
\section{Heuristics for Scalability}\label{sec:efficient-heuristics}

Reparameterization with the EGOP eigenbasis incurs three main sources of additional computation: (1) sampling gradients to estimate the EGOP, (2) forming the EGOP eigenbasis, and (3) storing and applying the change-of-basis matrix to compute values and gradients of $f\circ V$. We outline some implementation details and heuristics that reduce the computational cost and enhance the scalability of the proposed framework.

Empirically, we find that for functions with strong spectral decay, it suffices to accurately estimate only the leading EGOP eigenvectors. Based on this finding, one can use techniques such as randomized SVD \cite{halko2011finding} to form $V_r \in \R^{d\times r}$, a matrix whose columns contain the estimated $r$ leading eigenvectors of the EGOP. To optimize over the full parameter space when estimating only the leading eigenvectors, we define a procedure called \textit{auxiliary variable} EGOP reparameterization. In this procedure, we use parameters $\ttheta_r$ to optimize over the span of $V_r$, and we introduce auxiliary parameters $\ttheta_d \in \mathbb{R}^d$ to optimize over the orthogonal complement. This method then represents vectors $\theta \in \R^d$ in full parameter space as
\begin{equation}\label{eq:aux-reconstruction}
    \theta = V_r \ttheta_r + (I - V_r V_r^T) \ttheta_d.
\end{equation}
The explicit projection $(I - V_r V_r^T)$ applied to $\ttheta_d$ ensures the second term remains in $\text{span}(V_r)^\perp$ throughout optimization. This is less expensive than forming and storing  $(V_r)^\perp \in \R^{d \times (d-r)}$, and the projection $(I - V_r V_r^T)$ can be applied without explicitly forming a $d\times d$ matrix, making this computationally efficient. For small values of $r$, storing $V_r$ is far less expensive than storing the full $V\in \R^{d\times d}$ matrix. In Section~\ref{ssec:heuristic-EGOP-experiments}, we show empirically that using this procedure with even small values of $r \ll d$ suffices to improve convergence of adaptive algorithms. For full details and pseudocode for auxiliary variable EGOP reparameterization, see Section~\ref{sec:reduced-egop}. 

For functions with strong spectral decay, we find that empirically EGOP reparameterization performs well when the number of gradient samples $M$ used to estimate the EGOP eigenbasis satisfies $M \approx d$, for $d$ the ambient dimension. Existing results (e.g., Corollary 3.10 in \cite{constantine2015active}) formally establish that for functions whose EGOP matrices have large spectral gaps, the number of gradient samples required to obtain an $\epsilon$-accurate estimate of the leading eigenspace scales like $\log(d)/\epsilon^2$. When using heuristics like auxiliary variable EGOP reparameterization, it suffices to approximate only leading eigenspaces; in Section~\ref{ssec:heuristic-EGOP-experiments}, we show that auxiliary variable EGOP reparameterization can confer a benefit even when using $M = 0.01 d$ gradient samples to estimate the leading EGOP eigenspace.

Structured approximations of $V$ can also reduce the cost of storing and applying the change of basis. Here we describe one such approximation, which we call \textit{block reparameterization}; given a user-specified number of blocks $L$, let $S_1,\dots,S_L$ denote a disjoint partition of the indices $\{1,\dots,d\}$. Let $\theta(S_\ell) \in \R^{|S_\ell|}$ be the vector whose entries correspond to the elements of $\theta$ with indices in $S_\ell$. In block reparameterization, for each subset $S_\ell$, one obtains a separate  change-of-basis matrix $V^{(\ell)} \in \R^{\abs{S_\ell} \times \abs{S_\ell}}$ via the eigenvectors of the \emph{block} EGOP matrix, 
\[
    \EGOP^{(\ell)} \defeq \frac{1}{M}\sum_{k=1}^M \nabla_{S_\ell} f(\theta_k)\nabla_{S_\ell} f(\theta_k)^\T
\]
where $\nabla_{S_\ell} f(\theta_k) \in \R^{|S_\ell|}$ is the vector of partial derivatives of $f$ with respect to the entries in $S_\ell$, and the points $\{\theta_k\}_{k=1}^M$ are sampled i.i.d. from $\rho$. Block reparameterization is well-suited to multilayer neural networks, where each layer forms a parameter block; such choices are well-motivated by empirical observations about the block-diagonal structure of the Hessian in neural networks \cite{dong2025towards}. Cost can be further reduced by only reparameterizing a subset of blocks. We present experiments employing block reparameterization in Section~\ref{ssec:heuristic-EGOP-experiments}.

When employing block reparameterization at the layer level, we consider related heuristics for reducing storage cost; this heuristic applies to layers that admit a further decomposition into parameter subgroups, such as kernels in a convolutional layer. In this setting, the same reparameterization matrix $V$ can be shared across parameter subgroups with compatible dimensions within a layer. In our experiments in Section~\ref{ssec:main-body-large-scale}, we apply this heuristic to convolutional layers by treating each kernel as a parameter subgroup and using a single matrix $V$, shared across all kernels in a single convolutional layer. This approach substantially reduces storage requirements for wide and deep neural networks. Although sharing $V$ imposes a constraint on the change-of-basis matrix\footnote{We note that while this construction is equivalent to employing a Kronecker-product reparameterization matrix $V$, the effective Kronecker structure differs from that imposed by SOAP/Shampoo. See Section~\ref{ssec:CNN-sup-details} for details. Moreover, our experiments in Section~\ref{ssec:main-body-large-scale} show that EGOP reparameterization employing this heuristic outperforms SOAP in training deep residual networks.}, we show in Section~\ref{ssec:main-body-large-scale} that EGOP reparameterization with this heuristic can still yield significant performance improvements over methods trained in the original coordinates.

One natural extension of \cref{alg:meta-algorithm-block} is to consider periodically re-estimating the EGOP, using a modified sampling distribution $\rho$ which has been re-centered at the latest optimization iterate. This allows one to recompute a change-of-basis matrix based on local geometry, and may be well-suited to problems with highly heterogeneous structure. In Section~\ref{ssec:main-body-large-scale}, we report experimental results using periodic EGOP reparameterization for a large-scale, non-convex optimization problem, and find that it can offer further benefits over a single up-front estimation of the EGOP eigenbasis. However, we do find that throughout our experiments, the up-front, single-time reparameterization proposed in Algorithm~\ref{alg:meta-algorithm-block} is sufficient to improve convergence, even for objectives where the Hessian is \textit{not} Lipschitz. See Section~\ref{sec:experimental-results} for details.

\section{Experimental Results}\label{sec:experimental-results}

We examine the impact of EGOP reparameterization in a variety of machine learning optimization problems. We include comparisons with (S)GD, (S)GD with momentum, SOAP, and Shampoo. In Section~\ref{ssec:full-EGOP-experiments}, we examine the effect of EGOP reparameterization for Adagrad and Adam, using the full method described in Algorithm~\ref{alg:meta-algorithm-block}. In Section~\ref{ssec:heuristic-EGOP-experiments},  we present results employing the heuristics described in Section~\ref{sec:efficient-heuristics}. 


When comparing performance in original coordinates versus under reparameterization, we always choose equivalent initializations. For example, when using the full method described in Algorithm~\ref{alg:meta-algorithm-block}, when a method in original coordinates is initialized at $\theta_0$, chosen randomly, the reparameterized method is initialized at $V^\T\theta_0$. Full experimental details are in Section~\ref{sec:experimental-details}.

\subsection{Full EGOP Reparameterization}\label{ssec:full-EGOP-experiments}

In this section, we examine the impact of EGOP reparameterization, as described in Algorithm~\ref{alg:meta-algorithm-block}. 

\begin{figure*}
    \centering
    \begin{subfigure}[t]{0.325\textwidth}
        \centering
        \includegraphics[width=\linewidth]{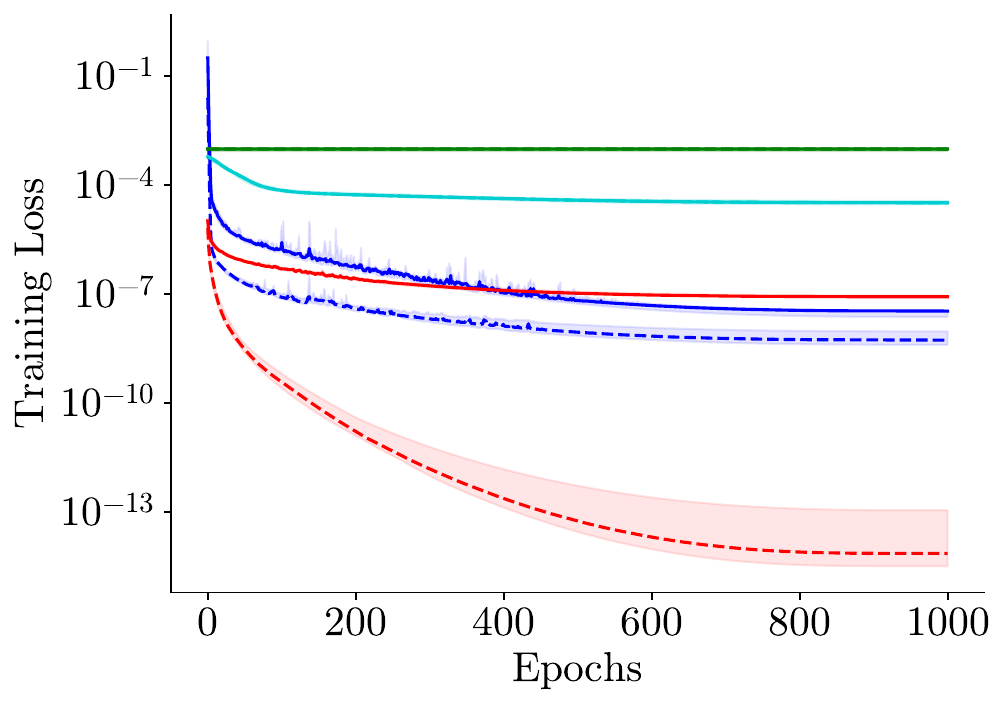}
        \caption{\scalebox{0.85}{Training loss over epochs}}
        \label{fig:linear-layers-global-reparam-loss-vs-epochs}
    \end{subfigure}
    \begin{subfigure}[t]{0.32\textwidth} 
        \centering
        \includegraphics[width=\linewidth]{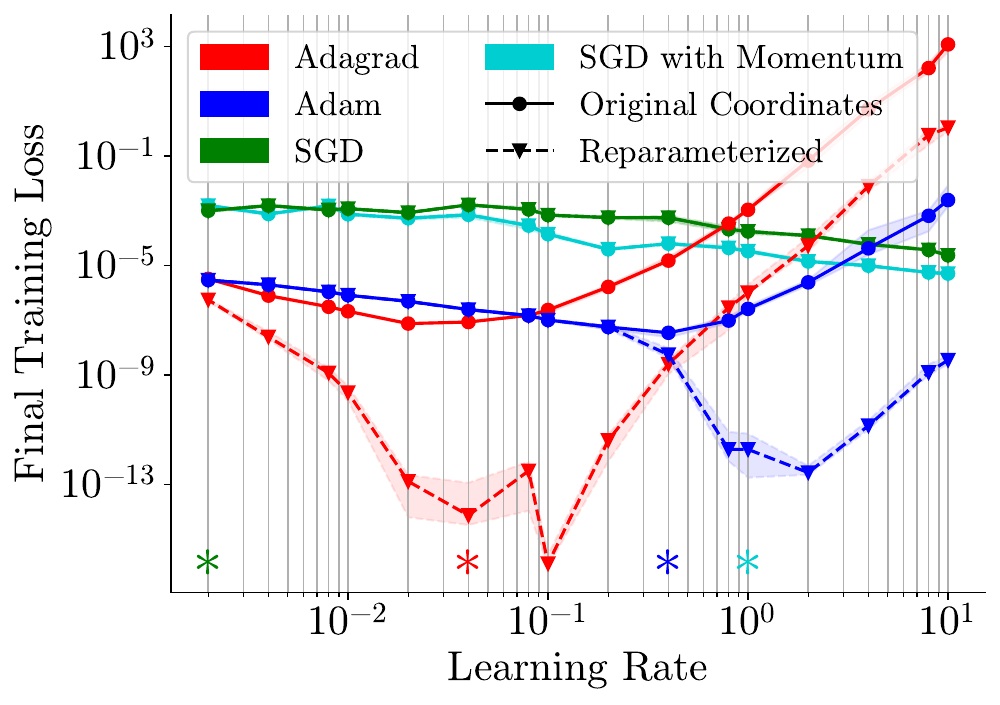}
        \caption{\scalebox{0.85}{Training loss over learning rates}}
        \label{fig:linear-layers-global-reparam-loss-vs-LR}
    \end{subfigure}
    \begin{subfigure}[t]{0.32\textwidth} 
        \centering
        \includegraphics[width=\linewidth]{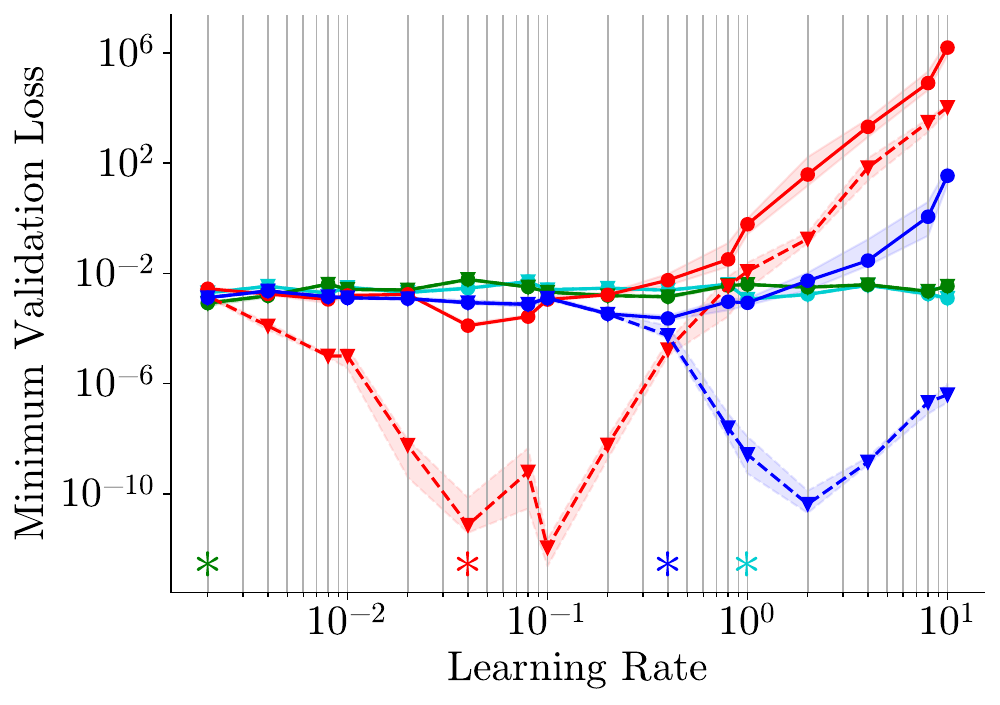}
        \caption{\scalebox{0.85}{Validation loss by learning rate}}
        \label{fig:linear-layers-global-reparam-valloss-vs-LR}
    \end{subfigure}
    \caption{Training multilayer linear networks (\ref{eq:linear-feedforward-objective}). Both SGD and SGD with momentum are equivariant optimization methods, so their results in original and reparameterized coordinates are exactly superimposed. In Figure~\ref{fig:linear-layers-global-reparam-valloss-vs-LR} we consider the minimum validation loss achieved over epochs during training. Results are aggregated over 10 independent trials, with traces showing medians and shading indicating 25th-75th quartile. Asterisks indicate the learning rate used for each method in Figure~\ref{fig:linear-layers-global-reparam-loss-vs-epochs}. Learning rates chosen to minimize validation loss of the algorithm in \textit{original} coordinates.
    }
    \label{fig:global-reparam-linear-layers}
\end{figure*}

\subsubsection{Linear Feedforward Networks} We consider EGOP reparameterization for training 3-layer fully-connected linear feedforward networks using synthetic data. We define parameters 
$
    \theta = [\textrm{vec}(W_1), \textrm{vec}(W_2), \textrm{vec}(W_3)]
$
and train by minimizing loss function
\begin{equation}\label{eq:linear-feedforward-objective}
    f(\theta) = \frobnorm{W_3 W_2 W_1 A - Y}^2 /\nsamples
\end{equation}
where $A\in \R^{10\times \nsamples}$, and $Y = M^* A$ for $M^*\in \R^{10\times 10}$ drawn from a standard Gaussian distribution. We use $W_1 \in \R^{50\times 10}$, $W_2 \in \R^{30\times 50}$, $W_3 \in \R^{10\times 30}$. We induce spectral decay in $\EGOP(f)$ by generating $A$ with singular values $\sigma_k(A) = k^{-2}$ and random singular vectors. We use mini-batch stochastic gradient samples throughout. For full details, see Section~\ref{sec:experimental-details}. 



Figure~\ref{fig:linear-layers-global-reparam-loss-vs-epochs} displays training loss over epochs obtained by algorithms in original coordinates and under EGOP reparameterization. Reparameterized Adagrad and reparameterized Adam achieve significantly lower final training loss and faster convergence than their counterparts in original coordinates. The adaptive methods also outperform the equivariant methods (SGD and SGD with momentum) in both coordinate settings. Figure~\ref{fig:linear-layers-global-reparam-loss-vs-LR} demonstrates the robustness of these results across learning rates. 
Figure~\ref{fig:linear-layers-global-reparam-valloss-vs-LR} confirms that the improved minimization of the training loss enabled by reparameterization does not lead to over-fitting; it shows that EGOP-reparameterized  adaptive methods achieve better performance on validation data.

In Section~\ref{ssec:sup-figs-main-experiments}, we evaluate heuristics from Section~\ref{sec:efficient-heuristics} on the same linear network training task. We compare the results in Figure~\ref{fig:global-reparam-linear-layers}, which use full EGOP reparameterization, to performance using block reparameterization and to performance when only the first layer of the network is reparameterized. Block-reparameterized Adagrad and Adam achieve lower training loss than their counterparts in original coordinates. For Adagrad, reparameterizing only the first layer yields a benefit comparable to reparameterizing all layers, whereas for Adam, reparameterizing the first layer alone provides only marginal benefit over original coordinates.

\subsubsection{Regression with Errors-in-Variables} In this section, we empirically demonstrate one distinction between our proposed reparameterization and the related methods SOAP and Shampoo; EGOP reparameterization exploits approximate low-rank loss function geometry, while SOAP and Shampoo target Kronecker-product structure in the Hessian.

Shampoo and SOAP are designed for optimization problems where the variable has a natural matrix structure (e.g., a neural network weight matrix), and thus the gradient is also matrix-valued. Their change-of-basis derives from the eigenvectors of the Gram matrices of these gradient matrices. When mapped back to a vector representation (i.e., the vectorization of the matrix parameter and the setting of our EGOP-based approach), the corresponding SOAP/Shampoo change-of-basis matrix is restricted to a Kronecker product form. As a result, Shampoo and SOAP are hypothesized to perform best when the Hessian admits a close Kronecker-product approximation. Figure~\ref{fig:EIV} presents evidence supporting this hypothesis.

%
\begin{figure}[t]
    \centering
    \includegraphics[width=\linewidth]{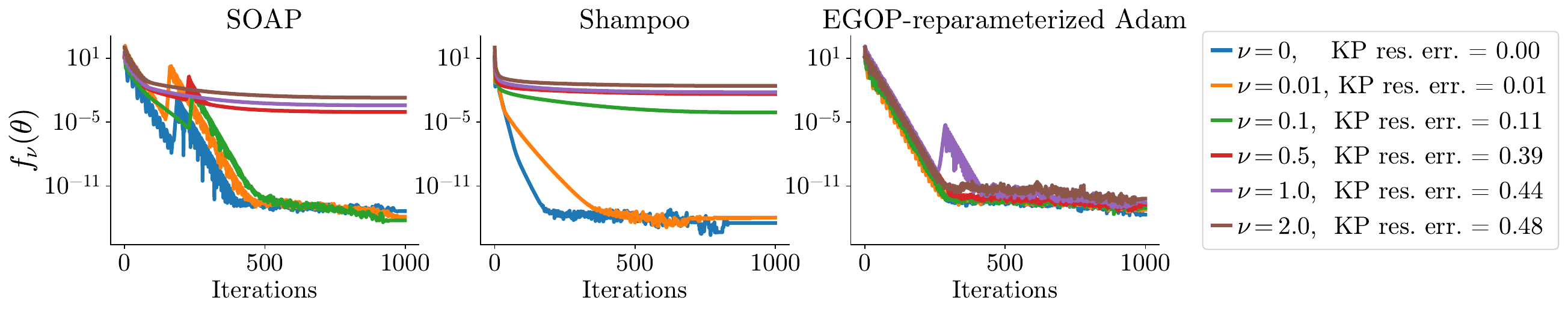}
    \caption{Comparing the impact of Hessian Kronecker-product (KP) structure on EGOP reparameterization, SOAP, and Shampoo, for objective $f(\theta)$ defined in \eqref{eq:EIV-objective}. SOAP and Shampoo both exhibit worse convergence as the Hessian KP approximation worsens. We quantify Hessian KP structure via the KP residual error, defined in \eqref{eq:def-KP-error}. EGOP-reparameterized Adam's convergence is robust to this structure.}
    \label{fig:EIV}
\end{figure}

We consider an objective that allows smooth interpolation between regimes when the Hessian admits an exact Kronecker-product structure, and cases where the Hessian is poorly approximated by such a factorization. For parameters $\theta = \textrm{vec}(W)$, $W \in \R^{n\times m}$, we consider
\begin{equation}\label{eq:EIV-objective}
    f_{\nu}(\theta) = \frac{1}{2} \norm{A(M_\nu\odot W) - Y}^2_F.
\end{equation}
Here, $M_\nu\in \R^{n\times m}$ is a fixed, known multiplicative noisy mask parameterized by $\nu\in \R$, $A \in \R^{n\times n}$ is the measurement map, and $Y = A(M_\nu\odot W^*)$ for some unknown $W^* \in \R^{n\times m}$. This models, for instance, regression problems based on data from sensors with well-characterized but potentially non-homogeneous gain. 

Figure~\ref{fig:EIV} displays $f(\cdot)$ by iteration for different instances of $M_\nu$. We draw the entries of $M_\nu$ i.\-i.d. from $\mathcal{N}(1, \nu^2)$, where different problem instances use different $\nu$. When $\nu = 0$, $M_\nu$ is all ones, and the Hessian is exactly Kronecker-structured. As $\nu$ increases, the Kronecker structure of the Hessian degrades. We quantify this via Kronecker product residual error: 
\begin{equation}\label{eq:def-KP-error}
    \operatorname{KP\ residual}(f) = \min_{A\in \R^{n\times n}, B\in \R^{n\times n}}\norm{\nabla^2 f(\theta) - A \otimes B}_F/\norm{\nabla^2 f(\theta)}_F.
\end{equation}

Figure~\ref{fig:EIV} (left, center) shows that SOAP and Shampoo performance degrade as the Hessian $\nabla^2 f_\nu(\cdot)$ becomes farther from Kronecker-product structured. In contrast, EGOP-reparameter\-ized Adam Figure~\ref{fig:EIV} (right) maintains linear convergence across all problem instances. In these experiments we allow SOAP and Shampoo to recompute their change-of-basis matrix on \textit{every iteration}, whereas for EGOP reparameterization we compute a single, up-front change-of-basis matrix following Algorithm~\ref{alg:meta-algorithm-block}. Even with this advantage, SOAP and Shampoo do not outperform EGOP-reparameterized Adam for $\nu \geq 0.5$.

\subsection{Examining Heuristics for EGOP Reparameterization}\label{ssec:heuristic-EGOP-experiments}

In this section, we empirically examine the impact of the computational heuristics for EGOP reparameterization presented in Section~\ref{sec:efficient-heuristics}. We conduct experiments on low-rank matrix factorization, 2-layer feedforward ReLU networks, and deep residual networks. All of these experiments use \textit{block reparameterization}, defined in Section~\ref{sec:efficient-heuristics}. In experiments with low-rank matrix factorization, parameter blocks correspond to the separate matrix factors. In experiments with neural networks, parameter blocks correspond to the parameters of different network layers.

\subsubsection{Low-rank matrix factorization}\label{sssec:LRMF} Here we discuss the empirical results displayed in Figure~\ref{fig:LRMF}, which show that EGOP reparameterization can offer advantages over SOAP and Shampoo even when the Hessian admits exact Kronecker product structure. We consider low-rank matrix factorization; for parameters $\theta = [\textrm{vec}(L), \textrm{vec}(R)]$ for $L\in \R^{n\times r}, R\in \R^{n\times r}$, we minimize the objective
\begin{equation}\label{eq:LRMF-objective}
    f(\theta) = \frac{1}{2}\norm{A(LR^\T) - Y}^2_F
\end{equation}
where $Y = AX^* + \zeta$ is a matrix of observations with noise $\zeta_{i,j} \sim \mathcal{N}(0, 0.001)$, $X^* \in \R^{n\times n}$ is an unknown ground-truth matrix of (known) rank $r$, and $A\in \R^{n\times n}$ is a measurement map  with singular value decay $\sigma_k(A) =k^{-1/2}$. We use standard spectral initialization. We perform block EGOP reparameterization, using two parameter blocks corresponding to the entries of $L$ and the entries of $R$ respectively.

Figure~\ref{fig:LRMF} compares the performance of different algorithms on this problem. EGOP-reparameterized Adam and Adagrad both achieve better minimization of $f(\cdot)$, and also yield lower reconstruction error, quantified by $\norm{LR^\T - X^*}^2_F$. We allow both SOAP and Shampoo to update their change-of-basis on every iteration, whereas EGOP reparameterization performs only a single up-front computation of the change-of-basis matrix. For this objective, the Hessian blocks $\nabla_{L,L}^2 f(\theta)$  and $\nabla_{R,R}^2 f(\theta)$ always admit exact Kronecker product decompositions, making this problem well-suited to the constraints imposed by SOAP and Shampoo\footnote{We note that in objectives such as \eqref{eq:LRMF-objective}, where the optimization parameters are collections of matrix-valued parameter blocks, SOAP and Shampoo construct block-wise preconditioners for matrix-valued parameters $L$ and $R$ separately. These preconditioners can be viewed as Kronecker product approximations to the Hessian diagonal blocks  $\nabla_{L,L}^2 f(\theta)$ and $\nabla_{R,R}^2 f(\theta)$ respectively.}. Nonetheless, EGOP-reparameterized algorithms outperform SOAP/Shampoo on this task, indicating the need for further analysis in order to predict when each approach will be optimal.

\begin{figure}[t]
\centering
\begin{subfigure}{0.47\textwidth}
    \centering
    \includegraphics[width=\linewidth]{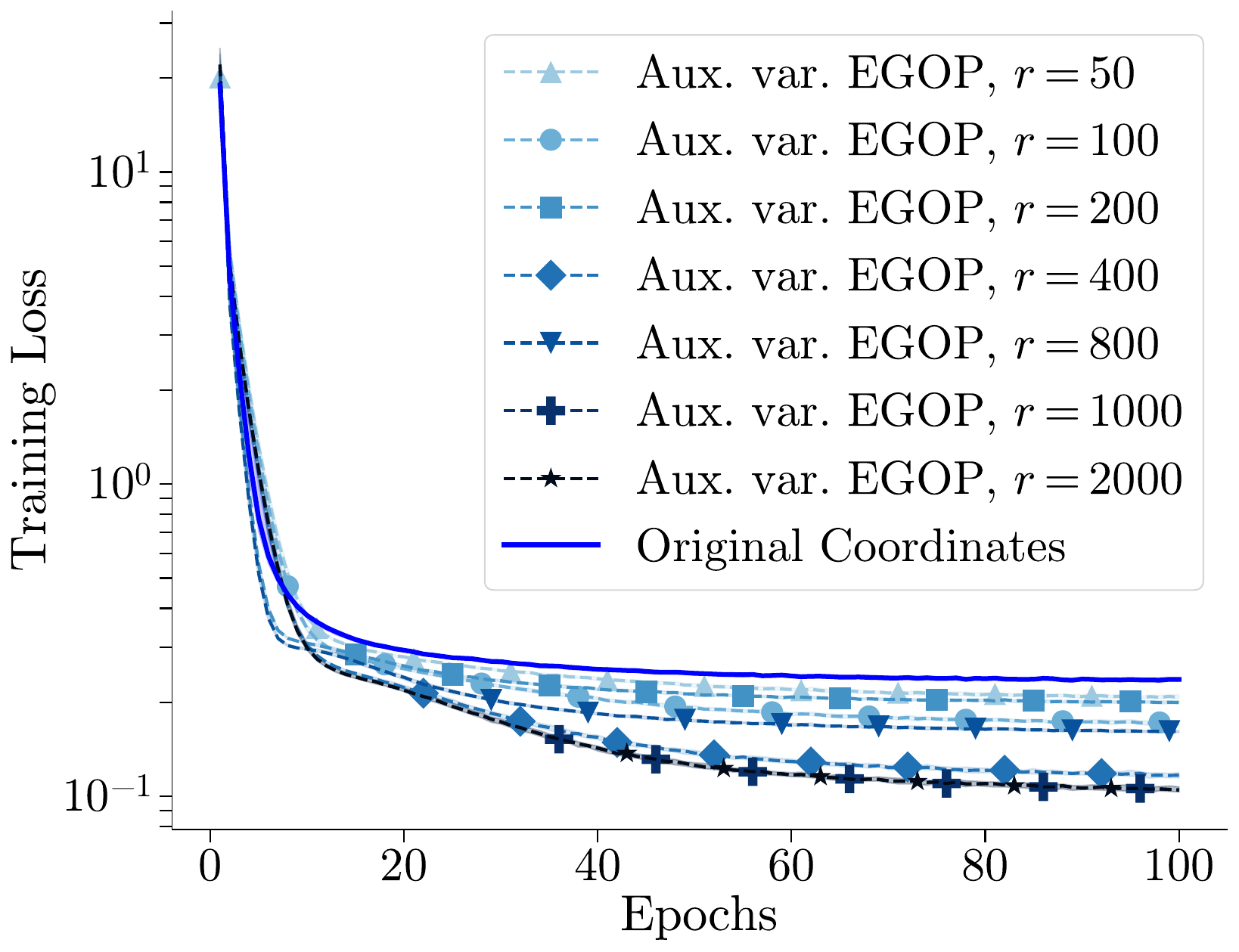}
    \caption{Training loss by epoch}
\end{subfigure}\hfill
\begin{subfigure}{0.47\textwidth}
    \centering
    \includegraphics[width=\linewidth]{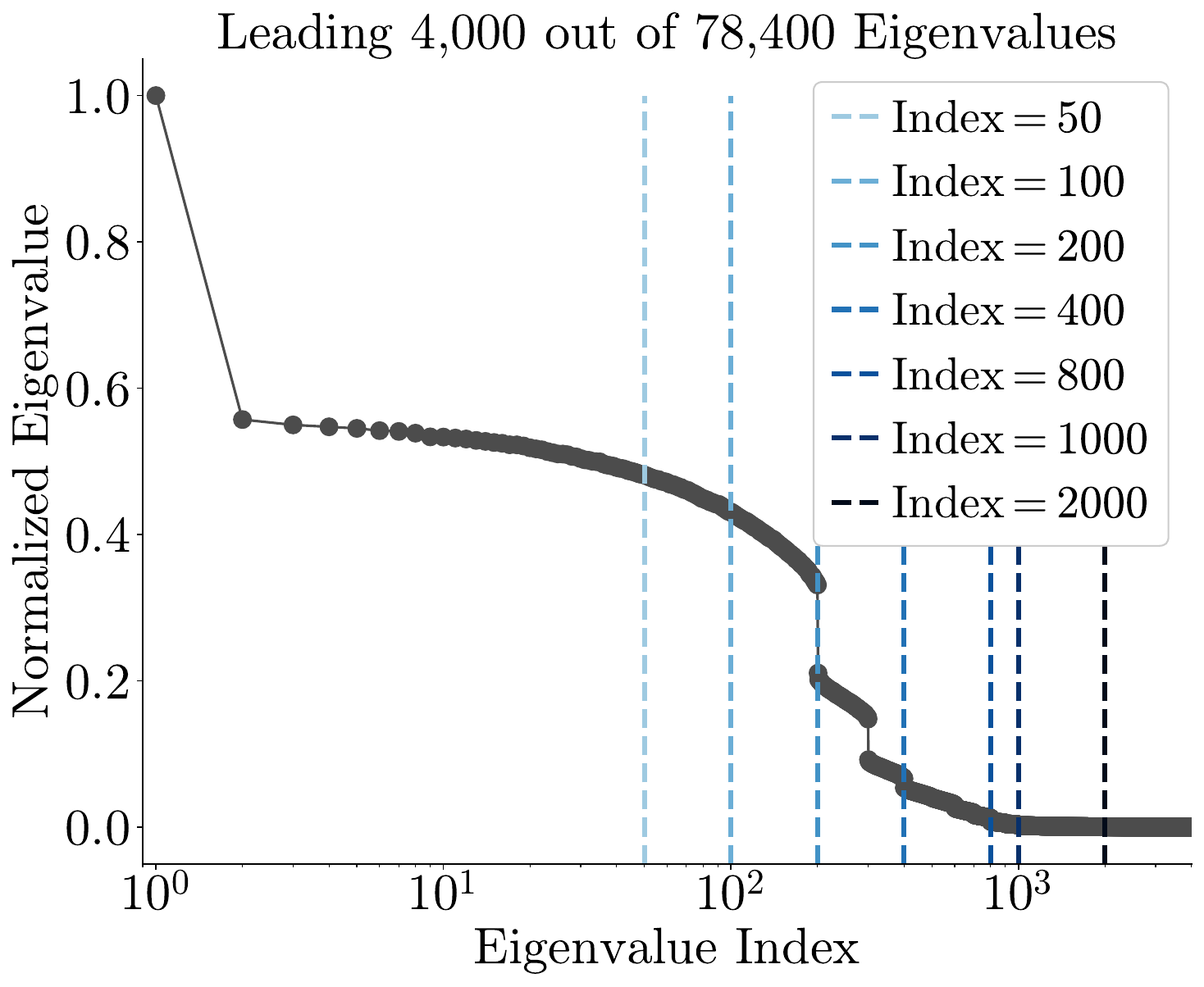}
    \caption{Estimated EGOP eigenvalues for first layer}
\end{subfigure}
\caption{Scaling EGOP reparameterization using block reparameterization and auxiliary variable EGOP reparameterization, as described in Section~\ref{sec:efficient-heuristics}. We train a 2-layer ReLU network, whose layers contain 78.4k weights and 1k weights respectively, to classify fashionMNIST data. We perform auxiliary variable reparameterization using varying numbers of leading EGOP eigenvectors, denoted by $r$. (Left) Auxiliary-variable reparameterization improves loss minimization; while larger values ($r=1000, 2000$) achieve the fastest convergence, even $r=50$ suffices to improve per-epoch convergence compared to original coordinates. (Right) We contextualize these values of $r$ by plotting $\lambda_k/\lambda_1$ for the top 4000 eigenvalues of the first-layer EGOP matrix. This plot shows strong spectral decay and suggests that $r=1000$ components suffice to approximately capture the leading eigenspace of the first-layer EGOP matrix. 
}\label{fig:fashionMNIST}
\end{figure}


\subsubsection{ReLU Networks for Image Classification} We examine the impact of EGOP reparameterization when training ReLU feedforward networks to perform image classification using real-world data. We consider two benchmark image classification datasets: the UCI hand-written digits dataset ($8\times 8$ pixel images), and the fashionMNIST dataset ($28\times 28$ pixel images of clothing) \cite{optical_recognition_of_handwritten_digits_80,xiao2017fashion}. We train multilayer fully-connected ReLU networks to perform image classification on each dataset by minimizing the cross-entropy loss. We use stochastic gradients, corresponding to batch size 300, to estimate gradients during both EGOP estimation and during training. For both image classification tasks, we perform block EGOP reparameterization, as described in Section~\ref{sec:efficient-heuristics}. We define the parameter blocks to correspond to the weight matrices of each layer; for further detail on this procedure, see  Section~\ref{sec:expanded-heuristics-discussion}.

Figure~\ref{fig:opener-cartoon} (right) plots training loss by epoch for the UCI digits classification task. This plot shows that reparameterized adaptive methods converge to lower training loss in fewer epochs compared to their counterparts in original coordinates, and compared to the equivariant baselines SGD and SGD with momentum. When performing EGOP reparameterization for this problem, we set $M$, the number of gradient samples used to estimate the EGOP, equal to the number of model weight parameters. 

Figure~\ref{fig:fashionMNIST} plots training loss by epoch for the fashionMNIST classification task. For this task, in addition to block reparameterization, we employ the auxiliary variable heuristic described in Section~\ref{sec:efficient-heuristics}. We use randomized SVD to approximate the leading $r$ EGOP eigenvectors, where $r \ll d$ is the number of optimization variables. In Figure~\ref{fig:fashionMNIST}, we compare performance for varying values of $r$. We find that even small values of $r$ suffice to improve convergence, and that this benefit increases as $r$ increases. In Section~\ref{ssec:sup-figs-main-experiments}, we demonstrate that these heuristics also lead to faster convergence in terms of wall-clock time for this task.

For both tasks, we report performance on hold-out data in Section~\ref{ssec:sup-figs-main-experiments}, showing that reparameterized methods generalize as well as, or better than, methods in original coordinates. 

\begin{figure}[t]
\centering
\begin{subfigure}{0.47\textwidth}
    \centering
    \includegraphics[width=\linewidth]{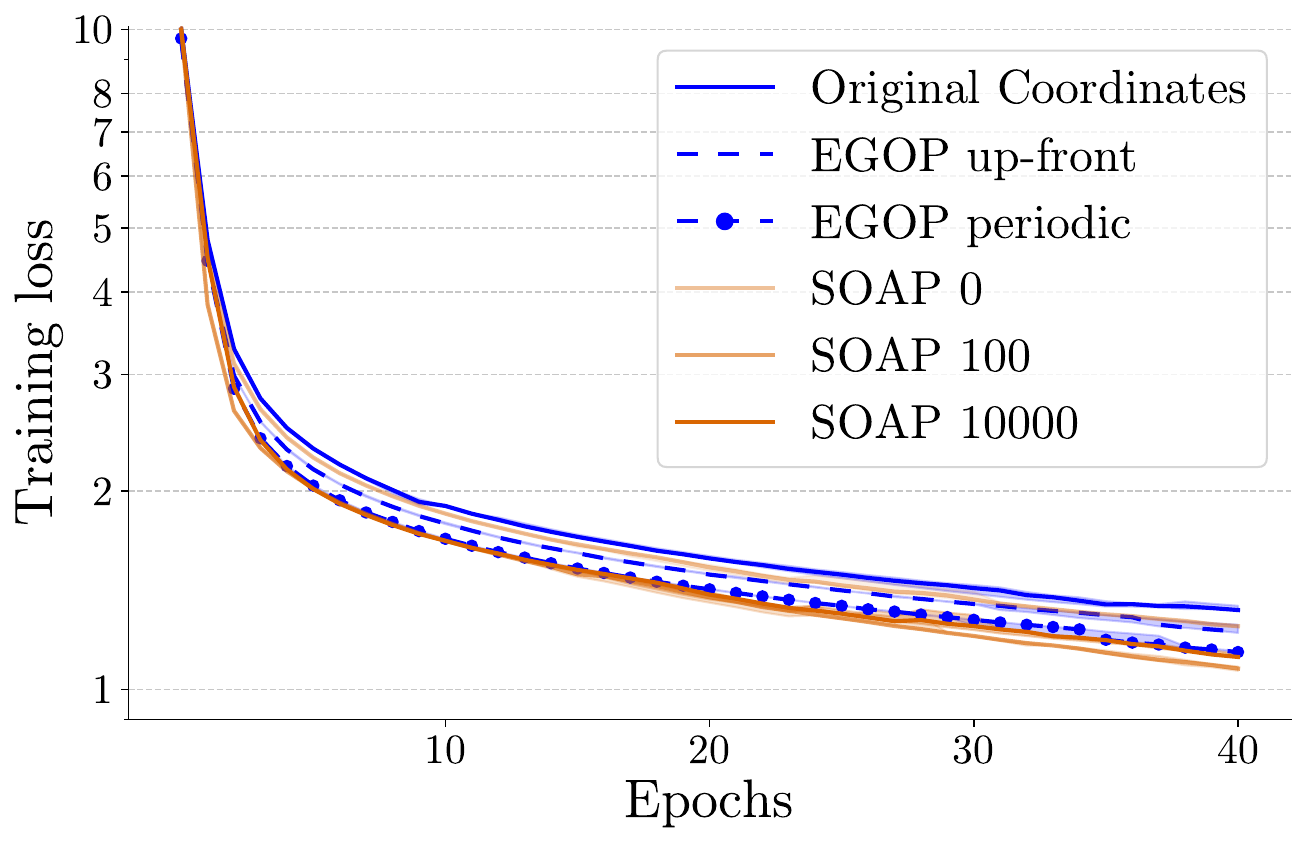}
    \caption{Training loss by epoch}
\end{subfigure}\hfill
\begin{subfigure}{0.52\textwidth}
    \centering
    \fontsize{9}{10}\selectfont
    \begin{tabular}{l l l}
        \toprule
        Method & Epochs & \makecell{Wallclock time\\(hours) } \\
        \midrule
        Original Coordinates & $24 \pm .49$& $5.66\pm 0.11$\\
        EGOP up-front & $15 \pm 0.98$& $4.14 \pm 0.49$\\
        EGOP periodic & $\mathbf{13 \pm .63}$& $\mathbf{3.62\pm .20}$\\
        SOAP 0 & $20\pm 1.60$& $7.06\pm .58$\\
        SOAP 100 & $\mathbf{13\pm .40}$& $4.82\pm .19$\\
        SOAP 10000 & $14\pm .40$& $4.63 \pm .16$\\
        \bottomrule
    \end{tabular}
    \vspace{.6cm}
    \caption{\centering Epochs and wallclock time to achieve validation loss $< 1.3$ (median $\pm$ standard deviation)}\label{table:wallclock-time}
\end{subfigure}
\caption{EGOP reparameterization can improve the convergence of adaptive gradient algorithms for large, modern neural network architectures. We consider training a deep residual network to perform image classification on ImageNet, and compare optimizing the training loss with AdamW in original coordinates (blue, solid trace), under up-front EGOP reparameterization (blue, dashed trace), and with periodic EGOP reparameterization performed once per epoch (blue, dashed trace with dots). We employ heuristics presented in Section~\ref{sec:efficient-heuristics}, particularly block reparameterization, to efficiently scale EGOP reparameterization to this large network. We compare performance with SOAP under different update cadences; see Section~\ref{ssec:main-body-large-scale} for descriptions of SOAP 0, SOAP 100, and SOAP 10000. Compared with SOAP, EGOP reparameterization converges faster in wall-clock time and exhibits competitive per-epoch minimization of training loss.}\label{fig:new-resnet-fig-plus-table}
\end{figure}

\subsubsection{Deep Residual Networks}\label{ssec:main-body-large-scale}

In Figure~\ref{fig:new-resnet-fig-plus-table}, we present results using EGOP repara\-meteriza\-tion to improve convergence in large-scale residual networks (ResNets). We emulate the well-tuned deep ResNet architecture established in He et al. \cite{he2016resnet} and consider the task of image classification using the ImageNet dataset \cite{deng2009imagenet}. We consider training with AdamW, an adaptive algorithm and variant of Adam which implements weight decay decoupled from the momentum computation \cite{loshchilov2019decoupledweightdecayregularization}. We compare AdamW in original coordinates with EGOP-reparameterized AdamW using a single up-front computation of the EGOP reparameterization matrix, as proposed in Algorithm~\ref{alg:meta-algorithm-block}.  We find that up-front EGOP reparameterization improves both minimization of the training and validation loss compared to AdamW in original coordinates. We also consider results using EGOP reparameterization with periodic re-computation of the EGOP eigenbasis, as discussed in Section~\ref{sec:efficient-heuristics}, where the EGOP eigenbasis is re-estimated every epoch. We find that periodic reparameterization further improves training loss minimization over up-front EGOP reparameterization.

We compare EGOP and SOAP convergence times on this large-scale task.
We compute the wall-clock time and epochs required to achieve validation loss $<1.3$, where the wall-clock time for EGOP-reparameterized methods includes the time required to estimate the EGOP eigenbasis. This threshold corresponds to achieving within $5\%$ of the minimum validation loss achieved by SOAP 100. Table~\ref{table:wallclock-time} demonstrates that EGOP-reparameterized AdamW converges faster in wall-clock time compared to SOAP and compared to AdamW optimization in original coordinates. In Section~\ref{ssec:sup-figs-main-experiments} we report validation loss, demonstrating that EGOP-reparameterized AdamW achieves lower validation loss than SOAP and converges at a competitive per-epoch rate.


Our experiments include comparisons with SOAP at a variety of update cadences. In their original work proposing SOAP, Vyas et al. \cite{vyas2024soap} reported that SOAP performance improves when the change-of-basis matrix is periodically updated, and they suggest an update cadence of every 10 to 100 batches. We include results with SOAP 0, which performs a single, up-front calculation of the change-of-basis matrix, comparable with up-front EGOP reparameterization; SOAP 100; and SOAP 10k, which updates the change-of-basis matrix once per epoch, comparable with our implementation of periodic EGOP reparameterization.

These results demonstrate that EGOP reparameterization can offer improvements over state-of-the-art methods in training large neural networks, that the heuristics discussed in Section~\ref{sec:efficient-heuristics} can effectively scale EGOP reparameterization to large optimization problems, and that EGOP reparameterization generalizes to a variety of modern network architectures.

\section{Conclusions and Limitations}\label{sec:conclusions}

In this work, we have shown through both analysis and experiments that EGOP reparameterization can improve the convergence of adaptive algorithms. There may be opportunities for future work to explore other reparameterizations and additional hyperparameter tuning, which we have not examined in this work. One limitation of this work is the engineering required for scalability, as discussed in Section~\ref{sec:efficient-heuristics}. Fully characterizing the trade-off between improved convergence and the up-front cost of reparameterization remains a direction for future research. Another interesting direction for future work would be studying whether the guarantees in Theorem~\ref{lemma:convex-cvgnce-ball-version} can be strengthened under a periodic EGOP reparameterization scheme, as described in Section~\ref{sec:efficient-heuristics}. Our experiments in Section~\ref{sec:experimental-results} demonstrate empirically that such periodic re-estimation of the EGOP eigenbasis can offer further improvements over the up-front reparameterization analyzed in this work. 

Further work examining the tradeoffs between EGOP reparameterization and the reparameterizations employed by SOAP and Shampoo would benefit practitioners by providing greater guidance on when each method may perform best. Our empirical results in Figure~\ref{fig:LRMF} demonstrate that EGOP reparameterization can outperform SOAP and Shampoo even in the presence of exact Kronecker product structure, suggesting that approximate low-rank geometry, as quantified by EGOP spectral decay, is a powerful tool even when additional geometric structure is present.


\section*{Acknowledgments}

We gratefully acknowledge the support of AFOSR FA9550-18-1-0166, NSF DMS-2023109, DOE DE-SC0022232, NSF DMS-2023109, the NSF-Simons National Institute for Theory and Mathematics in Biology (NITMB) through NSF (DMS-2235451) and Simons Foundation (MP-TMPS-00005320), the NSF-Simons AI-Institute for the Sky (SkAI) via grants NSF AST-2421845 and Simons Foundation MPS-AI-00010513, and the Margot and Tom Pritzker Foundation.

\bibliographystyle{plainnat}
\bibliography{ref}

@article{cartis2024learning,
  title={{Learning the subspace of variation for global optimization of functions with low effective dimension}},
  author={Cartis, Coralia and Liang, Xinzhu and Massart, Estelle and Otemissov, Adilet},
  journal={arXiv preprint arXiv:2401.17825},
  year={2024}
}

@article{duchi2011adaptive,
  title={{Adaptive subgradient methods for online learning and stochastic optimization.}},
  author={Duchi, John and Hazan, Elad and Singer, Yoram},
  journal={{Journal of Machine Learning Research}},
  volume={12},
  number={7},
  year={2011}
}

@article{hazan2016introduction,
  title={{Introduction to online convex optimization}},
  author={Hazan, Elad and others},
  journal={{Foundations and Trends{\textregistered} in Optimization}},
  volume={2},
  number={3-4},
  pages={157--325},
  year={2016},
  publisher={Now Publishers, Inc.}
}

@article{cosson2023low,
  title={{Low-Rank Gradient Descent}},
  author={Cosson, Romain and Jadbabaie, Ali and Makur, Anuran and Reisizadeh, Amirhossein and Shah, Devavrat},
  journal={{IEEE Open Journal of Control Systems}},
  year={2023},
  publisher={IEEE}
}

@book{constantine2015active,
  title={{Active subspaces: Emerging ideas for dimension reduction in parameter studies}},
  author={Constantine, Paul G},
  year={2015},
  publisher={SIAM}
}

@article{ling2022vectoradam,
  title={Vectoradam for rotation equivariant geometry optimization},
  author={Ling, Selena Zihan and Sharp, Nicholas and Jacobson, Alec},
  journal={Advances in Neural Information Processing Systems},
  volume={35},
  pages={4111--4122},
  year={2022}
}

@article{golub1973some,
  title={Some modified matrix eigenvalue problems},
  author={Golub, Gene H},
  journal={SIAM Review},
  volume={15},
  number={2},
  pages={318--334},
  year={1973},
  publisher={SIAM}
}

@InProceedings{pmlr-v247-chen24e,
  title = 	 {Open Problem: Black-Box Reductions and Adaptive Gradient Methods for Nonconvex Optimization},
  author =       {Chen, Xinyi and Hazan, Elad},
  booktitle = 	 {Proceedings of Thirty Seventh Conference on Learning Theory},
  pages = 	 {5317--5324},
  year = 	 {2024},
  editor = 	 {Agrawal, Shipra and Roth, Aaron},
  volume = 	 {247},
  series = 	 {Proceedings of Machine Learning Research},
  month = 	 {30 Jun--03 Jul},
  publisher =    {PMLR},
  url = 	 {https://proceedings.mlr.press/v247/chen24e.html}
}

@article{ward2020adagrad,
  title={Adagrad stepsizes: Sharp convergence over nonconvex landscapes},
  author={Ward, Rachel and Wu, Xiaoxia and Bottou, Leon},
  journal={Journal of Machine Learning Research},
  volume={21},
  number={219},
  pages={1--30},
  year={2020}
}

@article{defossez2020simple,
  title={{A Simple Convergence Proof of Adam and Adagrad}},
  author={D{\'e}fossez, Alexandre and Bottou, L{\'e}on and Bach, Francis and Usunier, Nicolas},
  journal={arXiv preprint arXiv:2003.02395},
  year={2020}
}

@misc{xie2024adamexploitsellinftygeometryloss,
      title={{Adam Exploits $\ell_\infty$-geometry of Loss Landscape via Coordinate-wise Adaptivity}}, 
      author={Shuo Xie and Mohamad Amin Mohamadi and Zhiyuan Li},
      year={2024},
      eprint={2410.08198},
      archivePrefix={arXiv},
      primaryClass={cs.LG},
      url={https://arxiv.org/abs/2410.08198}, 
}

@article{jiang2024convergence,
  title={Convergence analysis of adaptive gradient methods under refined smoothness and noise assumptions},
  author={Jiang, Ruichen and Maladkar, Devyani and Mokhtari, Aryan},
  journal={arXiv preprint arXiv:2406.04592},
  year={2024}
}

@article{liu2024adagrad,
  title={{AdaGrad under Anisotropic Smoothness}},
  author={Liu, Yuxing and Pan, Rui and Zhang, Tong},
  journal={arXiv preprint arXiv:2406.15244},
  year={2024}
}

@article{maes2024understanding,
  title={{Understanding Adam Requires Better Rotation Dependent Assumptions}},
  author={Maes, Lucas and Zhang, Tianyue H and Jolicoeur-Martineau, Alexia and Mitliagkas, Ioannis and Scieur, Damien and Lacoste-Julien, Simon and Guille-Escuret, Charles},
  journal={arXiv preprint arXiv:2410.19964},
  year={2024}
}

@inproceedings{gupta2018shampoo,
  title={Shampoo: Preconditioned stochastic tensor optimization},
  author={Gupta, Vineet and Koren, Tomer and Singer, Yoram},
  booktitle={International Conference on Machine Learning},
  pages={1842--1850},
  year={2018},
  organization={PMLR}
}

@article{vyas2024soap,
  title={{SOAP: Improving and stabilizing Shampoo using Adam}},
  author={Vyas, Nikhil and Morwani, Depen and Zhao, Rosie and Shapira, Itai and Brandfonbrener, David and Janson, Lucas and Kakade, Sham},
  journal={arXiv preprint arXiv:2409.11321},
  year={2024}
}

@article{zhao2024galore,
  title={Galore: Memory-efficient {LLM} training by gradient low-rank projection},
  author={Zhao, Jiawei and Zhang, Zhenyu and Chen, Beidi and Wang, Zhangyang and Anandkumar, Anima and Tian, Yuandong},
  journal={arXiv preprint arXiv:2403.03507},
  year={2024}
}

@article{kunstner2024heavy,
  title={{Heavy-tailed class imbalance and why Adam outperforms gradient descent on language models}},
  author={Kunstner, Frederik and Yadav, Robin and Milligan, Alan and Schmidt, Mark and Bietti, Alberto},
  journal={arXiv preprint arXiv:2402.19449},
  year={2024}
}

@article{zhang2024transformers,
  title={{Why transformers need Adam: A Hessian perspective}},
  author={Zhang, Yushun and Chen, Congliang and Ding, Tian and Li, Ziniu and Sun, Ruoyu and Luo, Zhi-Quan},
  journal={arXiv preprint arXiv:2402.16788},
  year={2024}
}

@article{sagun2017empirical,
  title={Empirical analysis of the {H}essian of over-parametrized neural networks},
  author={Sagun, Levent and Evci, Utku and Guney, V Ugur and Dauphin, Yann and Bottou, Leon},
  journal={arXiv preprint arXiv:1706.04454},
  year={2017}
}

@article{papyan2018full,
  title={The full spectrum of deepnet {H}essians at scale: Dynamics with {SGD} training and sample size},
  author={Papyan, Vardan},
  journal={arXiv preprint arXiv:1811.07062},
  year={2018}
}

@article{cui2020active,
  title={Active subspace of neural networks: Structural analysis and universal attacks},
  author={Cui, Chunfeng and Zhang, Kaiqi and Daulbaev, Talgat and Gusak, Julia and Oseledets, Ivan and Zhang, Zheng},
  journal={SIAM Journal on Mathematics of Data Science},
  volume={2},
  number={4},
  pages={1096--1122},
  year={2020},
  publisher={SIAM}
}

@article{zhu2023catapults,
  title={{Catapults in SGD: spikes in the training loss and their impact on generalization through feature learning}},
  author={Zhu, Libin and Liu, Chaoyue and Radhakrishnan, Adityanarayanan and Belkin, Mikhail},
  journal={arXiv preprint arXiv:2306.04815},
  year={2023}
}

@article{mallinar2024emergence,
  title={Emergence in non-neural models: grokking modular arithmetic via average gradient outer product},
  author={Mallinar, Neil and Beaglehole, Daniel and Zhu, Libin and Radhakrishnan, Adityanarayanan and Pandit, Parthe and Belkin, Mikhail},
  journal={arXiv preprint arXiv:2407.20199},
  year={2024}
}

@article{radhakrishnan2022mechanism,
  title={Mechanism of feature learning in deep fully connected networks and kernel machines that recursively learn features},
  author={Radhakrishnan, Adityanarayanan and Beaglehole, Daniel and Pandit, Parthe and Belkin, Mikhail},
  journal={arXiv preprint arXiv:2212.13881},
  year={2022}
}

@article{chou2024gradient,
  title={Gradient descent for deep matrix factorization: Dynamics and implicit bias towards low rank},
  author={Chou, Hung-Hsu and Gieshoff, Carsten and Maly, Johannes and Rauhut, Holger},
  journal={Applied and Computational Harmonic Analysis},
  volume={68},
  pages={101595},
  year={2024},
  publisher={Elsevier}
}

@article{rudelson2007sampling,
  title={Sampling from large matrices: An approach through geometric functional analysis},
  author={Rudelson, Mark and Vershynin, Roman},
  journal={Journal of the ACM (JACM)},
  volume={54},
  number={4},
  pages={21--es},
  year={2007},
  publisher={ACM New York, NY, USA}
}

@misc{kingma2017adammethodstochasticoptimization,
      title={Adam: A Method for Stochastic Optimization}, 
      author={Diederik P. Kingma and Jimmy Ba},
      year={2017},
      eprint={1412.6980},
      archivePrefix={arXiv},
      primaryClass={cs.LG},
      url={https://arxiv.org/abs/1412.6980}, 
}

@misc{reddi2019convergenceadam,
      title={{On the Convergence of Adam and Beyond}}, 
      author={Sashank J. Reddi and Satyen Kale and Sanjiv Kumar},
      year={2019},
      eprint={1904.09237},
      archivePrefix={arXiv},
      primaryClass={cs.LG},
      url={https://arxiv.org/abs/1904.09237}, 
}

@misc{zeiler2012adadeltaadaptivelearningrate,
      title={ADADELTA: An Adaptive Learning Rate Method}, 
      author={Matthew D. Zeiler},
      year={2012},
      eprint={1212.5701},
      archivePrefix={arXiv},
      primaryClass={cs.LG},
      url={https://arxiv.org/abs/1212.5701}, 
}

@misc{loshchilov2019decoupledweightdecayregularization,
      title={Decoupled Weight Decay Regularization}, 
      author={Ilya Loshchilov and Frank Hutter},
      year={2019},
      eprint={1711.05101},
      archivePrefix={arXiv},
      primaryClass={cs.LG},
      url={https://arxiv.org/abs/1711.05101}, 
}

@misc{Crew_2020,
    title={{Google Scholar reveals its most influential papers for 2020}},
    journal={Nature Index},
    publisher={Nature Publishing Group},
    author={Crew, Bec},
    year={2020},
    month={Jul}
}

@article{udell2019big,
  title={Why are big data matrices approximately low rank?},
  author={Udell, Madeleine and Townsend, Alex},
  journal={SIAM Journal on Mathematics of Data Science},
  volume={1},
  number={1},
  pages={144--160},
  year={2019},
  publisher={SIAM}
}

@misc{optical_recognition_of_handwritten_digits_80,
  author       = {Alpaydin, E. and Kaynak, C.},
  title        = {{Optical Recognition of Handwritten Digits}},
  year         = {1998},
  howpublished = {UCI Machine Learning Repository},
  note         = {{DOI}: https://doi.org/10.24432/C50P49}
}

@article{xiao2017fashion,
  title={Fashion-{MNIST}: a novel image dataset for benchmarking machine learning algorithms},
  author={Xiao, Han and Rasul, Kashif and Vollgraf, Roland},
  journal={arXiv preprint arXiv:1708.07747},
  year={2017}
}

@inproceedings{glorot2010understanding,
  title={Understanding the difficulty of training deep feedforward neural networks},
  author={Glorot, Xavier and Bengio, Yoshua},
  booktitle={Proceedings of the thirteenth international conference on artificial intelligence and statistics},
  pages={249--256},
  year={2010},
  organization={JMLR Workshop and Conference Proceedings}
}

@inproceedings{paszke2017automatic,
  title={{Automatic differentiation in PyTorch}},
  author={Paszke, Adam and Gross, Sam and Chintala, Soumith and Chanan, Gregory and Yang, Edward and DeVito, Zachary and Lin, Zeming and Desmaison, Alban and Antiga, Luca and Lerer, Adam},
  booktitle={NIPS-W},
  year={2017}
}

@software{jax2018github,
  author = {James Bradbury and Roy Frostig and Peter Hawkins and Matthew James Johnson and Chris Leary and Dougal Maclaurin and George Necula and Adam Paszke and Jake Vander{P}las and Skye Wanderman-{M}ilne and Qiao Zhang},
  title = {{JAX}: composable transformations of {P}ython+{N}um{P}y programs},
  url = {http://github.com/jax-ml/jax},
  version = {0.3.13},
  year = {2018},
}

@software{deepmind2020jax,
  title = {The {D}eep{M}ind {JAX} {E}cosystem},
  author = {DeepMind and Babuschkin, Igor and Baumli, Kate and Bell, Alison and Bhupatiraju, Surya and Bruce, Jake and Buchlovsky, Peter and Budden, David and Cai, Trevor and Clark, Aidan and Danihelka, Ivo and Dedieu, Antoine and Fantacci, Claudio and Godwin, Jonathan and Jones, Chris and Hemsley, Ross and Hennigan, Tom and Hessel, Matteo and Hou, Shaobo and Kapturowski, Steven and Keck, Thomas and Kemaev, Iurii and King, Michael and Kunesch, Markus and Martens, Lena and Merzic, Hamza and Mikulik, Vladimir and Norman, Tamara and Papamakarios, George and Quan, John and Ring, Roman and Ruiz, Francisco and Sanchez, Alvaro and Sartran, Laurent and Schneider, Rosalia and Sezener, Eren and Spencer, Stephen and Srinivasan, Srivatsan and Stanojevi\'{c}, Milo\v{s} and Stokowiec, Wojciech and Wang, Luyu and Zhou, Guangyao and Viola, Fabio},
  url = {http://github.com/google-deepmind},
  year = {2020},
}

@ARTICLE{2020SciPy-NMeth,
  author  = {Virtanen, Pauli and Gommers, Ralf and Oliphant, Travis E. and
            Haberland, Matt and Reddy, Tyler and Cournapeau, David and
            Burovski, Evgeni and Peterson, Pearu and Weckesser, Warren and
            Bright, Jonathan and {van der Walt}, St{\'e}fan J. and
            Brett, Matthew and Wilson, Joshua and Millman, K. Jarrod and
            Mayorov, Nikolay and Nelson, Andrew R. J. and Jones, Eric and
            Kern, Robert and Larson, Eric and Carey, C J and
            Polat, {\.I}lhan and Feng, Yu and Moore, Eric W. and
            {VanderPlas}, Jake and Laxalde, Denis and Perktold, Josef and
            Cimrman, Robert and Henriksen, Ian and Quintero, E. A. and
            Harris, Charles R. and Archibald, Anne M. and
            Ribeiro, Ant{\^o}nio H. and Pedregosa, Fabian and
            {van Mulbregt}, Paul and {SciPy 1.0 Contributors}},
  title   = {{{SciPy} 1.0: Fundamental Algorithms for Scientific
            Computing in Python}},
  journal = {Nature Methods},
  year    = {2020},
  volume  = {17},
  pages   = {261--272},
  adsurl  = {https://rdcu.be/b08Wh},
  doi     = {10.1038/s41592-019-0686-2},
}

@article{halko2011finding,
  title={Finding structure with randomness: Probabilistic algorithms for constructing approximate matrix decompositions},
  author={Halko, Nathan and Martinsson, Per-Gunnar and Tropp, Joel A},
  journal={SIAM {R}eview},
  volume={53},
  number={2},
  pages={217--288},
  year={2011},
  publisher={SIAM}
}

@article{dong2025towards,
  title={Towards quantifying the {H}essian structure of neural networks},
  author={Dong, Zhaorui and Zhang, Yushun and Yao, Jianfeng and Sun, Ruoyu},
  journal={arXiv preprint arXiv:2505.02809},
  year={2025}
}

@INPROCEEDINGS{deng2009imagenet,
  author={Deng, Jia and Dong, Wei and Socher, Richard and Li, Li-Jia and Kai Li and Li Fei-Fei},
  booktitle={2009 IEEE Conference on Computer Vision and Pattern Recognition}, 
  title={ImageNet: A large-scale hierarchical image database}, 
  year={2009},
  volume={},
  number={},
  pages={248-255},
  doi={10.1109/CVPR.2009.5206848}}

@InProceedings{he2016resnet,
author = {He, Kaiming and Zhang, Xiangyu and Ren, Shaoqing and Sun, Jian},
title = {Deep Residual Learning for Image Recognition},
booktitle = {Proceedings of the IEEE Conference on Computer Vision and Pattern Recognition (CVPR)},
month = {June},
year = {2016}
}

\pagebreak

\appendix

\section{Proofs}\label{sec:deferred-proofs}

\paragraph{Notation for proofs} Throughout this section, we denote the unit sphere in $d$ dimensions by $\Sbb^{d-1}$ and the Euclidean ball of radius $r$ centered at $\bar{x}$ by $\mathcal{B}(\bar{x}; r)$. Given a matrix $A \in \R^{m \times d}$, we write $\frobnorm{A} := \sqrt{\ip{A, A}}$ for its \emph{Frobenius} norm and $\opnorm{A} = \sup_{x \in \Sbb^{d-1}} \norm{Ax}$ for its \emph{spectral} norm. For matrices $A, B \in \R^{d\times d}$, we denote the L\"owner ordering by $A\preceq B$. We write $A \lesssim B$ to indicate the existence of a dimension-independent positive constant $c > 0$ such that $A \leq cB$.  We write $\mathcal{O}(d, r) := \set{X \in \R^{d \times r} \mid X^{\T} X = I_{r}}$ and $\mathcal{O}(d) \equiv \mathcal{O}(d, d)$ for the set of matrices with orthogonal columns.

We write $\ip{X, Y} \defeq \tr(X^{\T} Y)$ for the Euclidean inner product and $\norm{X} = \sqrt{\ip{X, X}}$ for its induced norm. Given a PSD matrix $H$ we write $\ip{X, Y}_{H}$ for the weighted inner product $\ip{HX, Y} = \ip{H^{1/2} X, H^{1/2} Y}$.  We write $\mathcal{O}(d)$ for the set of matrices with orthogonal columns. Given a set of scalars $\{L_i\}_{i=1}^d$, we let $\diag(L_1,\dots,L_d)$ denote the $d\times d$ diagonal matrix with values $L_1,\dots, L_d$ along the diagonal.  We let $\mat(\theta)\in \R^{m\times n}$ denote the reshaping of $\theta$ into a matrix for some $m$ and $n$ such that $m n = d$, and similarly we let $\vec{M}$ denote the reshaping of a matrix into a vector. 

Following the definitions introduced in Section~\ref{sec:EGOP-defn}, we let $V$ denote the eigenbasis of the EGOP matrix and let $\tf\defeq f\circ V$. Let $\tTheta$ denote the corresponding transformation of set $\Theta$: $\tTheta\defeq \{V^\T \theta \mid \theta\in \Theta\}$. We denote the corresponding diameter measurements as $\tilde{D}_p \defeq \max_{\ttheta_1,\ttheta_2\in \tTheta}\shortnorm{\ttheta_1 - \ttheta_2}_p$. Let $\{L_i\}_{i=1}^d$ denote the values for which $f(\cdot)$ is coordinate-wise smooth within $\Theta$, following the definition in Eq.~\ref{eq:def-coordinate-wise-smoothness}, and let $\{\tilde{L}_i\}_{i=1}^d$ denote the analogous coordinate-wise smoothness constants of $\tf(\cdot)$ within $\tTheta$. We denote the vectors of coordinate-wise smoothness constants in original and reparameterized coordinates respectively by $\vec{L}\in \R^d$ and $\vec{\tilde{L}}\in \R^d$.

We analyze the convergence of Adagrad in both convex and nonconvex settings. For completeness, we recall the full Adagrad algorithm in Algorithm~\ref{alg:Adagrad}.

\begin{algorithm}[h]
    \caption{Adagrad$(f, g,\theta_0, T,\epsilon: \texttt{default} = 10^{-8}, \Theta: \texttt{Optional})$}\label{alg:Adagrad}
    \begin{algorithmic}[1]
    \State {\bfseries Input:} Objective $f$, oracle gradient $g$, $\theta_0 \in \R^d$, $\eta$ step size, $T$ number of iterations, $\epsilon$ numerical stability parameter, $\Theta\subseteq \R^d$ constraint set.
    \State $v_{0}\gets \epsilon^2 \mathbb{I}_d$
    \For{$t\in [1,T]$}
        \State $g_t \gets g(f, \theta_t)$
        \State $v_t(i) \gets v_{t-1}(i) + g_t(i)^2\ \forall i \in [d]$
        \State $H_t \gets \operatorname{diag}(v_t)^{1/2}$
        \State $w_t \gets \theta_t - \eta H_t^{-1}g_t$
        \State \textbf{Constrained update:} $\theta_{t+1} = \Pi^{H_t}_{\Theta}(w_t)$
        \State \textbf{Unconstrained update:} $\theta_{t+1} = w_t$
    \EndFor
    \State {\bfseries Return:} $\theta_T \in \R^d$
    \end{algorithmic}
\end{algorithm}

\subsection{Proofs from Section~\ref{sec:convergence-analysis}}

Several recent works have proved convergence guarantees for Adagrad in terms of $\shortnorm{\vec{L}}_1$, showing that Adagrad enjoys better performance guarantees in settings when  $\shortnorm{\vec{L}}_1$ is small \cite{jiang2024convergence, liu2024adagrad, xie2024adamexploitsellinftygeometryloss}. Our main results establish that when $\EGOP(f)$ has strong spectral decay and dense leading eigenvectors, then reparameterization can produce $\shortnorm{\vec{\tilde{L}}}_1$ significantly smaller than $\shortnorm{\vec{L}}_1$, implying better convergence bounds. 

\begin{theorem}\label{thm:smoothness-constants-ratio}
    If $\exists \delta \in [0,\beta^2_1)$ such that the Hessian of $f(\cdot)$ is $H$-Lipschitz within $\Theta$ for $H$ satisfying
    \begin{equation}
        H \leq \frac{\sqrt{\delta \lambda_1(\EGOP) + D_2^2 \lambda^2_1(\nabla^2 f(\theta^*))} - D_2 \lambda_1(\nabla^2 f(\theta^*))}{D^2_2}
    \end{equation}
    Then within $\Theta$ and $\tTheta$ respectively, the smoothness constants of $f(\cdot)$ and $\tf(\cdot)$ satisfy
    \[
        \frac{\shortnorm{\vec{\tilde{L}}}_1}{\shortnorm{\vec{L}}_1} = O\left( \sqrt{1+\delta}\left(\frac{\sr(f)}{d(\beta_1 -\delta)} + \frac{\sqrt{\delta}}{\beta_1-\delta}\right)\right)
    \]
    where $\sr(f)$ denotes the \textit{stable rank} of $f(\cdot)$.
\end{theorem}
Theorem~\ref{thm:smoothness-constants-ratio} implies that when the EGOP leading eigenvectors have constant density, i.e. when $\beta_1 >> 1/d$, then for small $\delta$ the ratio $\shortnorm{\vec{\tilde{L}}}_1/\shortnorm{\vec{L}}_1$ scales as $\sr(f)/d\beta_1$. We note that the density condition $\beta \gg 1/d$ is satisfied with high probability for random unit vectors in high dimensions: in particular, for $v\sim\textrm{Unif}(S^{d-1})$ where $S^{d-1}$ is the unit sphere in $\R^d$, with high probability $\norm{v}^2_1/d \approx 2/\pi > 0.6$. 

We prove Theorem~\ref{thm:smoothness-constants-ratio} by establishing the following pair of claims. The first lower bounds the sum of the coordinate-wise smoothness constants of $f(\cdot)$ in original coordinates:
\begin{lemma}\label{lem:general-LB}
    Consider $f(\cdot)$, $\Theta$, $\rho$ satisfying Assumptions~\ref{assumption:rho-and-Theta} and \ref{assumption:H-Lipschitz} and consider the set of dense unit vectors $\nu \in \{\pm d^{-1/2}\}^d$, the collection of vectors whose entries all have magnitude $\abs{\nu(i)} = d^{1/2}$. Then within $\Theta$, $f(\cdot)$ is coordinate-wise smooth with respect to constants satisfying
    \[
        \shortnorm{\vec{L}}_1 \geq \frac{d}{2c}\cdot \max_{\nu \in \{\pm d^{-1/2}\}^d}\frac{\ip{\nu, \EGOP\nu} - \gamma}{\sqrt{\lambda_{\max}(\operatorname{EGOP}) + \gamma}}
    \]
    where
    \begin{equation}\label{eq:def-gamma}
        \gamma \defeq  2H \lambda_{\max}(\nabla^2 f(\theta^*)) M_3 + H^2 M_4
    \end{equation}
    where $H$ is the Lipschitz constant of the Hessian of $f(\cdot)$ in Assumption~\ref{assumption:H-Lipschitz} and $M_p\defeq \mathbb{E}_{\theta\sim\rho}[\norm{\theta-\theta^*}^p_2]$.
\end{lemma}
If the EGOP has dense leading eigenvectors, then the term $\langle\nu, \EGOP \nu\rangle$ is large. In particular, Lemma~\ref{lem:general-LB} implies the following:
\begin{corollary}\label{claim:OG-coor-LB}
    For $f(\cdot)$, $\Theta$,$\rho$ satisfying Assumptions~\ref{assumption:rho-and-Theta} and \ref{assumption:H-Lipschitz}, the smoothness constants of $f$ satisfy
    \[
        \shortnorm{\vec{L}}_1 \geq \frac{d}{2c}\frac{\beta_k\lambda_k(\operatorname{EGOP}) - \gamma}{\sqrt{\lambda_{\max}(\operatorname{EGOP}) + \gamma}}.
    \]
    where $\beta_k \defeq \norm{v_k}^2_1/d$ for $\lambda_k(\operatorname{EGOP}), v_k$ the $k$\ts{th} eigenvalue and eigenvector of $\EGOP(f)$, and $\gamma$ defined in Eq.~\ref{eq:def-gamma}.
\end{corollary}

In contrast, we show that under  reparameterization by the EGOP eigenbasis, the smoothness constants can be upper bounded by the following:
 
\begin{lemma}\label{claim:reparam-coor-UB}
    Let $V$ be the eigenbasis of the EGOP of $f(\cdot)$ with respect to $\rho$, and define $\tf(\ttheta)\defeq f(V\theta)$. Let the function $f(\cdot)$ satisfy Assumptions~\ref{assumption:rho-and-Theta} and \ref{assumption:H-Lipschitz}. Then within $\tTheta$, $\tf(\cdot)$ satisfies Eq.~\ref{eq:def-coordinate-wise-smoothness} with respect to smoothness coordinates whose sum is bounded by
    \[
        \shortnorm{\vec{\tilde{L}}}_1 \leq d\left(\frac{\sqrt{\gamma}}{c} + H D_2\right)+ \frac{1}{c}\sum_{i=1}^d \sqrt{\lambda_i}
    \]
    where $\gamma$ is defined in Eq.~\ref{eq:def-gamma}. and $\lambda_i$ denotes the $i$th eigenvalue of the EGOP of $f$ with respect to $\rho$.
\end{lemma}

Note that as Lipschitz constant of Hessian in Assumption~\ref{assumption:H-Lipschitz} goes to zero, so does the value of $\gamma$ in both Lemmas~\ref{claim:OG-coor-LB} and \ref{claim:reparam-coor-UB}. Theorem~\ref{thm:smoothness-constants-ratio} follows from Lemmas~\ref{claim:OG-coor-LB} and \ref{claim:reparam-coor-UB}. In order to prove these lemmas, we first establish some intermediate lemmas.

\paragraph{Helper Lemmas} 

The first intermediate result gives an alternative characterization of the coordinate wise smoothness constants defined in Eq.~\ref{eq:def-coordinate-wise-smoothness}.

\begin{lemma}\label{lem:Hessian-coor-wise-smooth}
    A twice-differentiable function $f:\R^d\rightarrow \R$ satisfies Eq.~\ref{eq:def-coordinate-wise-smoothness} with respect to smoothness constants $L_1,\dots,L_d$ within $\Theta$ if and only if  $\forall \theta\in \Theta, \forall v\in \R^d$,
    \[
        \abs{\ip{v, \nabla^2 f(\theta) v}} \leq \ip{v, \diag(L) v}
    \]
    where $\diag(L) \in \R^{d\times d}$ is the diagonal matrix with diagonal entries $L_i$.
\end{lemma}
\begin{proof}[Proof of Lemma~\ref{lem:Hessian-coor-wise-smooth}]
    By definition, $f$ is twice-differentiable if and only if it satisfies
    \begin{equation}
        f(y) = f(x) + \ip{\grad f(x), y - x} + \ip{y - x, \grad^2 f(x)(y - x)} + o(\|y - x\|^2),
        \quad \text{for all $x, y \in \R^d$.}
        \label{eq:2nd-order-frechet}
    \end{equation}
    We first prove the ``if'' version: assume $\abs{\ip{v,\nabla^2 f(\theta)v}} \leq \ip{v,\diag(L)v}$. Consider any $x,y\in \Theta$. By the mean value theorem,
    \begin{align*}
        f(y) - f(x) - \ip{\grad f(x), y - x} &=
        \int_{0}^1 t \ip{y - x, \grad^2 f(x + t(y - x)) (y - x)} \mathrm{d}t \\
        &\leq
        \int_{0}^1 t \ip{y - x, \diag(L)(y - x)} \mathrm{d}t \\
        &=
        \ip{y - x, \diag(L)(y - x)},
    \end{align*}
    where the inequality follows by assumption and the convexity of $\Theta$. Taking absolute values on both sides yields the claim.

    We now prove the ``only if'' part. In particular, for any $x, y\in \Theta$, writing $y = x + tv$ and invoking~\eqref{eq:2nd-order-frechet} we obtain
    \begin{align*}
        \abs{\ip{y - x, \grad^2 f(x)(y - x)}} &\leq
        \abs{f(y) - f(x) - \ip{\grad f(x), y - x}} +
        o(\|y - x\|^2) \\
        &\leq
        \ip{y - x, \diag(L)(y - x)} + o(\|y - x\|^2) \\
        \Leftrightarrow
        t^2 \abs{\ip{v, \grad^2 f(x) v}} &\leq
        t^2 \ip{v, \diag(L)v} + o(t^2).
    \end{align*}
    Because $t$ was chosen such that $y = x+tv$ for $x,y\in \Theta$, convexity of $\Theta$ implies that $x+\tau v \in \Theta$ for all $\tau \in [0,t]$. Thus dividing both sides by $t$ and letting $t \downarrow 0$ implies the result.
\end{proof}

The second intermediate result relates the EGOP to the Hessian.

\begin{lemma}\label{lem:EGOP-equation}
    Consider twice-differentiable function $f:\R^d \rightarrow \R$ and sampling distribution $\rho$ satisfying Assumptions~\ref{assumption:H-Lipschitz}, and \ref{assumption:rho-and-Theta} with respect to local minimum $\theta^*$. Then the EGOP of $f(\cdot)$ with respect to $\rho$ satisfies
    \[
        \mathbb{E}_{\theta\sim \rho}[\nabla f(\theta)\nabla f(\theta)^\T] = G(\theta^*) + E_f(\theta^*)
    \]
    where $G(\theta^*)$ is a PSD matrix satisfying
    \[
        c^2 \nabla^2 f(\theta^*)\nabla^2 f(\theta^*)^\T \preceq G(\theta^*) \preceq 2c^2 \nabla^2 f(\theta^*)\nabla^2 f(\theta^*)^\T
    \]
    and the matrix $E_f(\theta^*)$ satisfies
    \[
    \abs{\ip{v, E_f (\theta^*) v}} \leq \left(2H \lambda_{\max}(\nabla^2 f(\theta^*)) M_3 + H^2 M_4\right)\norm{v}^2_2
   \]
   where
   \[
    M_p \defeq \mathbb{E}_{\theta\sim\rho}[\norm{\theta-\theta^*}^{p}_2].
   \]
\end{lemma}

\begin{proof}[Proof of Lemma~\ref{lem:EGOP-equation}]
    Given $f(\cdot)$ satisfying Assumption~\ref{assumption:H-Lipschitz}, $\forall \theta_1, \theta_2 \in \Theta$
    \[
        \norm{\nabla f(\theta_1) - \nabla f(\theta_2) - \nabla^2 f(\theta_2)(\theta_1-\theta_2)}_2 \leq H\norm{\theta_1 - \theta_2}^2_2.
    \]
    For fixed $\theta^*$ the local minimum in Assumption~\ref{assumption:rho-and-Theta}, this implies that $\forall\theta\in \Theta$ $\exists R_f(\theta)\in \R^d$ such that 
    \[
        \nabla f(\theta) = \nabla f(\theta^*) + \nabla^2 f(\theta^*)(\theta-\theta^*) + R_f(\theta) \quad \text{ and }\quad \norm{R_f(\theta)}_2 \leq H\norm{\theta-\theta^*}^2_2.
    \]
    Given $\theta^*$ a stationary point of $f(\cdot)$, $\nabla f(\theta^*) = 0$, so the above simplifies to
    \[
     \nabla f(\theta) = \nabla^2 f(\theta^*)(\theta-\theta^*) + R_f(\theta).
    \]
    We can substitute the above expression into the definition of the EGOP matrix:
    \begin{align*}
       \mathbb{E}_{\theta\sim \rho}[\nabla f(\theta)\nabla f(\theta)^\T] &= \mathbb{E}_{\theta\sim \rho}\left[\left( \nabla^2 f(\theta^*)(\theta-\theta^*) + R_f(\theta)\right)\left( \nabla^2 f(\theta^*)(\theta-\theta^*) + R_f(\theta)\right)^\T  \right]\\
       &=\mathbb{E}_{\theta\sim \rho}\left[\nabla^2 f(\theta^*)(\theta-\theta^*)(\theta-\theta^*)^\T \nabla^2 f(\theta^*)^\T\right] + E_f(\theta^*)
    \end{align*}
    where
    \[
        E_f(\theta^*) \defeq \mathbb{E}_{\theta\sim \rho}[R_f(\theta)(\theta-\theta^*)^\T\nabla^2 f(\theta^*)^\T + \nabla^2 f(\theta^*)(\theta-\theta^*)R_f(\theta)^\T + R_f(\theta) R_f(\theta)^\T].
    \]

   We can simplify the first term by leveraging the fact that $\rho$ is isotropic about $\theta_c$, as specified in Assumption~\ref{assumption:rho-and-Theta}: 
   \begin{align*}
       &\mathbb{E}_{\theta\sim \rho}\left[\nabla^2 f(\theta^*)(\theta-\theta^*)(\theta-\theta^*)^\T \nabla^2 f(\theta^*)^\T\right]= \nabla^2 f(\theta^*) \mathbb{E}_{\theta\sim \rho}\left[(\theta-\theta^*)(\theta-\theta^*)^\T\right] \nabla^2 f(\theta^*)^\T\\
       &\quad =\nabla^2 f(\theta^*) \mathbb{E}_{\theta\sim \rho}\left[(\theta-\theta_c + (\theta_c-\theta^*))(\theta-\theta_c+(\theta_c-\theta^*))^\T\right] \nabla^2 f(\theta^*)^\T\\
       &\quad =\nabla^2 f(\theta^*) \left(\mathbb{E}_{\theta\sim \rho}\left[(\theta_c-\theta)(\theta_c-\theta)^\T\right]+ (\theta_c-\theta^*) (\theta_c-\theta^*)^\T\right) \nabla^2 f(\theta^*)^\T\\
       &\quad =\nabla^2 f(\theta^*) \left(c^2 \mathbb{I}+ (\theta_c-\theta^*) (\theta_c-\theta^*)^\T\right) \nabla^2 f(\theta^*)^\T
   \end{align*}
   where the penultimate equality follows from the assumption that $\rho$ has mean $\theta_c$ and the last equality follows from the assumption that $\rho$ is isotropic about $\theta_c$. Because $\theta^*$ is a local minimum, $\nabla^2 f(\theta^*)$ is PSD, and further the rank-1 outer product  is also PSD,  so we have
   \[
    \nabla^2 f(\theta^*) \left(c^2 \mathbb{I}+ (\theta_c-\theta^*) (\theta_c-\theta^*)^\T\right) \nabla^2 f(\theta^*)^\T \succeq c^2 \nabla^2 f(\theta^*)\nabla^2 f(\theta^*)^\T.
   \]
   Moreover by Assumption~\ref{assumption:rho-and-Theta}, $\norm{\theta_c-\theta^*}^2_2\leq c^2$, and thus
   \[
        (\theta_c-\theta^*) (\theta_c-\theta^*)^\T \preceq c^2 \mathbb{I}.
   \]
   This implies
   \[
    \nabla^2 f(\theta^*) \left(c^2 \mathbb{I}+ (\theta_c-\theta^*) (\theta_c-\theta^*)^\T\right) \nabla^2 f(\theta^*)^\T \preceq 2c^2 \nabla^2 f(\theta^*)\nabla^2 f(\theta^*)^\T.
   \]
   Combining these two matrix inequalities yields the matrix inequalities for 
   \[
    G(\theta^*) \defeq \nabla^2 f(\theta^*) \left(c^2 \mathbb{I}+ (\theta_c-\theta^*) (\theta_c-\theta^*)^\T\right) \nabla^2 f(\theta^*)^\T.
   \]

   Lastly, we bound $\abs{\ip{v, E_f(\theta^*)v}}$ for $E_f(\theta^*)$ defined above. Applying triangle inequality, Jensen's inequality,  and Cauchy-Schwarz, to the definition of $E_f(\theta^*)$ yields
   \begin{align*}
       \abs{\ip{v, E_f (\theta^*) v}} &\leq \mathbb{E}_{\theta\sim \rho} \left[ 2\abs{\ip{v, R_f(\theta)}}\cdot \abs{\langle \theta-\theta^*, \nabla^2 f(\theta^*) v\rangle}+ \abs{\ip{v, R_f(\theta)}^2}\right]\\
       &\leq \mathbb{E}_{\theta\sim \rho} \left[ 2 \norm{R_f(\theta)}_2 \norm{\theta-\theta^*}_2 \lambda_{\max}(\nabla^2 f(\theta^*))\norm{v}^2_2 + \norm{R_f(\theta)}^2_2\norm{v}^2_2\right].
   \end{align*}
   Recall that under Assumption~\ref{assumption:H-Lipschitz}, $\norm{R_f(\theta)}_2 \leq H\norm{\theta-\theta^*}^2_2.$ Substituting this into the above bound yields
   \[
    \abs{\ip{v, E_f (\theta^*) v}} \leq \left(2H \lambda_{\max}(\nabla^2 f(\theta^*)) M_3 + H^2 M_4\right)\norm{v}^2_2
   \]
   where
   \[
    M_p \defeq \mathbb{E}_{\theta\sim\rho}[\norm{\theta-\theta^*}^{p}_2].
   \]
\end{proof}

We'll also use the following basic results from linear algebra, which we prove here for completeness.
\begin{lemma}\label{lemma:relate-square-mat-to-self}
    Given $A \in \R^{d\times d}$ a symmetric PSD matrix, for any vector $v\in \R^d$, 
    \[
        \ip{v, A^\T A v} \leq \lambda_{\max}(A)\ip{v, A v}
    \]
    where $\lambda_{\max}(A)$ denotes the largest eigenvalue of $A$.
\end{lemma}
\begin{proof}[Proof of Lemma~\ref{lemma:relate-square-mat-to-self}]
    Since $A$ is PSD, it admits a square root $A^{1/2}$ which is itself PSD. Consequently,
    \[
        \ip{Av, Av} =
        \norm{Av}^2 \leq
        \opnorm{A^{1/2}}^2 \| A^{1/2} v \|^2 =
        \lambda_{\max}(A) \ip{v, Av}.
    \]
\end{proof}

\begin{lemma}[Order-preserving square root]\label{lemma:order-preserving-square-root}
For any positive-semidefinite matrices $X, Y$, the following holds:
\begin{equation}
    X \preceq Y \implies X^{1/2} \preceq Y^{1/2}.
    \label{eq:sqrt-order-preserving}
\end{equation}
\end{lemma}
\begin{proof}
    Since $Y - X \succeq 0$, any unit vector $v$ satisfies
    \[
    0 \leq \ip{v, (Y - X) v} = \ip{v, (Y^{1/2} + X^{1/2})(Y^{1/2} - X^{1/2})v}.
    \]
    Since $X, Y$ are PSD, their square roots are well-defined and
    symmetric. In particular, the matrix $Y^{1/2} - X^{1/2}$ is symmetric and, as such, has real eigenvalues. Therefore, it suffices to prove all eigenvalues of this matrix are nonnegative.

    To that end, let $(\lambda, v)$ be an eigenpair of $Y^{1/2} - X^{1/2}$. We have
    \begin{equation}
    0 \leq \ip{(Y^{1/2} - X^{1/2})v, (Y^{1/2} + X^{1/2}) v} =
    \lambda \ip{v, (Y^{1/2} + X^{1/2}) v}.
    \label{eq:contradiction-stmt}
    \end{equation}
    In particular, assume $\lambda < 0$ (since otherwise there is nothing to show). Consequently,
    \[
    0 > \lambda = \ip{v, (Y^{1/2} - X^{1/2})v} \implies
    \ip{v, X^{1/2} v} > \ip{v, Y^{1/2} v} \geq 0,
    \]
    where the last inequality follows from the fact that $Y$ is PSD.
    In particular, this implies that $\ip{v, (Y^{1/2} + X^{1/2}) v} > 0$;
    combined with~\eqref{eq:contradiction-stmt} and the assumption
    $\lambda < 0$, this yields a contradiction. Therefore, $\lambda \geq 0$. Since the choice of eigenpair was arbitrary, the claim
    follows.
\end{proof}

\paragraph{Dense EGOP eigenvectors lead to large $\mathbf{||\vec{L}||_1}$} Given the above helper lemmas, we are now equipped to prove Lemma~\ref{claim:OG-coor-LB}.
\begin{proof}[Proof of Lemma~\ref{claim:OG-coor-LB}]
    By Lemma~\ref{lem:Hessian-coor-wise-smooth},  the smoothness constants $L_1,\dots,L_d$ must satisfy
    \[
        \forall \theta \in \Theta, \forall v, \quad \ip{v, \diag(L) v} \geq \abs{\ip{v, \nabla^2 f(\theta) v}}.
    \]
    where $\diag(L)\in \R^{d\times d}$ is the diagonal matrix with entries $\diag(L)_{i,i} = L_i$. In particular, consider $\theta=\theta^*$, and let $\nu \in \{\pm d^{-1/2}\}^d$ be a dense unit vector whose entries all satisfy $\abs{\nu(i)} = d^{-1/2}$. Then the smoothness constants must satisfy
    \[
        \ip{\nu, \diag(L)\nu} =\frac{1}{d}\sum_{i=1}^d L_i \geq \ip{\nu, \nabla^2 f(\theta^*) \nu}.
    \]
    Given $\theta^*$ a local minimum, $\nabla^2 f(\theta^*)$ is PSD. Thus by Lemma~\ref{lemma:relate-square-mat-to-self}, 
    \[
        \ip{\nu, \nabla^2 f(\theta^*) \nu} \geq \frac{1}{\lambda_{\max}(\nabla^2 f(\theta^*))}\ip{\nu, \nabla^2 f(\theta^*)\nabla^2 f(\theta^*)^\T \nu}
    \]
    where $\lambda_{\max}(\nabla^2 f(\theta^*))$ denotes the leading eigenvalue of $\nabla^2 f(\theta^*)$. Moreover,  Lemma~\ref{lem:EGOP-equation} implies
    \begin{align*}
        \ip{\nu, \nabla^2 f(\theta^*) \nabla f(\theta^*)^\T \nu} &\geq \frac{1}{2c^2} \ip{\nu, G(\theta^*) \nu}=\frac{1}{2c^2}\ip{ \nu,\left(\mathbb{E}_{\theta\sim \rho}[\nabla f(\theta) \nabla f(\theta)^\T] - E_f(\theta^*)\right)\nu}.
    \end{align*}
    
    \[
        \ip{v_1, \nabla^2 f(\theta^*) \nabla f(\theta^*)^\T v_1} \geq \frac{\lambda_{\max}(\operatorname{EGOP}) - \ip{v_1, E_f(\theta^*) v_1}}{2 c^2}.
    \]
    
    Chaining the preceding inequalities yields
    \begin{equation}\label{eq:L-quadratic-LB}
        \ip{\nu, \diag(L) \nu} \geq \frac{\ip{\nu,\operatorname{EGOP}\nu} - \ip{\nu,E_f(\theta^*) \nu}}{2 c^2 \lambda_{\max}(\nabla^2 f(\theta^*))}.
    \end{equation}
    Given $\nabla^2 f(\theta^*)$ PSD, we have
    \[
        \lambda_{\max}(\nabla^2 f(\theta^*)) = \sqrt{\lambda_{\max}(\nabla^2 f(\theta^*))^2} = \sqrt{\lambda_{\max}(\nabla^2 f(\theta^*)\nabla^2 f(\theta^*)^\T)}.
    \]
    Moreover, by Lemma~\ref{lem:EGOP-equation} we have the following upper bound:
    \begin{align*}
        \lambda_{\max}(\nabla^2 f(\theta^*)\nabla^2 f(\theta^*)^\T) &\leq \frac{1}{c^2}\lambda_{\max}(G(\theta^*))\\
        &= \frac{1}{c^2} \max_{v\in \R^d} \ip{v, (\mathbb{E}_{\theta\sim \rho}[\nabla f(\theta)\nabla f(\theta)^\T] - E_f(\theta^*)) v}\\
        &\leq \frac{1}{c^2}\left(\lambda_{\max}(\operatorname{EGOP}) + \lambda_{\max}(E_f(\theta^*))\right).
    \end{align*}
    Combining this inequality with the preceding equation implies
    \[
        \lambda_{\max}(\nabla^2 f(\theta^*)) \leq \frac{1}{c}\sqrt{\lambda_{\max}(\operatorname{EGOP}) + \lambda_{\max}(E_f(\theta^*))}.
    \]
    Substituting this inequality into Eq.~\ref{eq:L-quadratic-LB} implies
    \[
        \ip{\nu, \diag(L) \nu} \geq \frac{\ip{\nu,\operatorname{EGOP}\nu} - \ip{\nu, E_f(\theta^*)\nu}}{2 c \sqrt{\lambda_{\max}(\operatorname{EGOP}) + \lambda_{\max}(E_f(\theta^*))}}.
    \]
    Recall that $\nu \in \{\pm d^{-1/2}\}^d$. Thus
    \[
        \ip{\nu, \diag(L) \nu} = \sum_{i=1}^d L_i v(i)^2 = \frac{1}{d}\sum_{i=1}^d L_i.
    \]
    Combining this with the above inequality implies
    \[
        \sum_{i=1}^d L_i \geq \frac{d\left(\ip{\nu,\operatorname{EGOP}\nu} - \ip{\nu, E_f(\theta^*)\nu}\right)}{2 c \sqrt{\lambda_{\max}(\operatorname{EGOP}) + \lambda_{\max}(E_f(\theta^*))}}.
    \]
    By Lemma~\ref{lem:EGOP-equation}, $\forall v\in \R^d$ 
    \[
        \abs{\ip{v, E_f(\theta^*) v}} \leq \gamma \norm{v}^2_2
    \]
    where
    \[
        \gamma \defeq  2H \lambda_{\max}(\nabla^2 f(\theta^*)) M_3 + H^2 M_4
    \]
    Thus
    \[
        \sum_{i=1}^d L_i \geq \frac{d}{2c}\frac{\ip{\nu,\operatorname{EGOP}\nu} - \gamma}{\sqrt{\lambda_{\max}(\operatorname{EGOP}) + \gamma}}.
    \]
    As this holds for all $\nu\in \{\pm d^{-1/2}\}^d$, we conclude 
    \[
        \sum_{i=1}^d L_i \geq \frac{d}{2c}\cdot \max_{\nu\in \{\pm d^{-1/2}\}^d}\frac{\ip{\nu,\operatorname{EGOP}\nu} - \gamma}{\sqrt{\lambda_{\max}(\operatorname{EGOP}) + \gamma}}.
    \]
    We observe that for any $v\in \R^d$,
    \[
        \ip{v, \EGOP v} = \sum_{i=1}^d \lambda_i \langle v, v_i\rangle^2 \geq \max_{k} \lambda_{k} \langle v, v_k\rangle^2
    \]
    where $(\lambda_i, v_i)$ denote the $i$\ts{th} eigenvalue and eigenvector of the EGOP indexed in decreasing order of magnitude. 
    
    In particular, for any $v_k$, consider $\nu_k\in \R^d$ defined to have entries $\nu_k(i) = \operatorname{sign}(v_k(i))d^{-1/2}$. Observe that 
    \[
        \nu_k \in \operatorname{argmax}\{\langle \nu, v_k\rangle^2 \mid  \nu\in \{\pm d^{-1/2}\}^d \}.
    \]
    For this $\nu_k$,
    \[
        \langle \nu, v_k\rangle^2 = \left(\sum_{i=1}^d \operatorname{sign}(v_k(i))d^{-1/2} v_k(i)\right)^2 = \frac{1}{d}\left(\sum_{i=1}^d \abs{v_k(i)}\right)^2 = \frac{1}{d}\norm{v_k}^2_1.
    \]
    Thus in general,
    \[
        \max_{\nu\in \{\pm d^{-1/2}\}^d} \ip{\nu, \EGOP \nu} = \max_{\nu\in \{\pm d^{-1/2}\}^d} \sum_{k=1}^d \lambda_k \langle \nu, v_k\rangle^2 \geq \max_k \frac{\lambda_k}{d}\norm{v_k}^2_1.
    \]
    This implies the lower bound
    \[
        \sum_{i=1}^d L_i \geq \frac{d}{2c}\cdot \max_{k}\frac{\lambda_k \norm{v_k}^2_1/d - \gamma}{\sqrt{\lambda_{\max}(\operatorname{EGOP}) + \gamma}}.
    \]
\end{proof}

\paragraph{Reparameterization by EGOP eigenbasis produces small $\mathbf{||\vec{\tilde{L}}||_1}$} In order to prove Lemma~\ref{claim:reparam-coor-UB}, we first establish the equivariance of the EGOP matrix. Recall that Assumption~\ref{assumption:rho-and-Theta} implies $\rho$ is isotropic about $\theta_c\in \Theta$. It is convenient to express $\rho$ as
\[
    \rho(\theta) = \rho_0(\theta-\theta_c)
\]
for $\rho_0$ mean-zero and isotropic (about the origin). For reparameterized set 
\[
    \tTheta \defeq \{V^\T \theta \mid \theta\in \Theta\}
\]
and element $V^\T \theta_c \in \tTheta$, we consider the reparameterized sampling distribution 
\begin{equation}\label{eq:def-tilde-rho}
    \tilde{\rho}(\ttheta) \defeq \rho(\ttheta-V^\T \theta_c).
\end{equation}
The distribution $\tilde{\rho}$ is then isotropic about $V^\T \theta_c$. Moreover, because $\rho_0$ is isotropic,
\[
    \tilde{\rho}(\tilde{\theta}) = \rho_0(\ttheta-V^\T \theta_c) = \rho_0(V(\ttheta-V^\T \theta_c)) = \rho_0(V\ttheta - \theta_c) =\rho(V\ttheta)
\]
and similarly $\rho(\theta) = \tilde{\rho}(V^T\theta)$.
\begin{lemma}\label{lemma:equivariant-EGOP}
    Given $f(\cdot)$ whose EGOP with respect to isometric distribution $\rho$ has eigenvalue decomposition
    \[
        \mathbb{E}_{\theta\sim \rho}[\nabla f(\theta)\nabla f(\theta)^\T] = V \Lambda V^\T, 
    \]
    the EGOP of the reparameterized function $\tf(\theta)\defeq f(V\theta)$ with respect to reparameterized sampling distribution $\tilde{\rho}$ (defined in Eq.~\ref{eq:def-tilde-rho}), we have
    \[
        \mathbb{E}_{\theta\sim \tilde{\rho}}[\nabla \tf(\theta)\nabla \tf(\theta)^\T] = \Lambda
    \]
\end{lemma}
\begin{proof}
    Because $\tilde{\rho}(\ttheta) = \rho(V\ttheta)$,
    \[
        \mathbb{E}_{\ttheta\sim\tilde{\rho}}[q(\ttheta)] = \mathbb{E}_{\theta\sim\rho}[q(V^\T\theta)]
    \]
    By the chain rule,
    \[
        \nabla \tf(\ttheta) = V^\T \nabla f(V\ttheta).
    \]
    Thus 
    \begin{align*}
        \mathbb{E}_{\ttheta\sim \tilde{\rho}}[\nabla \tf(\ttheta)\nabla \tf(\ttheta)^\T]=V^\T\mathbb{E}_{\ttheta\sim \tilde{\rho}}[\nabla f(V\ttheta)f(V\ttheta)^\T]V=V^\T \mathbb{E}_{\theta\sim \rho}[\nabla f(\theta)\nabla f(\theta)^\T]V = \Lambda.
    \end{align*}
\end{proof}

Using Lemma~\ref{lemma:equivariant-EGOP}, we now prove Lemma~\ref{claim:reparam-coor-UB}.

\begin{proof}[Proof of Lemma~\ref{claim:reparam-coor-UB}]
    Consider reparameterized function $\tf(\theta)\defeq f(V\theta)$ and corresponding $\tTheta$ and $\tilde{\rho}$ as defined in Eq.~\ref{eq:def-tilde-rho}. Observe that Assumption~\ref{assumption:H-Lipschitz} implies $\nabla^2 f(\cdot)$ is also $H$-Lipschitz within $\tTheta$ and that $V^T\theta^*\in \tTheta$ is a local minimum. Applying Lemma~\ref{lem:EGOP-equation} to $\tilde{f}(\cdot)$ implies
    \[
        \mathbb{E}_{\tilde{\rho}}[\nabla \tf(\theta)\nabla \tf(\theta)^T] = G(V^T\theta^*) + E_{\tf}[V^T\theta^*]
    \]
    where
    \[
        c^2 \nabla^2 \tf(V^T\theta^*)\nabla^2 \tf(V^T\theta^*)^T \preceq G(V^T \theta^*) \preceq 2c^2 \nabla^2 \tf(V^T\theta^*)\nabla^2 \tf(V^T\theta^*)^T
    \]
    and
    \[
        \abs{\langle v, E_{\tf}[V^T\theta^*] v\rangle} \leq \gamma \norm{v}^2_2
    \]
    for
    \[
        \gamma \defeq  2H \lambda_{\max}(\nabla^2 f(\theta^*)) M_3 + H^2 M_4.
    \]
    We've used the fact that 
    \[
        M_p \defeq \mathbb{E}_{\theta\sim \rho}[\norm{\theta-\theta^*}^p_2] = \mathbb{E}_{\theta\sim \rho}[\norm{V^\T(\theta-\theta^*)}^p_2] = \mathbb{E}_{\ttheta\sim \tilde{\rho}}[\shortnorm{\ttheta-V^\T\theta^*}^p_2]
    \]
    and
    \[
        \lambda_{\max}(\nabla^2 \tf(V^\T \theta^*))
    \]
    to bound $\abs{\langle v, E_{\tf}[V^T\theta^*] v\rangle}$ by the same value $\gamma$.
    
    Combining the above matrix inequalities with the implication $E_{\tf}[V^T\theta^*]\preceq \gamma\mathbb{I}$ implies
    \[
        c^2\nabla^2 \tf(V^\T\theta^*)\nabla^2\tf (V^\T\theta^*)^\T  \preceq \mathbb{E}_{\theta\sim\rho}[\nabla\tf(\theta)\nabla\tf(\theta)^\T] +\gamma\mathbb{I}
    \]
    Recall that the PSD ordering $X\preceq Y$ implies $X^{1/2} \preceq Y^{1/2}$ (for completeness, see statement and proof in Lemma~\ref{lemma:order-preserving-square-root}). Thus the preceding matrix inequality implies
    \[
        \nabla^2 \tf(V^\T\theta^*) \preceq \frac{1}{c}\left(\mathbb{E}_{\theta\sim\rho}[\nabla\tf(\theta)\nabla\tf(\theta)^\T] +\gamma\mathbb{I}\right)^{1/2}
    \]
    
    Lemma~\ref{lemma:equivariant-EGOP} implies $\mathbb{E}_{\tilde{\rho}}[\nabla \tf(\theta)\nabla \tf(\theta)^\T]$ is diagonal, and its entries are given by the eigenvalues of $\operatorname{EGOP}(f)\defeq \mathbb{E}_{\theta\sim \rho}[\nabla f(\theta)\nabla f(\theta)^\T]$. As the EGOP is an expectation of PSD matrices, these eigenvalues are all non-negative, and thus
    \[
        \mathbb{E}_{\theta\sim \tilde{\rho}}[\nabla \tf(\theta)\nabla \tf(\theta)^\T]^{1/2}_{i,i} = \sqrt{\lambda_i(\operatorname{EGOP}(f))}.
    \]
    
    Since $\mathbb{E}_{\theta\sim\rho}[\nabla\tf(\theta)\nabla\tf(\theta)^\T]$ is diagonal and $\gamma\mathbb{I}$ is diagonal, 
    \[
        \left(\mathbb{E}_{\theta\sim\rho}[\nabla\tf(\theta)\nabla\tf(\theta)^\T] +\gamma\mathbb{I}\right)^{1/2}_{i,i} = \sqrt{\mathbb{E}_{\theta\sim\rho}[\nabla\tf(\theta)\nabla\tf(\theta)^\T]_{i,i} + \gamma}.
    \]
    Thus for all $v\in \R^d$, we have
    \begin{equation}\label{eq:UB-intermediate-quadratic}
        \ip{v, \nabla^2 \tf(V^\T \theta^*) v} \leq \frac{1}{c}\sum_{i=1}^d v(i)^2 \sqrt{\mathbb{E}_{\theta\sim\rho}[\nabla\tf(\theta)\nabla\tf(\theta)^\T]_{i,i} + \gamma}.
    \end{equation}
    Employing reverse triangle inequality and the fact that $V^T\theta^*$ is a local minimum implies $\nabla^2 \tf(V^T\theta^*)$ is PSD, we can show that Assumption~\ref{assumption:H-Lipschitz} implies
    \[
        \abs{\langle v, \nabla^2 \tf(\theta)v\rangle } \leq \langle v, \nabla^2 \tf(V^T\theta^*)v\rangle + H\norm{\theta-V^T \theta^*}_2\norm{v}^2_2
    \]
    that $\forall \theta\in \tTheta, \forall v\in \R^d$. Observe that $\norm{\theta-V^T \theta^*}_2 \leq \tilde{D}_2 = D_2$ for $D_2$ the $\ell_2$-diameter of $\Theta$. Thus combining the above with Eq.~\ref{eq:UB-intermediate-quadratic} implies $\forall \theta\in \tTheta, \forall v\in \R^d$,
    \[
        \abs{\ip{v, \nabla^2 \tf(\theta)v}} \leq \sum_{i=1}^d \sqrt{\lambda_i(\EGOP(f) + \gamma}/c + HD_2)v(i)^2.
    \]
    Thus by Lemma~\ref{lem:Hessian-coor-wise-smooth}, this implies that $\tf(\cdot)$ satisfies coordinate-wise smoothness with respect to constants
    \[
        L_i \leq \frac{1}{c}\sqrt{\lambda_i(\operatorname{EGOP}(f))+\gamma} + HD_2.
    \]
    This implies the sum of the coordinate-wise smoothness constants is bounded above as
    \begin{align*}
        L_{\tf} &\leq  \sum_{i=1}^d \left(\frac{1}{c}\sqrt{\lambda_i(\operatorname{EGOP}(f))+\gamma}  + HD_2\right)\\
        &\leq d\left(\frac{\sqrt{\gamma}}{c} + HD_2\right)+ \frac{1}{c}\sum_{i=1}^d \sqrt{\lambda_i(\operatorname{EGOP}(f))}.
    \end{align*}
\end{proof}

\paragraph{Proof of Theorem~\ref{thm:smoothness-constants-ratio}} Lemmas~\ref{claim:OG-coor-LB} and \ref{claim:reparam-coor-UB} imply Theorem~\ref{thm:smoothness-constants-ratio}:
\begin{proof}[Proof of Theorem~\ref{thm:smoothness-constants-ratio}]    
    By Lemma~\ref{claim:OG-coor-LB},
    \[
        \norm{\vec{L}}_1 \geq \max_{k}\frac{d}{2c}\frac{\beta_k\lambda_k(\EGOP(f)) - \gamma}{\sqrt{\lambda_{1}(\EGOP(f)) + \gamma}} \geq \frac{d}{2c}\frac{\beta_1\lambda_1(\EGOP(f)) - \gamma}{\sqrt{\lambda_{1}(\EGOP(f)) + \gamma}}
    \]
    for
    \[
        \gamma \defeq  2H \lambda_{\max}(\nabla^2 f(\theta^*)) M_3 + H^2 M_4.
    \]
    Recall that
    \[
        M_p \defeq \mathbb{E}_{\theta\sim \rho}[\norm{\theta-\theta^*}^p_2]
    \]
    and thus, because we assume $\operatorname{supp}(\rho)\subseteq \Theta$, $M_p \leq D^p_2$ where $D_2$ is the $\ell_2$ diameter of $\Theta$. 

    Additionally, Lemma~\ref{claim:reparam-coor-UB} states that
    \[
        \norm{\vec{\tilde{L}}}_1 \leq d\left(\frac{\sqrt{\gamma}}{c} + HD_2\right)+ \frac{1}{c}\sum_{i=1}^d \sqrt{\lambda_i(\EGOP(f))}
    \]
    where $\gamma$ is the same as defined above.

    Combining these two results implies
    \begin{align*}
        \frac{\norm{\vec{\tilde{L}}}_1}{\norm{\vec{L}}_1} \leq \left(d\left(\frac{\sqrt{\gamma}}{c} + HD_2\right)+ \frac{1}{c}\sum_{i=1}^d \sqrt{\lambda_i}\right)\left(\frac{d}{2c}\frac{\beta_1\lambda_1 - \gamma}{\sqrt{\lambda_{1} + \gamma}}\right)^{-1}
    \end{align*}
    where, for conciseness, we let $\lambda_i \defeq \lambda_i(\EGOP(f))$ the $i$th eigenvalue of $\EGOP(f)$, indexed in decreasing order. Expanding the above yields 
    \begin{align}
        \frac{\norm{\vec{\tilde{L}}}_1}{\norm{\vec{L}}_1} &\leq d\left(\frac{\sqrt{\gamma}}{c} + HD_2\right)\cdot \frac{2c}{d}\frac{\sqrt{\lambda_{1} + \gamma}}{\beta_1\lambda_1 - \gamma}+ \frac{1}{c}\sum_{i=1}^d \sqrt{\lambda_i} \cdot \frac{2c}{d}\frac{\sqrt{\lambda_{1} + \gamma}}{\beta_1\lambda_1 - \gamma}\\
        &= 2\left(\sqrt{\gamma} + HcD_2\right)\cdot \frac{\sqrt{\lambda_{1} + \gamma}}{\beta_1\lambda_1 - \gamma}+ 2\sum_{i=1}^d \sqrt{\lambda_i} \cdot \frac{1}{d}\frac{\sqrt{\lambda_{1} + \gamma}}{\beta_1\lambda_1 - \gamma}.\label{eq:pf-messy-expression}
    \end{align}
    Because $M_p \leq D_2^p$, by definition $\gamma$ satisfies 
    \[
        \gamma \leq  2H \lambda_{\max}(\nabla^2 f(\theta^*)) D^3_2 + H^2 D_2^4.
    \]
    Thus 
    \[
        H \leq \frac{\sqrt{\delta \lambda_1(\EGOP) + D_2^2 \lambda^2_1(\nabla^2 f(\theta^*))} - D_2 \lambda_1(\nabla^2 f(\theta^*))}{D^2_2},
    \]
    implies $\gamma\leq \delta \lambda_1$ for $\lambda_1 \defeq \lambda_1(\EGOP(f))$, the largest EGOP eigenvalue. Additionally, $H$ admits a simplified upper bound:
    \[
        H \leq \frac{\sqrt{\delta \lambda_1(\EGOP) + D_2^2 \lambda^2_1(\nabla^2 f(\theta^*))} - D_2 \lambda_1(\nabla^2 f(\theta^*))}{D^2_2} \leq \frac{\sqrt{\delta\lambda_1}}{D_2^2}.
    \]

    Using these bounds on $H$ and $\gamma$, the first term in Eq.~\ref{eq:pf-messy-expression} can be bounded as
    \begin{align*}
        2\left(\sqrt{\gamma} + HcD_2\right)\cdot \frac{\sqrt{\lambda_{1} + \gamma}}{\beta_1\lambda_1 - \gamma} & \leq 2\left(\sqrt{\delta \lambda_1} + \frac{\sqrt{\delta\lambda_1}}{D_2^2}\cdot cD_2\right)\cdot \frac{\sqrt{1+\delta}}{\beta_1-\delta}\cdot \frac{1}{\sqrt{\lambda_1}}\\
        &=2\left(1 + \frac{cD_2}{D_2^2} \right)\frac{\sqrt{\delta(1+\delta)}}{\beta_1-\delta}\\
        &\leq 6\frac{\sqrt{\delta(1+\delta)}}{\beta_1-\delta}
    \end{align*}
    where the last inequality follows from the fact that because $\operatorname{supp}(\rho)\subseteq \Theta$, $c\leq D_2$.

    The second term in Eq.~\ref{eq:pf-messy-expression} can be bounded as
    \begin{align*}
        2\sum_{i=1}^d \sqrt{\lambda_i} \cdot \frac{1}{d}\frac{\sqrt{\lambda_{1} + \gamma}}{\beta_1\lambda_1 - \gamma} &\leq 2\sum_{i=1}^d \sqrt{\lambda_i} \cdot \frac{1}{d}\frac{\sqrt{1+\delta}}{(\beta_1-\delta)\sqrt{\lambda_1}} = \frac{2\sr(f)}{d(\beta_1-\delta)}\sqrt{1+\delta}
    \end{align*}
    where the equality follows from the definition of the stable rank. Combining these bounds yields the result.
\end{proof}

\subsubsection{Convergence guarantees in the convex setting}\label{ssec:cvx-cvgnce}

Here we prove \cref{lemma:convex-cvgnce-ball-version} by proving a more general statement in terms of any convex constraint set $\Theta$.

\begin{lemma}[Convergence for convex objectives with noise-free gradients]\label{lemma:convex-cvgnce}
    Consider $f(\cdot)$ a convex objective, constraint set $\Theta$, and sampling distribution $\rho(\cdot)$ satisfying Assumptions \ref{assumption:rho-and-Theta} and \ref{assumption:H-Lipschitz}. Let $\tf(\cdot)$ and $\tTheta$ denote the reparameterized objective and constraint set respectively corresponding to the eigenbasis of $\EGOP(f)$.  Then running Algorithm~\ref{alg:Adagrad} with constrained updates and with inputs $(\tf(\cdot), \nabla \tf(\cdot),\ttheta_0,T, \epsilon, \tTheta)$ for any initial condition $\ttheta_0\in \tTheta$ produces iterates $\{\ttheta_t\}_{t=1}^T$ satisfying
    \begin{equation}\label{eq:cvx-reparam-bounds}
        \frac{1}{T}\sum_{t=0}^{T-1}(\tf(\ttheta_t) - \tf(\ttheta^*)) =O\left( \left(\eta + \frac{\tilde{D}^2_{\infty}}{\eta}\right)^2 \frac{\shortnorm{\vec{L}}_1}{T}\left(\frac{\sr(f)\sqrt{1+\delta}}{d(\beta_1-\delta)} + \frac{\sqrt{\delta(1+\delta)}}{\beta_1-\delta} \right) + \frac{\epsilon D^2_2}{\eta T} \right).
    \end{equation}
    Setting $\eta = \tilde{D}_\infty$ to minimize this upper bound yields
    \[
        \frac{1}{T}\sum_{t=0}^{T-1}(\tf(\ttheta_t) - \tf(\ttheta^*)) = O\left( \tilde{D}^2_{\infty} \frac{\shortnorm{\vec{L}}_1}{T}\left(\frac{\sr(f)\sqrt{1+\delta}}{d(\beta_1-\delta)} + \frac{\sqrt{\delta(1+\delta)}}{\beta_1-\delta} \right) + \frac{\epsilon D^2_2}{\eta T} \right).
    \]
\end{lemma}

\begin{proof}[Proof of Lemma~\ref{lemma:convex-cvgnce}]
    An intermediate result in the proof of Thm. 4.1 in \citet{liu2024adagrad} implies that in the case of noise-free gradients, i.e. oracle gradient $g(\cdot) = \nabla \tf(\cdot)$, Adagrad with constraint set $\tTheta$ produces iterates such that
    \[
        \frac{1}{T}\sum_{t=0}^{T-1}(\tf(\ttheta_t) - \tf(\theta^*)) \leq \left(\eta + \frac{\tilde{D}^2_{\infty}}{\eta}\right)^2\cdot \frac{\norm{\vec{\tilde{L}}}_1}{T} + \frac{\epsilon \tilde{D}^2_2}{\eta T}.
    \]
    We note $D^2_2 = \tilde{D}_2^2$ because the Euclidean norm is invariant under orthonormal transformation. Because the constraint update implies $\{\ttheta_t\}\subseteq \tTheta$, applying Thm.~\ref{thm:smoothness-constants-ratio} implies
    \[
        \norm{\vec{\tilde{L}}}_1 \leq \norm{\vec{L}}_1\cdot \sqrt{1+\delta}\left(\frac{2\sr(f)}{d(\beta_1-\delta)} + 2\frac{\sqrt{\delta}}{\beta_1-\delta}\left(1 + \frac{2}{\sqrt{d}} \right) \right)
    \]
    which immediately implies the result.
\end{proof}

We note that the bound on reparameterized Adagrad is a function of $\tilde{D}_\infty$ while the bound on Adagrad in original coordinates is a function of $D_{\infty}$. While $\tilde{D}^2_\infty$ can be smaller or larger than $D^2_\infty$ by up to a factor of $d$, in many typical settings $\tilde{D}_{\infty} \leq D_{\infty}$: if the constraint set $\Theta$ is chosen to be a ball, $D_\infty = \tilde{D}_\infty$. Optimizing convex objectives with L2 regularization--sometimes referred to as with optimization with weight decay in machine learning literature--corresponds to optimizing with $\Theta$ some ball centered at the origin. In \cref{ssec:L2-reg}, we present experiments using L2 regularization with adaptive algorithms, and confirm that the empirical benefit from reparameterization persists in this setting.

\subsubsection{Convergence guarantees for non-convex objectives}\label{ssec:nonconvex-guarantees}

In the non-convex setting we can establish related local improved convergence guarantee for unconstrained Adagrad. For a given initial point $\theta_0$, let $\Delta_f(\theta_0)\defeq f(\theta_0)-\min_{\theta}f(\theta)$. Observe that $\Delta_{\tf}(\ttheta_0) =\Delta_f(\theta_0)$ for $\ttheta_0 \defeq V^T \theta_0$. For unconstrained optimization, we add the following additional assumption:
\begin{assumption}\label{assumption:f-bounded-below}
    The objective $f(\cdot)$ is bounded below: $\inf_{\theta} f(\theta) > -\infty$.
\end{assumption}

\begin{theorem}[Convergence for non-convex, unconstrained optimization]\label{lemma:nonconvex-cvgnce}
    Consider $f(\cdot)$, $\Theta$, and $\rho(\cdot)$ satisfying Assumptions \ref{assumption:rho-and-Theta}, \ref{assumption:H-Lipschitz}, and \ref{assumption:f-bounded-below}. Let $\tf(\cdot)$ and $\tTheta$ denote the reparameterization of $f(\cdot)$ and $\Theta$ by $V$ the eigenbasis of $\EGOP(f)$. Consider any initialization $\theta_0\in \Theta$, and let $\ttheta_0 \defeq V^T \theta_0$. Assume $\exists \delta \in [0,\beta^2_1)$ such that $\nabla^2 f(\cdot)$ has Lipschitz constant $H$ satisfying the bound in Eq.~\ref{eq:H-bound}. Let $\{\ttheta\}_{t=1}^T$ denote the iterates produced by running Algorithm~\ref{alg:Adagrad} using \textit{unconstrained} updates and with inputs $(\tf(\cdot), \nabla \tf(\cdot),\ttheta_0,T, \epsilon)$ for $\epsilon < 1/d$. Then for all $T$ such that $\{\ttheta_t\}_{t=1}^T \subseteq \tTheta$, we have
    \[
        \frac{1}{T} \sum_{t=1}^T \norm{\nabla f(\theta_t)}_1  = O\left(\frac{\Delta_{f}(\theta_0)}{\eta \sqrt{T}} + \frac{\eta \norm{\vec{L}}_1}{\sqrt{T}}\left(\frac{\sr(f)\sqrt{1+\delta}}{d(\beta_1-\delta)} + \frac{\sqrt{\delta(1+\delta)}}{\beta_1-\delta} \right) \log(p)\right)
    \]
    where $p(T, \shortnorm{\vec{\tilde{L}}}_1, \nabla \tf(\theta_0)))$ is polynomial in $T$, $\shortnorm{\vec{\tilde{L}}}_1$, and $\shortnorm{\nabla \tf(\ttheta_0)}_\infty$.
\end{theorem}
We compare the result in \cref{lemma:nonconvex-cvgnce} to the convergence guarantee for Adagrad in original coordinates initialized at $\theta_0$, as established in \citet{jiang2024convergence}:
\[
        \frac{1}{T} \sum_{t=1}^T \norm{\nabla f(\theta_t)}_1  = O\left(\frac{\Delta_{f}(\theta_0)}{\eta \sqrt{T}} + \frac{\eta \shortnorm{\vec{L}}_1}{\sqrt{T}} \log(p(T, \shortnorm{\vec{L}}_1, \nabla f(\theta_0)))\right).
\]
As with Lemma~\ref{lemma:convex-cvgnce}, this result implies that for $f(\cdot)$ with strong EGOP spectral decay and dense leading EGOP eigenvectors, the local convergence bounds for reparameterized can be smaller with by a factor of $1/d$.

We emphasize that the above is a \textit{local} guarantee, and only holds for time horizons $T$ such that $\{\ttheta_t\}_{t=1}^T \subseteq \tTheta$. 
If iterates move arbitrarily far away from the region in which EGOP samples were concentrated, then without stronger assumptions about the landscape (i.e. vanishingly small $H$) one cannot hope that the EGOP reparameterization will be informative for new local geometry. For quadratic functions, $H$ = 0 and thus the result holds for any $T$; we note that related works have restricted their analysis to quadratic functions for analyzing adaptive algorithms \cite{zhang2024transformers}. 

\begin{proof}[Proof of Thm.~\ref{lemma:nonconvex-cvgnce}]
    The result follows immediately from combining the convergence result in Thm. 3.1 of \citet{jiang2024convergence} with the bound on $\shortnorm{\vec{\tilde{L}}}_1/\shortnorm{\vec{L}}_1$ from Thm.~\ref{thm:smoothness-constants-ratio}.
\end{proof}

\subsubsection{Lipschitz Constants of Hessians in Machine Learning}\label{ssec:ML-lipschitz-Hessians}

 In this section, we note one family of naturally-motivated non-convex objectives satisfying Assumption~\ref{assumption:H-Lipschitz} that arise in machine learning problems. We consider the over-parameterized matrix factorization problem: let $\theta \in \R^{(d_1+d_2)k}$ be parameters, whose entries can be grouped into two matrices as $\theta = (L,R)$ for $L\in \R^{d_1\times k}, R\in \R^{d_2\times k}$. We define the objective 
\begin{equation}\label{eq:linear-measurement-overparam-objective}
    f(\theta) = \norm{\mathcal{A}(LR^\T) - b}^2_2
\end{equation}
where $b \in \R^m$ correspond to measurements of some matrix under map $\mathcal{A}(\cdot)$, and where  $\mathcal{A}:\R^{d_1\times d_2}\rightarrow \R^m$ denotes a map 
\[
    \mathcal{A}(X) \defeq (\langle A_1, X\rangle, \dots, \langle A_m, X\rangle) \in \R^m.
\]
We note that one special case of Eq.~\ref{eq:linear-measurement-overparam-objective} is the set of objectives that arise when training a two-layer linear feed-forward network using mean-squared-error loss.

We can bound the Lipschitz constant of the Hessian of this objectives in this family:
\begin{lemma}\label{lemma:overparam-linear-measurement-Lipschitz-Hessian}
    Consider the overparameterized matrix factorization objective with a linear measurement map, as defined in Eq.~\ref{eq:linear-measurement-overparam-objective}. Let $\theta \in \R^{(d_1+d_2)k}$ be parameters, whose entries can be grouped into two matrices: $\theta = (L,R)$ for $L\in \R^{d_1\times k}, R\in \R^{d_2\times k}$. Consider any measurement vector $b \in \R^{m}$. Then in the ball $\norm{\theta}_2 \leq B$, the objective in Eq.~\ref{eq:linear-measurement-overparam-objective} satisfies
     \[
        \opnorm{\nabla^2 f(\theta_1)-\nabla^2 f(\theta_2)} \leq 12 B \left(\sum_{i=1}^m \frobnorm{A_i}^2 \right)\norm{\theta_1 - \theta_2}_2.
     \]
\end{lemma}

 We first characterize the quadratic form of the Hessian of Eq.~\ref{eq:linear-measurement-overparam-objective}.
\begin{lemma}\label{lemma:linear-measurement-overparam-Hessian-bilinear}
    For the objective $f(\cdot)$ as defined in Eq.~\ref{eq:linear-measurement-overparam-objective}, for any $\theta, v\in \R^{(d_1\times d_2)k}$ with entries denoted $\theta = (L,R)$ and $v=(U,V)$, the Hessian quadratic form can be expressed as
    \[
        D^2 f(L, R)[U,V] = 2\norm{\mathcal{A}(UR^\T + LV^\T)}^2_2 + 4 \langle \mathcal{A}(LR^\T)-b, \mathcal{A}(UV^\T)\rangle.
    \]
\end{lemma}
\begin{proof}
    We begin by deriving the gradient form using the limit definition:
    \begin{align*}
        \nabla f(L,R)[U,V] &\defeq \lim_{t\rightarrow 0}\frac{1}{t}\left(f(L+tU, R+tV)-f(L,R)\right)\\
        &=\lim_{t\rightarrow 0}\frac{1}{t}\left(\norm{\mathcal{A}((L+tU)(R+tV)^\T) - b}^2_2 - \norm{\mathcal{A}(LR^\T) - b}^2_2\right).
    \end{align*}
    By linearity of the map $\mathcal{A}(\cdot)$,
    \begin{align*}
        \mathcal{A}((L+tU)(R+tV)^\T) &= \mathcal{A}(LR^\T + tUR^\T + tLV^\T + t^2 UV^\T)\\
        &= \mathcal{A}(LR^\T) + t \mathcal{A}(UR^\T + LV^\T) + t^2 \mathcal{A}(UV^\T).
    \end{align*}
    Substituting this into the above limit yields
    \begin{align*}
        \nabla f(L,R)[U,V] &=\lim_{t\rightarrow 0}\frac{1}{t}\left(\norm{\Delta + t\mathcal{A}(UR^\T + LV^\T) + t^2 \mathcal{A}(UV^\T)}^2_2 -\norm{\Delta}^2_2\right)
    \end{align*}
    where $\Delta \defeq \mathcal{A}(LR^\T)-b$. Expanding the squared norm, this yields
    \begin{align*}
        \nabla f(L,R)[U,V] &=\lim_{t\rightarrow 0}\frac{1}{t}\Big(\norm{\Delta}^2_2 + \norm{t\mathcal{A}(UR^\T + LV^\T) + t^2 \mathcal{A}(UV^\T)}^2_2\\
        &\qquad\qquad\quad + 2 \langle \Delta,t\mathcal{A}(UR^\T + LV^\T) + t^2 \mathcal{A}(UV^\T)\rangle -\norm{\Delta}^2_2\Big)\\
        &=\lim_{t\rightarrow 0}\Big(t \norm{\mathcal{A}(UR^\T + LV^\T) + t \mathcal{A}(UV^\T)}^2_2+2 \langle \Delta,\mathcal{A}(UR^\T + LV^\T) + t \mathcal{A}(UV^\T)\rangle\Big)\\
        &=0+2 \langle \Delta,\mathcal{A}(UR^\T + LV^\T) + 0\rangle\\
        &=2\langle \mathcal{A}(LR^\T)-b), \mathcal{A}(UR^\T + LV^\T\rangle
    \end{align*}
    where the last line follows from $\Delta \defeq\mathcal{A}(LR^\T)-b$. Given this expression for the gradient, we can then define the Hessian quadratic form on input $v=(U,V)$ as
    \[
        \nabla^2 f(L,R)[U,V]\defeq \lim_{t\rightarrow 0}\frac{1}{t} \left(\nabla f(L+tU, R+tV)[U,V]-\nabla f(L, R)[U,V]\right).
    \]
    Given the above expression for the gradient, this yields
    \begin{align*}
        \nabla^2 f(L,R)[U,V] &= \lim_{t\rightarrow 0}\frac{1}{t}\Big(2\langle \mathcal{A}((L+tU)(R+tV)^\T) - b, \mathcal{A}(U (R+tV)^\T + (L+tU)V^\T) \rangle\\
        &\qquad \qquad \quad - 2\langle \Delta, \mathcal{A}(UR^\T + LV^\T\Big).
    \end{align*}
    As noted above, 
    \begin{align*}
        \mathcal{A}((L+tU)(R+tV)^\T) = \mathcal{A}(LR^\T) + t \mathcal{A}(UR^\T + LV^\T) + t^2 \mathcal{A}(UV^\T).
    \end{align*}
    Substituting this in to the first term in the expression for the Hessian, and using linearity of $\mathcal{A}(\cdot)$ to simplify terms, yields
    \begin{align*}
        \nabla^2 f(L,R)&[U,V] = \lim_{t\rightarrow 0}\frac{2}{t}\Big(\langle \Delta  +t \mathcal{A}(UR^\T + LV^\T) + t^2 \mathcal{A}(UV^\T), \mathcal{A}(UR^\T + LV^\T) + 2t\mathcal{A}(UV^\T) \rangle\\
        &\qquad \qquad \quad - \langle \Delta, \mathcal{A}(UR^\T + LV^\T\Big)\\
        &= \lim_{t\rightarrow 0}\frac{2}{t}\Big(t\langle  \mathcal{A}(UR^\T + LV^\T) + t \mathcal{A}(UV^\T), \mathcal{A}(UR^\T + LV^\T)\rangle \\
        &\qquad \qquad \quad +2t \langle \Delta, \mathcal{A}(UV^\T)\rangle + 2t^2\langle \mathcal{A}(UR^\T + LV^\T) + t^2 \mathcal{A}(UV^\T),\mathcal{A}(UV^\T)\rangle\Big).
    \end{align*}
    Taking the limit as $t\rightarrow 0$,
    \begin{align*}
        \nabla^2 f(L,R)[U,V] &= 2\Big(\langle \mathcal{A}(UR^\T + LV^\T), \mathcal{A}(UR^\T + L^\T)\rangle + 2\langle \Delta, \mathcal{A}(UV^\T)\rangle\Big)\\
        &= 2\norm{\mathcal{A}(UR^\T + LV^\T)}^2_2 + 4\langle \mathcal{A}(LR^\T)-b, \mathcal{A}(UV^\T)\rangle
    \end{align*}
    which yields the result.
\end{proof}

We can prove Lemma~\ref{lemma:overparam-linear-measurement-Lipschitz-Hessian}, which bounds the Lipschitz constant of the Hessian of $f(\cdot)$ as defined in Eq.~\ref{eq:linear-measurement-overparam-objective}.

\begin{proof}[Proof of Lemma~\ref{lemma:overparam-linear-measurement-Lipschitz-Hessian}]
    Given some vector $v\in \R^{(d_1 + d_2)k}$, group the entries of the vector $v$ into two matrices: let $v = (U, V)$ for matrices $U\in \R^{d_1\times k}, V\in \R^{d_2\times k}$. Then the quadratic form of the Hessian is
     \[
        \ip{v, \nabla^2 f(\theta) v} = D^2 f(L,R)[U,V].
     \]
     The objective $f(\cdot)$ satisfies Assumption~\ref{assumption:H-Lipschitz} with respect to Lipschitz constant $H$ if  $\forall \theta_1, \theta_2, v\in \R^{(d_1 + d_2)k}$,
     \[
        \abs{\ip{v, (\nabla^2 f(\theta_1)-\nabla^2 f(\theta_2)) v}} \leq H \norm{v}^2_2 \norm{\theta_1-\theta_2}_2.
     \]
     For any pair $\theta_1, \theta_2\in \R^{(d_1 + d_2)k}$, denote the entries by $\theta_1 = (L_1, R_1)$ and $\theta_2 = (L_2, R_2)$, and for any $v\in \R^{(d_1 + d_2)k}$, denote the entries by $v = (U, V)$ as above. Then the above inequality is equivalent to
     \[
        \abs{D^2f(L_1, R_1)[U,V] - D^2f(L_2, R_2)[U,V]} \leq H \left(\frobnorm{U}^2 + \frobnorm{V}^2 \right)\norm{(L_1, R_1)-(L_2, R_2)}_2.
     \]
     where
     \[
        \norm{(L,R)}_2 \defeq \sqrt{\frobnorm{L}^2 + \frobnorm{R}^2} = \norm{\theta}_2.
     \]
    By Lemma~\ref{lemma:linear-measurement-overparam-Hessian-bilinear}, for any $\theta = (L,R)$ and any $v=(U,V)$, the Hessian satisfies
    \begin{align*}
        D^2 f(L, R)[U,V] &= 2\norm{\mathcal{A}(UR^\T + LV^\T)}^2_2 + 4 \langle \mathcal{A}(LR^\T)-b, \mathcal{A}(UV^\T)\rangle.
    \end{align*}
    By linearity of $\mathcal{A}(\cdot)$, we can expand the squared norm:
    \begin{align*}
        D^2 f(L, R)&[U,V] = 2\norm{\mathcal{A}(UR^\T) + \mathcal{A}(LV^\T)}^2_2 + 4 \langle \mathcal{A}(LR^\T)-b, \mathcal{A}(UV^\T)\rangle\\
        &= 2\left(\norm{\mathcal{A}(LV^\T)}^2_2+\norm{\mathcal{A}(UR^\T)}^2_2\right) +4\langle \mathcal{A}(UR^\T), \mathcal{A}(LV^\T)\rangle+ 4 \langle \mathcal{A}(LR^\T)-b, \mathcal{A}(UV^\T)\rangle.
    \end{align*}
    Thus
    \begin{align}
        D^2f(L_1, R_1)[U,V] - &D^2f(L_2, R_2)[U,V] = 2\left(\norm{\mathcal{A}(L_1V^\T)}^2_2-\norm{\mathcal{A}(L_2V^\T)}^2_2\right)\label{eq:quadratic-diff-line-1}\\ 
        &\quad + 2\left(\norm{\mathcal{A}(UR_1^\T)}^2_2-\norm{\mathcal{A}(UR_2^\T)}^2_2\right)\label{eq:quadratic-diff-line-2}\\ 
        &\quad + 4\left(\langle \mathcal{A}(UR_1^\T), \mathcal{A}(L_1V^\T)\rangle-\langle \mathcal{A}(UR_2^\T), \mathcal{A}(L_2V^\T)\rangle\right)\label{eq:quadratic-diff-line-3}\\ 
        & \quad + 4\langle \mathcal{A}(L_1 R_1^\T)-\mathcal{A}(L_2 R_2^\T), \mathcal{A}(UV^\T)\rangle. \label{eq:quadratic-diff-line-4}
    \end{align}
    We bound the magnitude of each term in sequence. For the first term in Line~\ref{eq:quadratic-diff-line-1}, we observe
    \begin{align*}
        \left|\norm{\mathcal{A}(L_1V^\T)}^2_2-\norm{\mathcal{A}(L_2V^\T)}^2_2\right| &= \abs{\langle \mathcal{A}(L_1V^\T)+\mathcal{A}(L_2V^\T), \mathcal{A}(L_1V^\T)-\mathcal{A}(L_2V^\T) \rangle}\\
        &= \abs{\langle \mathcal{A}\left((L_1+L_2)V^\T\right), \mathcal{A}\left((L_1-L_2)V^\T\right) \rangle}\\
        &\leq \norm{\mathcal{A}\left((L_1+L_2)V^\T\right)}_2 \norm{\mathcal{A}\left((L_1-L_2)V^\T\right)}_2.
    \end{align*}
    The operator norm of $\mathcal{A}(\cdot)$ can be bounded as
    \[
        \norm{\mathcal{A}(X)}^2_2 = \sum_{i=1}^m \langle A_i, X\rangle^2 \leq \frobnorm{X}^2 \sum_{i=1}^m \frobnorm{A_i}^2.
    \]
    Thus
    \begin{equation}\label{eq:linear-measurement-map-2-norm}
        \norm{\mathcal{A}}_2 \leq \left(\sum_{i=1}^m \frobnorm{A_i}^2\right)^{1/2}.
    \end{equation}
    We can thus bound the term in Line~\ref{eq:quadratic-diff-line-1} as
    \begin{align*}
        \left|\norm{\mathcal{A}(L_1V^\T)}^2_2-\norm{\mathcal{A}(L_2V^\T)}^2_2\right| &\leq \norm{\mathcal{A}\left((L_1+L_2)V^\T\right)}_2 \norm{\mathcal{A}\left((L_1-L_2)V^\T\right)}_2\\
        &\leq \norm{\mathcal{A}}^2_2 \frobnorm{(L_1+L_2)V^\T} \frobnorm{(L_1-L_2)V^\T} \\
        &\leq \norm{\mathcal{A}}^2_2 \frobnorm{L_1+L_2} \frobnorm{V} \frobnorm{L_1-L_2} \frobnorm{V}.
    \end{align*}
    In the ball $\norm{\theta}_2 \leq B$, we have that for $\theta = (L, R)$
     \[
        \frobnorm{L} \leq \norm{\theta}_2 \leq B.
     \]
    Thus $\frobnorm{L_1+L_2} \leq 2B$, so we can bound
    \begin{align*}
        \left|\norm{\mathcal{A}(L_1V^\T)}^2_2-\norm{\mathcal{A}(L_2V^\T)}^2_2\right| &\leq 2B \norm{\mathcal{A}}^2_2 \norm{V}^2_F \frobnorm{L_1-L_2} \\
        &\leq 2B \norm{\mathcal{A}}^2_2 \frobnorm{V}^2 \norm{(L_1, R_1)-(L_2, R_2)}_2
    \end{align*}
    where the last line follows by
     \[
        \norm{(L_1, R_1)-(L_2, R_2)}_2 \defeq \sqrt{\frobnorm{L_1-L_2}^2 + \frobnorm{R_1-R_2}^2} \geq \frobnorm{L_1 - L_2}.
     \]
     An analogous argument bounds the term in Line~\ref{eq:quadratic-diff-line-2} as
     \[
        \left|\norm{\mathcal{A}(UR_1^\T)}^2_2-\norm{\mathcal{A}(UR_2^\T)}^2_2\right|\leq 2B \norm{\mathcal{A}}^2_2 \frobnorm{U}^2 \norm{(L_1, R_1)-(L_2, R_2)}_2.
     \]
     To bound the term in Line~\ref{eq:quadratic-diff-line-3}, we first observe that
     \begin{align*}
        \langle &\mathcal{A}(UR_1^\T), \mathcal{A}(L_1V^\T)\rangle-\langle \mathcal{A}(UR_2^\T), \mathcal{A}(L_2V^\T)\rangle\\
        &=\langle \mathcal{A}\left(U(R_1-R_2)^\T\right), \mathcal{A}(L_1V^\T)\rangle+\langle\mathcal{A}\left((L_1-L_2)V^\T\right), \mathcal{A}(U R_2^\T)\rangle
     \end{align*}
     Thus
     \begin{align*}
        & \abs{\ip{\mathcal{A}(UR_1^\T), \mathcal{A}(L_1V^\T)}-\ip{ \mathcal{A}(UR_2^\T), \mathcal{A}(L_2V^\T)}}\\
        &=\abs{\langle \mathcal{A}\left(U(R_1-R_2)^\T\right), \mathcal{A}(L_1V^\T)\rangle+\langle\mathcal{A}\left((L_1-L_2)V^\T\right), \mathcal{A}(U R_2^\T)\rangle}\\
        &\leq \norm{\mathcal{A}}^2_2\frobnorm{U}\frobnorm{R_1-R_2}\frobnorm{L_1}\frobnorm{V}+\norm{\mathcal{A}}^2_2\frobnorm{L_1-L_2}\frobnorm{V} \frobnorm{U} \frobnorm{R_2}\\
        &= \norm{\mathcal{A}}^2_2\frobnorm{U}\frobnorm{V} \left(\frobnorm{R_1-R_2}\frobnorm{L_1}+\frobnorm{L_1-L_2}\frobnorm{R_2}\right)
    \end{align*}
    In the ball $\norm{\theta}_2 \leq B$, we have that for $\frobnorm{L_1} \leq B$ and $\frobnorm{R_2}\leq B$. Thus
    \begin{align*}
        \abs{\langle \mathcal{A}(UR_1^\T), \mathcal{A}(L_1V^\T)\rangle&-\langle \mathcal{A}(UR_2^\T), \mathcal{A}(L_2V^\T)\rangle}\\
        &\leq B \norm{\mathcal{A}}^2_2\frobnorm{U}\frobnorm{V}\left(\frobnorm{R_1-R_2}+\frobnorm{L_1-L_2}\right).
    \end{align*}
    Recall that for any $a, b\in R$, it holds that $a+b \leq 2\sqrt{a^2 + b^2}$ and $2ab \leq a^2 + b^2$. Thus
     \[
        \frobnorm{R_1-R_2} + \frobnorm{L_1-L_2} \leq 2 \sqrt{\frobnorm{L_1 - L_2}^2 + \frobnorm{R_1 - R_2}^2} = 2\norm{(L_1,R_1) - (L_2, R_2)}_2
     \]
     and
     \[
        \frobnorm{U} \frobnorm{V} \leq \frac{1}{2}\left(\frobnorm{U}^2  + \frobnorm{V}^2\right).
     \]
     Combining these bounds yields
     \begin{align*}
        \abs{\langle \mathcal{A}(UR_1^\T), \mathcal{A}(L_1V^\T)\rangle&-\langle \mathcal{A}(UR_2^\T), \mathcal{A}(L_2V^\T)\rangle}\\
        &\leq B \norm{\mathcal{A}}^2_2\left(\frobnorm{U}^2 + \frobnorm{V}^2\right)\norm{(L_1,R_1) - (L_2, R_2)}_2.
    \end{align*}
    Lastly, we bound the term in Line~\ref{eq:quadratic-diff-line-4}. We begin by observing that
    \begin{align*}
        \mathcal{A}(L_1 R_1^\T)-\mathcal{A}(L_2 R_2^\T) &= \mathcal{A}(L_1 R_1^\T)-\mathcal{A}(L_1 R_2^\T)+\mathcal{A}(L_1 R_2^\T)-\mathcal{A}(L_2 R_2^\T)\\
        &=\mathcal{A}\left(L_1 (R_1-R_2)^\T\right)+\mathcal{A}\left((L_1-L_2)R_2^\T\right)
    \end{align*}
    Thus
    \begin{align*}
        \abs{\langle \mathcal{A}(L_1 R_1^\T)&-\mathcal{A}(L_2 R_2^\T), \mathcal{A}(UV^\T)\rangle} =\abs{\langle \mathcal{A}\left(L_1 (R_1-R_2)^\T\right)+\mathcal{A}\left((L_1-L_2)R_2^\T\right), \mathcal{A}(UV^\T)\rangle}\\
        &=\abs{\langle \mathcal{A}\left(L_1 (R_1-R_2)^\T\right),\mathcal{A}(UV^\T)\rangle +\langle \mathcal{A}\left((L_1-L_2)R_2^\T\right), \mathcal{A}(UV^\T)\rangle}\\
        &\leq \norm{\mathcal{A}}^2_2\left(\frobnorm{L_1} \frobnorm{R_1-R_2}\frobnorm{U}\frobnorm{V} +\frobnorm{L_1-L_2}\frobnorm{R_2}\frobnorm{U}\frobnorm{V}\right)\\
        &= \norm{\mathcal{A}}^2_2\frobnorm{U}\frobnorm{V}\left(\frobnorm{L_1} \frobnorm{R_1-R_2} +\frobnorm{L_1-L_2}\frobnorm{R_2}\right).
    \end{align*}
    In the ball $\norm{\theta}_2 \leq B$, we have that for $\frobnorm{L_1} \leq B$ and $\frobnorm{R_2}\leq B$. Thus
    \begin{align*}
        \abs{\langle \mathcal{A}(L_1 R_1^\T)-\mathcal{A}(L_2 R_2^\T), \mathcal{A}(UV^\T)\rangle} &\leq B\norm{\mathcal{A}}^2_2\frobnorm{U}\frobnorm{V}\left(\frobnorm{R_1-R_2} +\frobnorm{L_1-L_2}\right).
    \end{align*}
    Recalling the bounds $\frobnorm{V} \frobnorm{U}$ and $\left(\frobnorm{L_1 - L_2}+\frobnorm{R_1 - R_2}\right)$ established above, we have
    \[
        \abs{\langle \mathcal{A}(L_1 R_1^\T)-\mathcal{A}(L_2 R_2^\T), \mathcal{A}(UV^\T)\rangle} \leq B\norm{\mathcal{A}}^2_2\left(\frobnorm{U}^2+\frobnorm{V}^2\right)\norm{(L_1, R_1)-(L_2,R_2)}_2.
    \]
    Employing triangle inequality and each of the above bounds implies
    \begin{align*}
        \left|D^2f(L_1, R_1)[U,V] - D^2f(L_2, R_2)[U,V]\right|&\leq 2\cdot 2B \norm{\mathcal{A}}^2_2 \frobnorm{V}^2 \norm{(L_1, R_1)-(L_2, R_2)}_2\\
        &\quad + 2\cdot 2B \norm{\mathcal{A}}^2_2 \frobnorm{U}^2 \norm{(L_1, R_1)-(L_2, R_2)}_2\\
        &\quad + 4B \norm{\mathcal{A}}^2_2\left(\frobnorm{U}^2 + \frobnorm{V}^2\right)\norm{(L_1,R_1) - (L_2, R_2)}_2\\
        &\quad +4B\norm{\mathcal{A}}^2_2\left(\frobnorm{U}^2+\frobnorm{V}^2\right)\norm{(L_1, R_1)-(L_2,R_2)}_2\\
        &\leq 12 B \norm{\mathcal{A}}^2_2\left(\frobnorm{U}^2+\frobnorm{V}^2\right)\norm{(L_1, R_1)-(L_2,R_2)}_2
    \end{align*}
    which implies the result.
\end{proof}

\subsection{Proofs from Section~\ref{sec:EGOP-spectral-decay}}\label{ssec:spectral-decay-proofs}

In Section~\ref{sec:EGOP-spectral-decay}, we noted that for objectives of the form $f(\theta) = h(A\theta)$, for some loss function $h(\cdot)$ and data matrix $A \in \R^{d\times n}$, the EGOP of $f(\cdot)$ is
    \begin{equation}\label{eq:EGOP-chain-rule}
        \EGOP(f) = A^\T \mathbb{E}_{\theta\sim \rho}\left[\nabla_\theta h(A\theta) \nabla_\theta h(A\theta)^\T\right] A
    \end{equation}

For general loss functions $h(\cdot)$, one can establish the following upper bounds showing how the singular values of $A$ control the EGOP eigepsectrum of $f(\cdot)$.    
    \begin{lemma}\label{lem:data-decay-induces-EGOP-decay}
        Consider $f:\R^n\rightarrow \R$ satisfying $f(\theta) = h(A\theta)$ for some loss function $h:\R^n\rightarrow \R$ and nonsingular data matrix $A \in \R^{n\times n}$. Denote by  $\sigma_i(\cdot)$ and $\lambda_i(\cdot)$ the $i$th singular value and eigenvalue of a matrix respectively, indexed by decreasing value. Then all nonzero eigenvalues of the EGOP of $f(\cdot)$ satisfy
        \[
            \frac{\lambda_k(\EGOP(f))}{ \lambda_1(\EGOP(f))} \leq \left(\frac{\sigma_k(A)}{\sigma_1(A)}\right)^2\frac{\lambda_1(M)}{\lambda_n(M)}.
        \]
        where
        \[
            M\defeq \mathbb{E}_{\theta\sim\rho}[\nabla_\theta h(A\theta) \nabla_\theta h(A\theta)^\T].
        \]
    \end{lemma}
    If the matrix $M$ has some finite condition number that does not go to infinity as the spectral decay in $A$ increases, then Lemma~\ref{lem:data-decay-induces-EGOP-decay} shows that increasing spectral decay in the data matrix $A$ induces spectral decay in the EGOP eigenvalues of $f(\cdot)$.

    For specific choices of $h(\cdot)$, such as $h(\cdot)\defeq \shortnorm{\cdot}^2_2$, one can more precisely characterize how spectral decay in the matrix $A$ induces decay in the EGOP of $f\defeq h\circ A$.
    \begin{lemma}\label{lem:spectral-decay-least-squares}
        Consider $A \in \R^{d\times n}$ and let $f(\theta) = \frac{1}{2}\norm{A\theta- y}^2_2$ where $y = A\theta^* + \eta$, for $\eta$ some mean-zero measurement noise. Assume sampling density $\rho$ is a standard Gaussian distribution. Then the eigenvalues of $\EGOP(f)$,  $\{\lambda_k\}_{k=1}^d$ indexed in decreasing order, satisfy
        \begin{equation}
            \frac{\lambda_k}{\lambda_1} \leq \left(\frac{\sigma_{k-1}(A)}{\sigma_1(A)}\right)^4 \ \forall k\in [2,\dots,n]
        \end{equation}
        where $\sigma_i(A)$ denotes the $i$\ts{th} singular value of $A$, indexed in decreasing order.
    \end{lemma}
    \begin{proof}
        Observe that for any $\theta\in \R^d$,
        \[
            \nabla f(\theta) = A^{\T} (A\theta-y) = A^{\T}(A\theta - A\theta^* - A\eta) = A^{\T} A(\theta-\theta^*) + A\eta.
        \]
        Thus for $\rho$ a standard Gaussian,
        \begin{align*}
            \EGOP(f) &= \mathbb{E}_{\theta\sim \rho}\left[\left(A^{\T} (A\theta-y)\right) \left(A^{\T} (A\theta-y)\right)^{\T}\right]\\
            &=A^{\T} \mathbb{E}_{\theta\sim \rho}\left[A \theta \theta^{\T} A + yy^{\T}\right]A\\
            &=\left(A^{\T} A\right)^2 + (A^{\T}y)(A^{\T} y)^{\T}
        \end{align*}
        using the fact that $\rho$ is mean-zero and isotropic. Thus $\EGOP(f)$ is a rank-1 perturbation of the matrix $(A^{\T} A)^2$, which is itself PSD and has eigenvalues $\sigma_k(A)^4$. Moreover if we let $A = Q_1 \Sigma Q_2^\T$ denote the SVD of $A$, we can rewrite
        \[
            \EGOP(f) = Q_2\left(\Sigma^4 + \Sigma Q_1^{\T} y y^{\T} Q_1 \Sigma \right)V^{\T}
        \]
        where $\Sigma^4$ is diagonal and has entries $\{\sigma_i^4(A)\}_{i=1}^n$. By \citet{golub1973some}, the eigenvalues of $\EGOP(f)$ thus satisfy
        \[
            \sigma_i^4(A) \geq \lambda_{i}(\EGOP(f)) \geq \sigma^4_{i+1}(A) \ \forall i\in[2,\dots,n] 
        \]
        and
        \[
            \sigma_1^4(A) + \norm{\Sigma Q_1^{\T} y}^2_2 \geq \lambda_1(\EGOP)\geq \sigma^4_2(A)
        \]
        which implies the result.
    \end{proof}

\section{Supplementary Figures}\label{sec:supplementary-figures}

Here we present figures deferred from the main body.

\subsection{Supplementary Figures for Section~\ref{sec:EGOP-spectral-decay}}\label{ssec:sup-figs-spectral-decay}

\paragraph{Spectral decay is robust to choice of sampling distribution} In Figure~\ref{fig:tinyMNIST-compare-spectral-decay}, we present evidence that spectral decay visualized in Figure~\ref{fig:tinyMNIST-global-spectral-decay} is robust to choice of sampling distribution. The EGOP whose eigenspectrum is displayed in Figure~\ref{fig:tinyMNIST-global-spectral-decay} was generated from gradient samples $\nabla f(\theta_i)$ for $\theta_i\sim \mathcal{N}(0,\mathbb{I})$, but as visualized in later figures  (Figure~\ref{fig:tinyMNIST-compare-spectral-decay}), the level of spectral decay is comparable to that obtained with Gaussian distributions of differing scales. We also compare to the spectral decay exhibited by the EGOP matrix when $\rho$ is taken to be an initialization distribution used in practice \cite{glorot2010understanding}. For details on this distribution, see Section~\ref{ssec:details-for-main-body-experimental-results}.

\begin{figure*}[t]
    \centering
    \includegraphics[width=.9\linewidth]{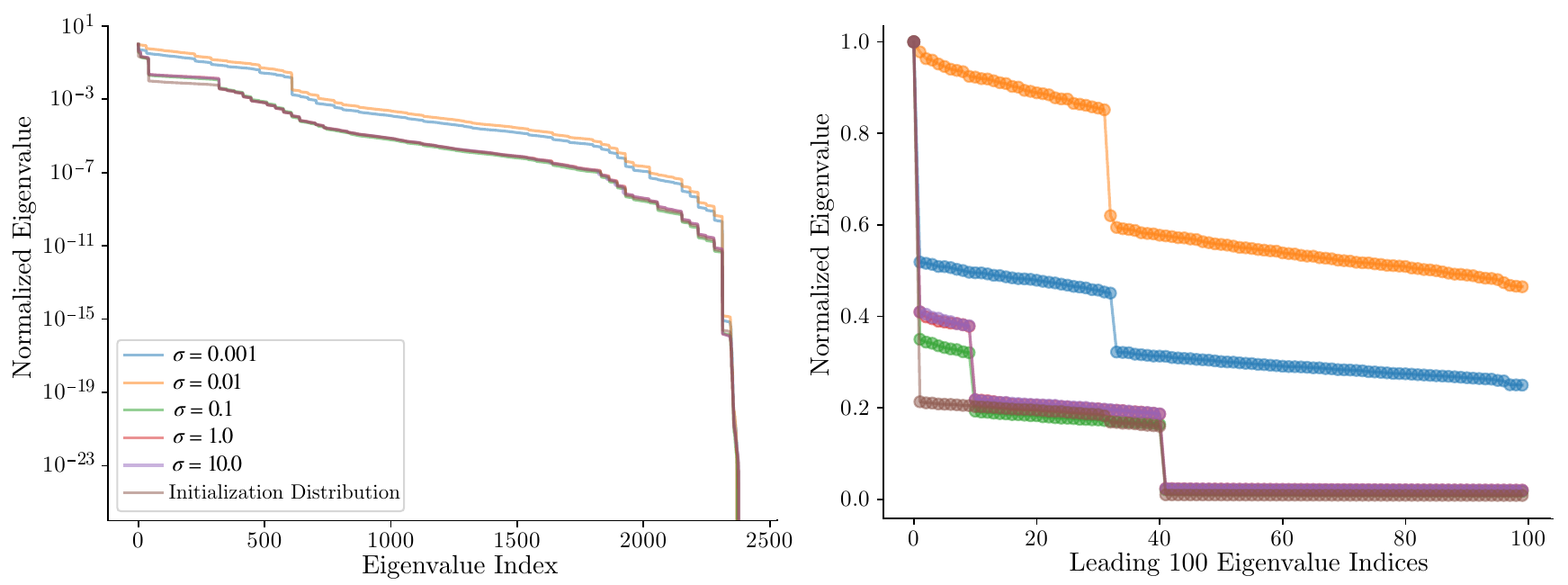}
    \caption{Comparing EGOP spectral decay of a 2-layer ReLU network on tinyMNIST dataset. Plot displays the ratio $\lambda_k/\lambda_1$ as a function of eigenvalue index $k$, for eigenvalues indexed in decreasing order. The blue, orange, green, red, and purple colored traces display the eigenspectrum of the EGOP with respect to a mean-zero Gaussian with covariance $\sigma^2 \mathbb{I}$, for varying values of $\sigma$. The brown trace displays the eigenspectrum of the EGOP with respect to a realistic initialization distribution for this architecture: weights for each layer are drawn from a scaled Xavier normal distribution, and biases are initialized from a scaled uniform distribution (see Section~\ref{ssec:details-for-main-body-experimental-results}). We observe that under all sampling distributions, the eigenspectrum exhibits spectral decay, and that the realistic initialization distribution has spectral decay very comparable to that of the standard Gaussian, displayed in Figure~\ref{fig:tinyMNIST-global-spectral-decay} of the main body. }
    \label{fig:tinyMNIST-compare-spectral-decay}
\end{figure*}

We also note that the interesting shelf structure of the leading eigenvalues is also robust to choice of sampling distribution, as illustrated in Figure~\ref{fig:tinyMNIST-compare-spectral-decay}. Full details for the objective and EGOP estimation procedure for this figure are detailed in Section~\ref{ssec:details-for-spectral-decay}.

\paragraph{Spectral decay persists in block EGOP matrices} Figure~\ref{fig:layer-by-layer-spectra} plots the normalized eigenspectra of the block matrices corresponding to the first and second layers of ReLU networks on the UCI digits dataset and fashionMNIST dataset respectively. For full details on these datasets and the architectures used, see Section~\ref{ssec:details-for-spectral-decay}. Interestingly, both datasets exhibit shared characteristics: the normalized spectral decay in the first layer is strikingly similar, and in both networks the spectral decay in the first layer is more pronounced than in the second layer.

\begin{figure*}[h]
    \centering        \includegraphics[width=0.9\linewidth]{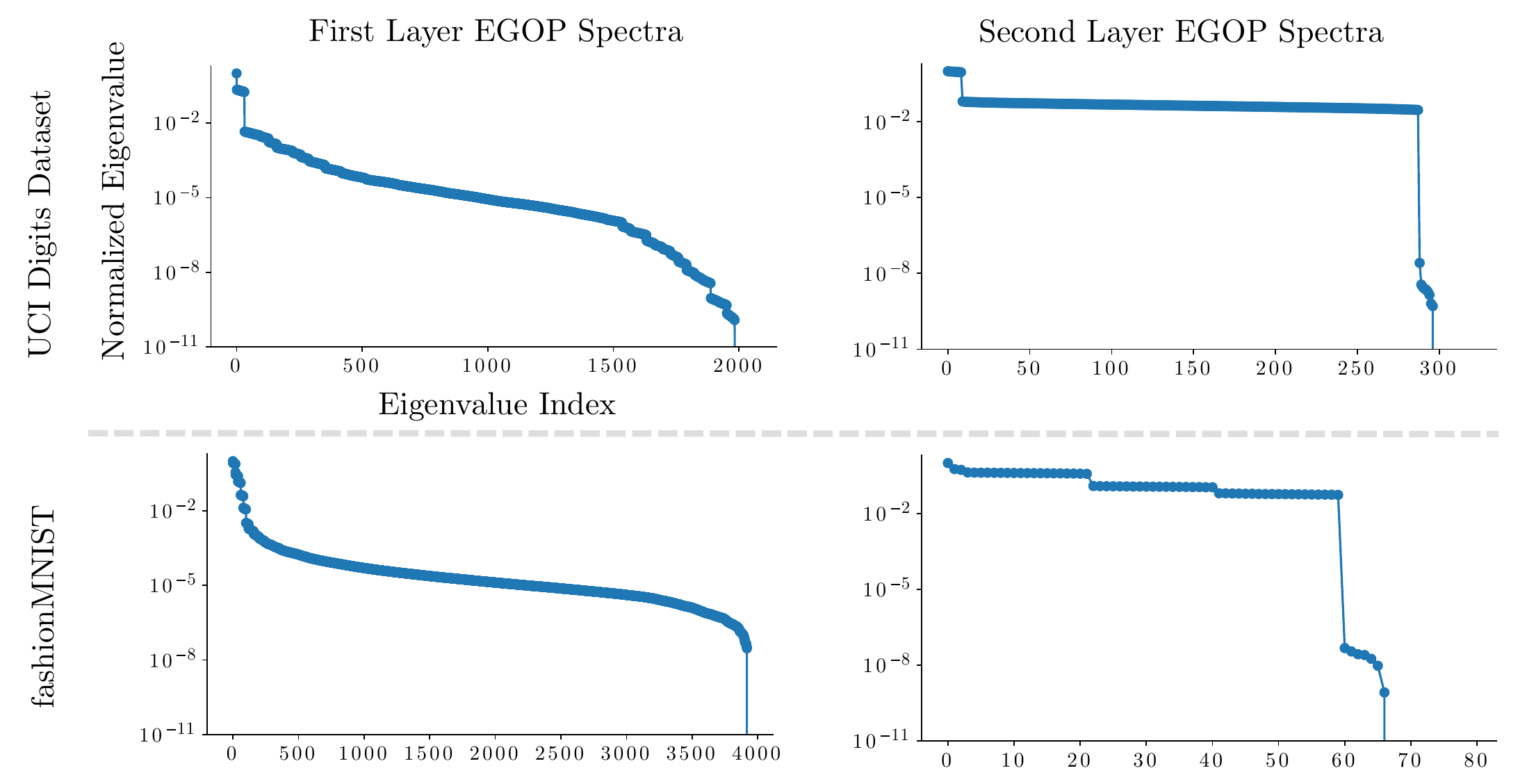}
    \caption{Eigenspectra of the layerwise EGOP matrices of neural networks on \texttt{tinyMNIST} and \texttt{fashionMNIST}. The spectral decay observed in Figure~\ref{fig:tinyMNIST-global-spectral-decay} persists for layer EGOP matrices, defined in Section~\ref{sec:efficient-heuristics}, and across datasets. Y-axes for all figures display identical ranges.}
    \label{fig:layer-by-layer-spectra}
\end{figure*}

\paragraph{Density of leading EGOP eigenvectors} Theorems~\ref{lemma:convex-cvgnce-ball-version}, \ref{lemma:convex-cvgnce} and \ref{lemma:nonconvex-cvgnce} show that when the leading EGOP eigenvector is dense, reparameterized Adagrad enjoys much stronger convergence guarantees. Density of the $k$\ts{th} eigenvector is measured by $\beta_k\defeq \norm{v_k}^2_1$. Our results state a guarantee in terms of $\beta_1$, the density of the leading eigenvector, though as we state in Section~\ref{sec:deferred-proofs} our guarantees can also be formalized in terms of $\beta_k$ the density of the $k$th eigenvector (see e.g. Lemma~\ref{claim:OG-coor-LB}). In Figure~\ref{fig:MNIST-global-eigvec-density}, we confirm that in real-world tasks not only does the leading eigenvector satisfy $\beta_1 \gg 1/d$, but that this also holds for all of the 100 leading eigenvectors.

\begin{figure*}[h]
    \centering        \includegraphics[width=0.5\linewidth]{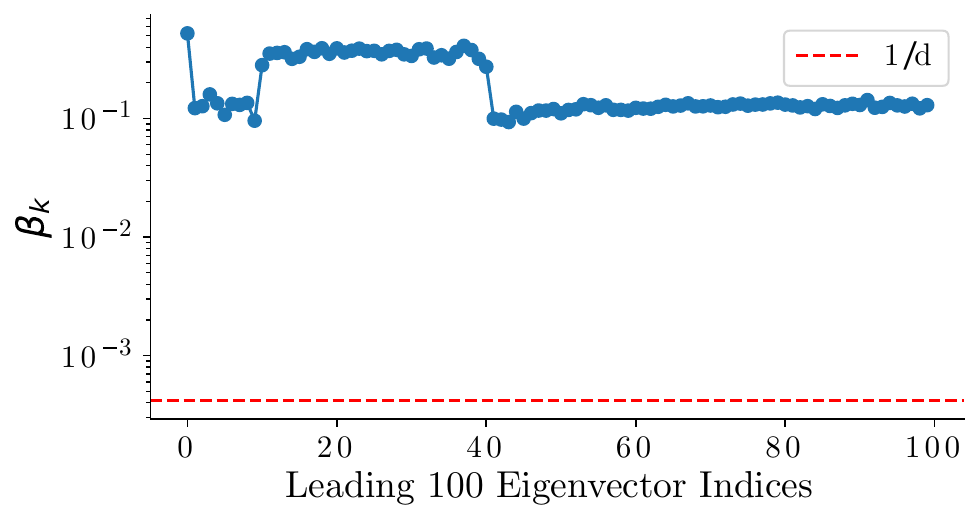}
    \caption{Plotting the density measure $\beta_k\defeq \norm{v_k}^2_1$ for the leading 100 eigenvectors of $\EGOP(f)$, where $f(\cdot)$ is the cross-entropy loss of a 2-layer ReLU neural network on the UCI digits training dataset. The leading eigenvector satisfies $\beta_1 > 0.5$ and several have density $\beta_k > 0.3$. We visualize the value $1/d$ in red  (for this example, $d = 2,410$) to verify that for the leading eigenvectors, $\beta_k \gg 1/d$.}
    \label{fig:MNIST-global-eigvec-density}
\end{figure*}

In order for the factors $\sr(f)/d(\beta_1-\delta)$ in the bound for reparameterized Adagrad from Theorems~\ref{lemma:convex-cvgnce-ball-version}, \ref{lemma:convex-cvgnce} and \ref{lemma:nonconvex-cvgnce} to reflect a strong improvement, it is necessary that $\beta_1\gg1/d$. In Figure~\ref{fig:MNIST-global-eigvec-density}, we show that for the UCI digits dataset, this condition is satisfied. Specifically, we show that $\beta_1 > 0.5$, while for this network $1/d < 5e-4$. Moreover, we show that not only does the leading eigenvector satisfy this density assumption, but several of the leading eigenvectors satisfy $\beta_k \gg1/d$.

\subsection{Supplementary Figures for Section~\ref{sec:experimental-results}}\label{ssec:sup-figs-main-experiments}

\paragraph{Residual Networks} In this section, we supplement the results shown in Figure~\ref{fig:new-resnet-fig-plus-table}. In Figure~\ref{fig:ResNet-sup-val} we plot validation over epochs. These results show that EGOP reparameterization offers a benefit over AdamW in original coordinates in terms of validation loss minimization, not only training loss minimization. Also, comparing EGOP reparameterization to SOAP, we find that EGOP reparameterization minimizes validation loss at a comparable per-epoch rate to SOAP and achieves a lower minimum validation loss (the median validation loss achieved by periodic EGOP reparameterization is $1.19$, compared with $1.22$ for SOAP).

\begin{figure}[t]
    \centering
    \includegraphics[width=0.55\linewidth]{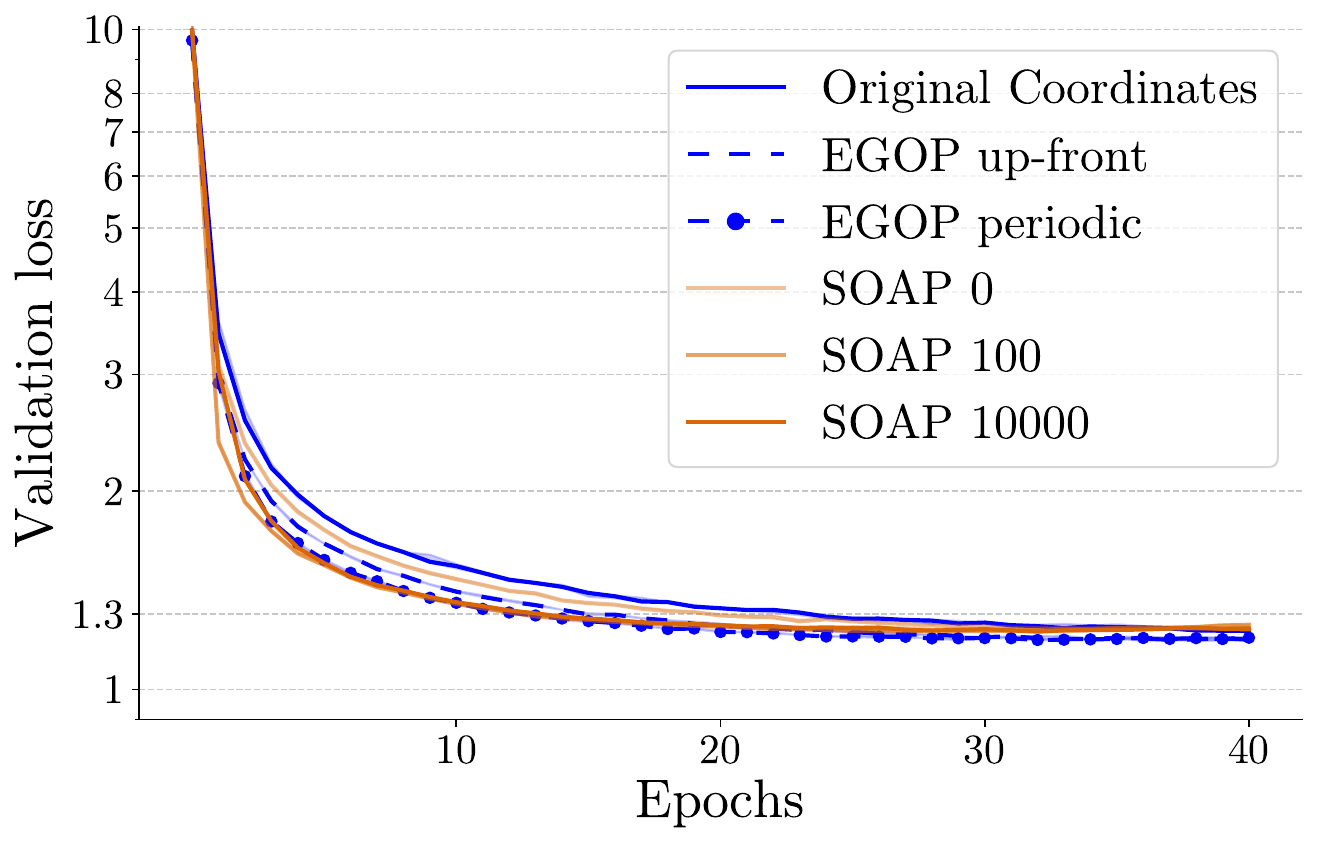}
    \caption{Validation loss for image classification with residual networks.
    }
    \label{fig:ResNet-sup-val}
\end{figure}

In Figure~\ref{fig:ResNet-sup}, we plot training and validation accuracy over epochs, confirming that the improved training loss convergence under reparameterization leads to improved classification accuracy. Moreover, we confirm that EGOP reparameterization is competitive with SOAP across both accuracy and loss metrics.

\begin{figure}[h]
    \centering
    \begin{subfigure}[t]{0.49\textwidth}
        \centering
        \includegraphics[width=\linewidth]{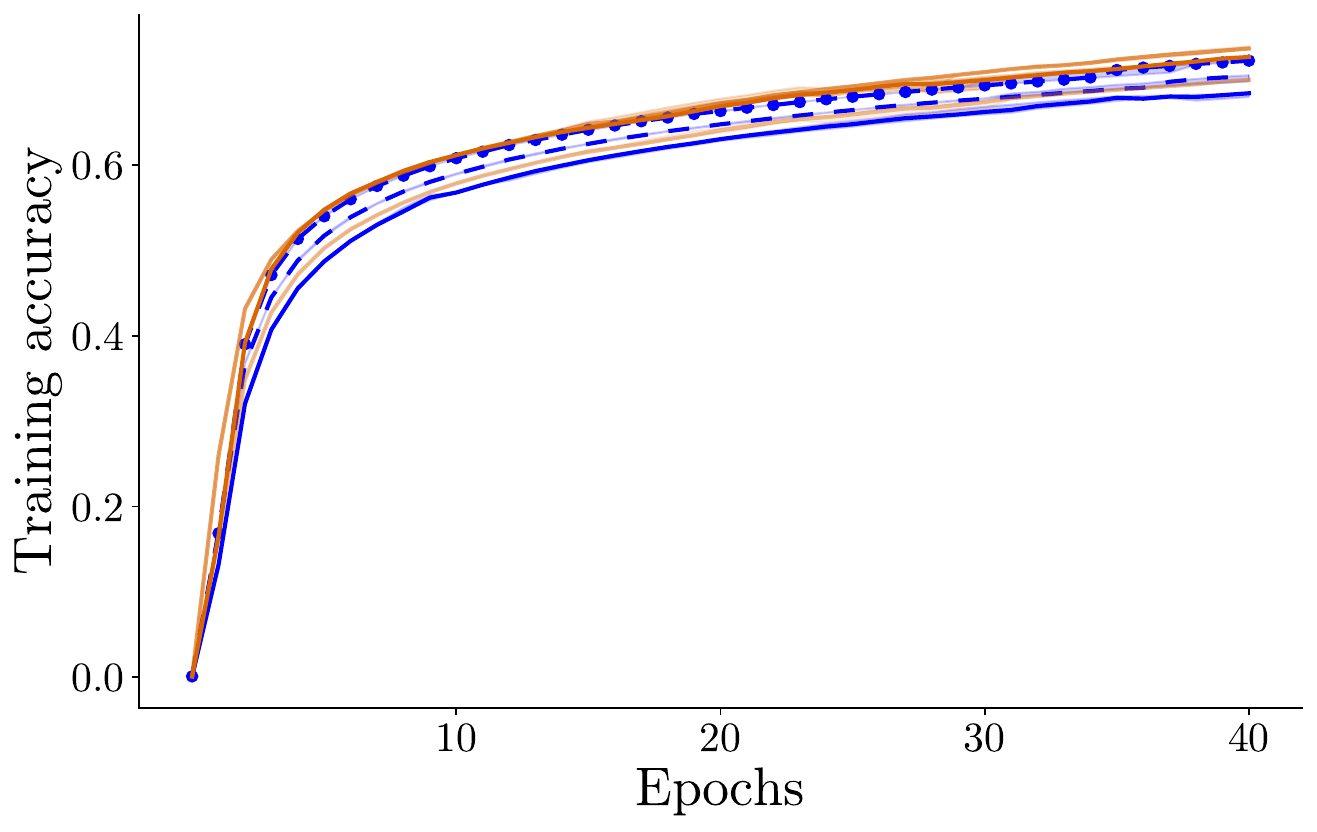}
    \end{subfigure}
    \begin{subfigure}[t]{0.49\textwidth}
        \centering
        \includegraphics[width=\linewidth]{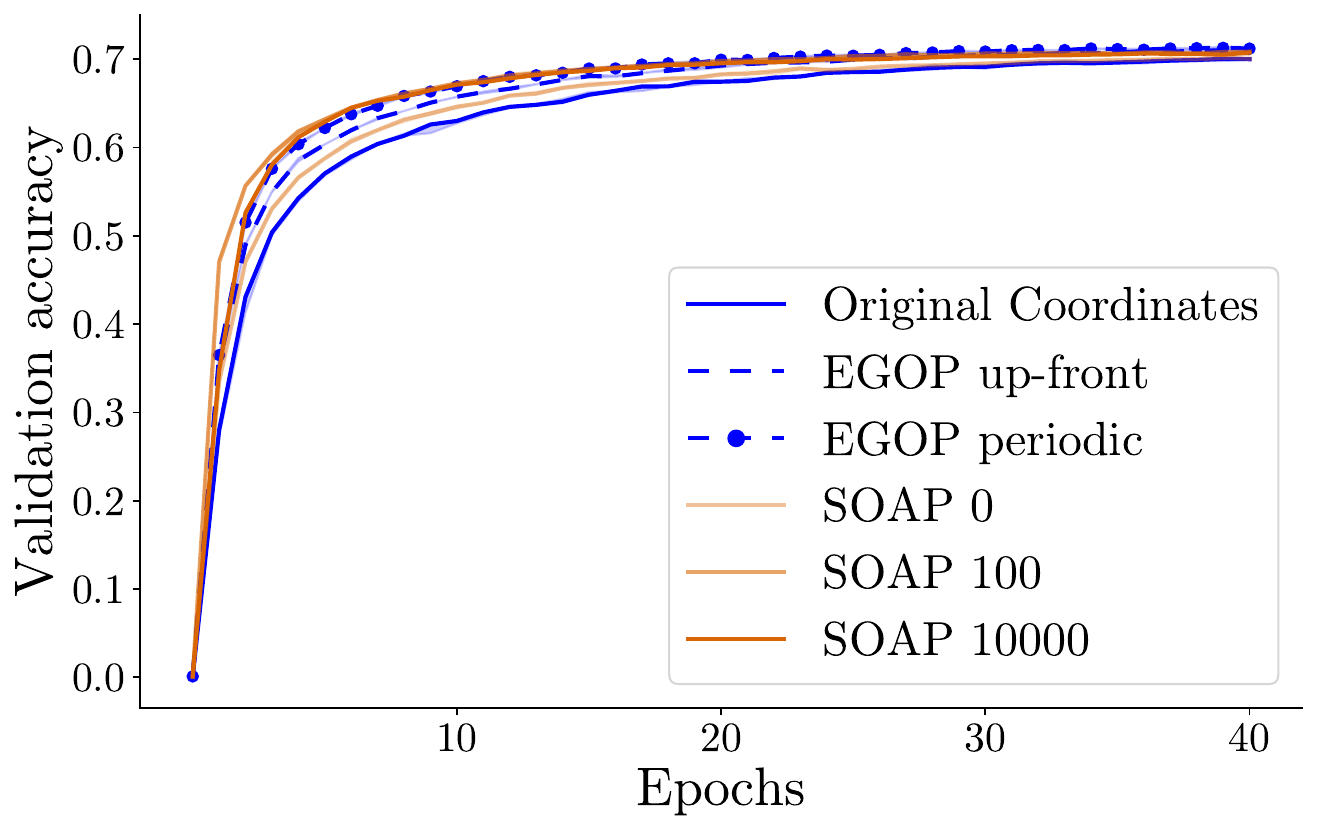}
    \end{subfigure}
    \caption{Training and validation accuracy for image classification with residual networks, corresponding to the experiments discussed in Section~\ref{ssec:main-body-large-scale}.}\label{fig:ResNet-sup}
\end{figure}

\paragraph{Multilayer linear networks} We compare three methods for EGOP reparameterization in order to examine some heuristics proposed in Section~\ref{sec:efficient-heuristics}. In Figure~\ref{fig:sup-linear-networks}, we consider training a multilayer linear network (\ref{eq:linear-feedforward-objective}) under  global EGOP reparameterization, wherein all parameters are reparameterized simultaneously as in \cref{alg:meta-algorithm-block} (Figure~\ref{fig:global-sup});  block reparameterization for all layers, following the procedure defined in Section~\ref{sec:efficient-heuristics} (Figure~\ref{fig:block-sup}); and block reparameterization of only the parameters in the first layer (Figure~\ref{fig:first-layer-only-sup}). For all three methods, we estimate the EGOP using the same number of gradient samples: $M = 2d$, where $d$ is the total number of network parameters.

Comparing Figure~\ref{fig:global-sup} with Figure~\ref{fig:block-sup} shows that for this problem, EGOP reparameterization offers comparable benefit when using block reparameterization as when using global reparameterization. This suggests that block EGOP reparameterization, which has a reduced computational cost compared to global EGOP reparameterization, may be an effective way to accelerate adaptive methods when problem instances are too large to permit global reparameterization. Reparameterizing only the first layer (Figure~\ref{fig:first-layer-only-sup}) improves Adagrad's performance by a margin comparable to that of global and block reparameterization of all layers, but the benefit to Adam under reparameterization of only the first layer is much less pronounced.

\begin{figure*}[h]
    \centering
    \begin{subfigure}[t]{0.32\textwidth}
        \centering
        \includegraphics[width=\linewidth]{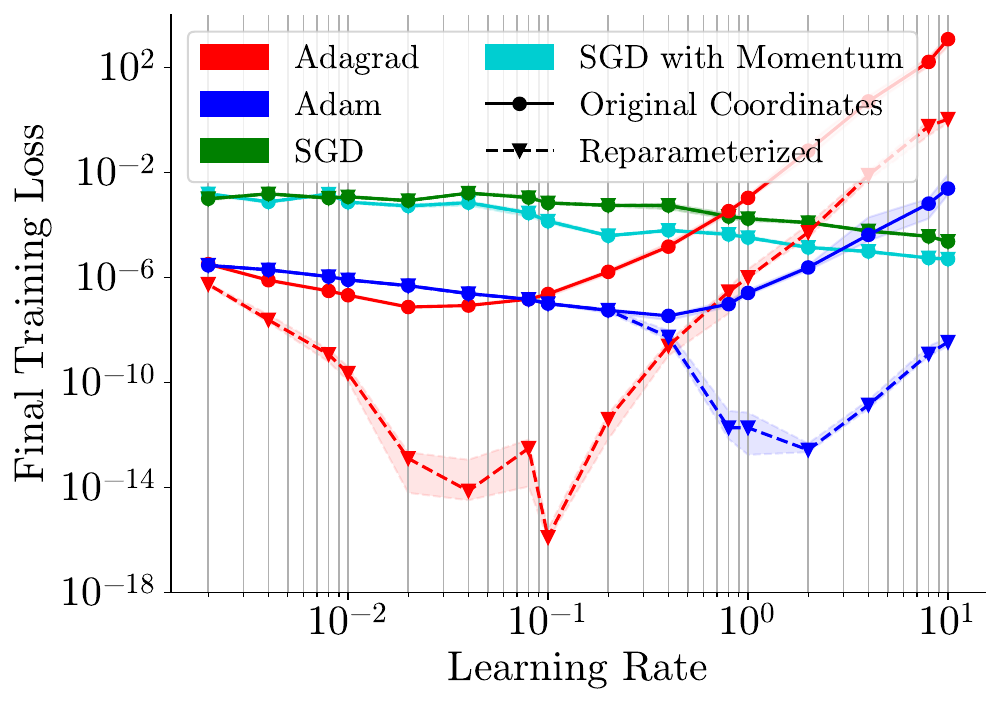}
        \caption{\centering Global EGOP Reparameterization}
        \label{fig:global-sup}
    \end{subfigure}
    \begin{subfigure}[t]{0.32\textwidth} 
        \centering
        \includegraphics[width=\linewidth]{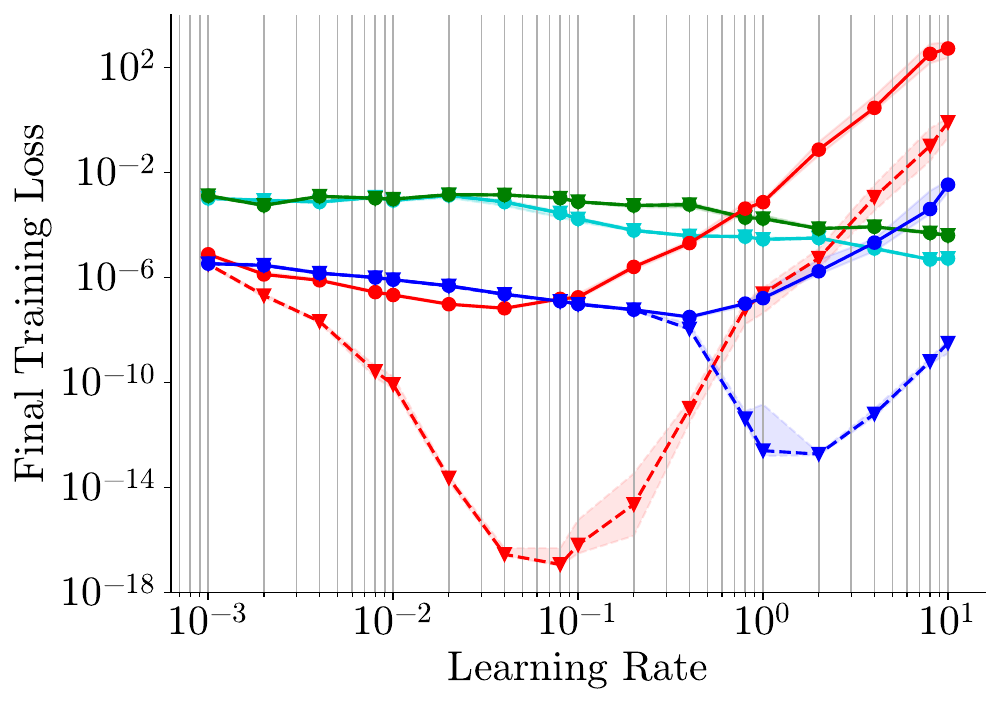}
        \caption{\centering  Block reparameterization, all layers}
        \label{fig:block-sup}
    \end{subfigure}
    \begin{subfigure}[t]{0.32\textwidth} 
        \centering
        \includegraphics[width=\linewidth]{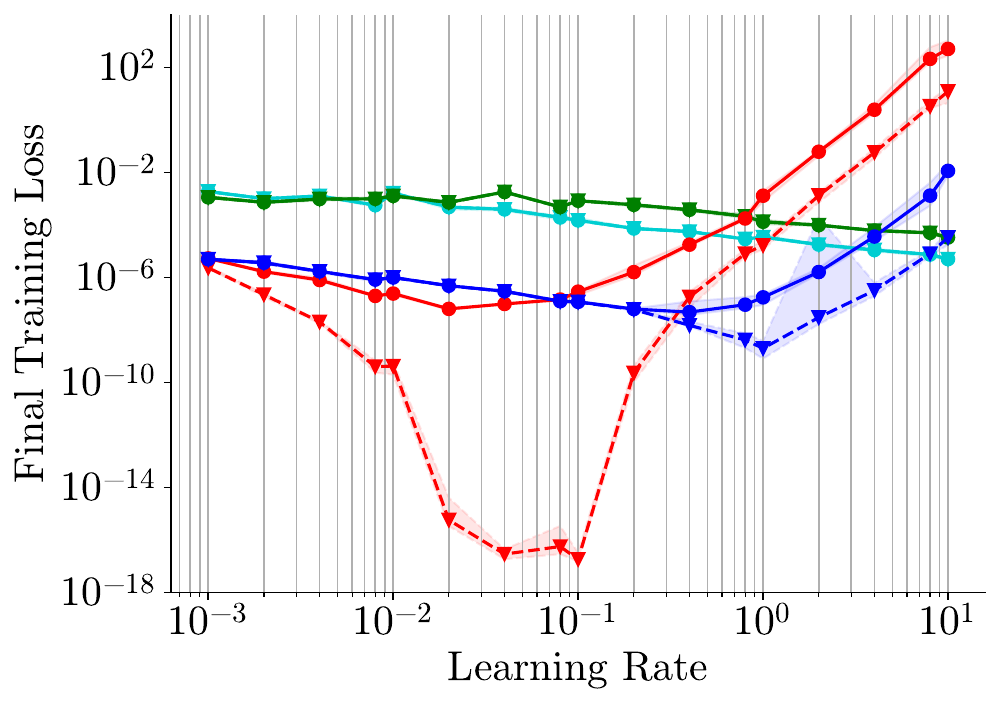}
        \caption{\centering Block reparameterization, first layer only}
        \label{fig:first-layer-only-sup}
    \end{subfigure}
    \caption{Comparing three  EGOP reparameterization methods for training a multilayer linear network (\ref{eq:linear-feedforward-objective}). Figure~\ref{fig:global-sup} is a reproduction of Figure~\ref{fig:linear-layers-global-reparam-loss-vs-LR} in the main body, and shows results for performing EGOP reparameterization of all parameters simultaneously. Figure~\ref{fig:block-sup} shows results when performing block EGOP reparameterization, where each network layer forms a block. Figure~\ref{fig:first-layer-only-sup} shows results when block-reparameterizing only the parameters in the first layer.}
    \label{fig:sup-linear-networks}
\end{figure*}

\paragraph{Image Classification with ReLU Networks} Figure~\ref{fig:tinyMNIST-sup} expands on the results shown in Figure~\ref{fig:opener-cartoon}  (right) for 2-layer ReLU networks on the UCI digits dataset. Figure~\ref{fig:tinyMNIST-loss-sup} plots final training loss versus learning rate, and shows that benefit of reparameterization shown in Figure~\ref{fig:opener-cartoon} is robust to choice of learning rate. Figure~\ref{fig:tinyMNIST-testacc-sup} plots final test accuracy versus learning rate, and shows that the improved training offered by reparameterization does not lead to over-fitting. For full experimental details on the architecture, dataset, and loss function used for these experiments, see Section~\ref{sec:experimental-details}.

\begin{figure*}[h]
    \centering
    \begin{subfigure}[t]{0.45\textwidth}
        \centering
        \includegraphics[width=\linewidth]{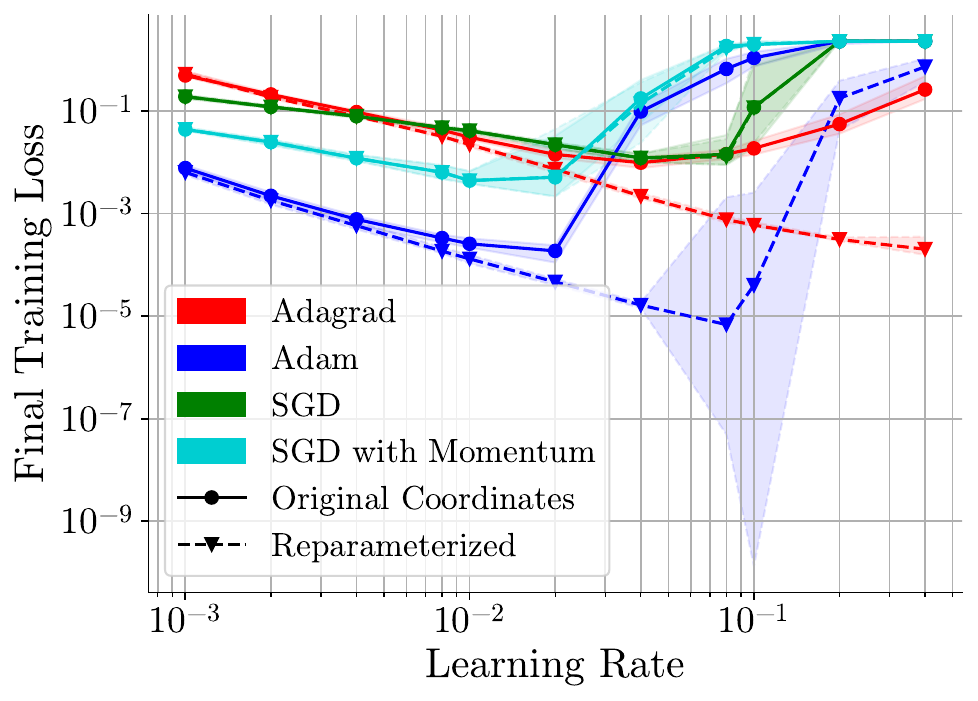}
        \caption{\centering Median final training loss versus learning rate}
        \label{fig:tinyMNIST-loss-sup}
    \end{subfigure}
    \hspace{0.05\textwidth} 
    \begin{subfigure}[t]{0.45\textwidth} 
        \centering
        \includegraphics[width=\linewidth]{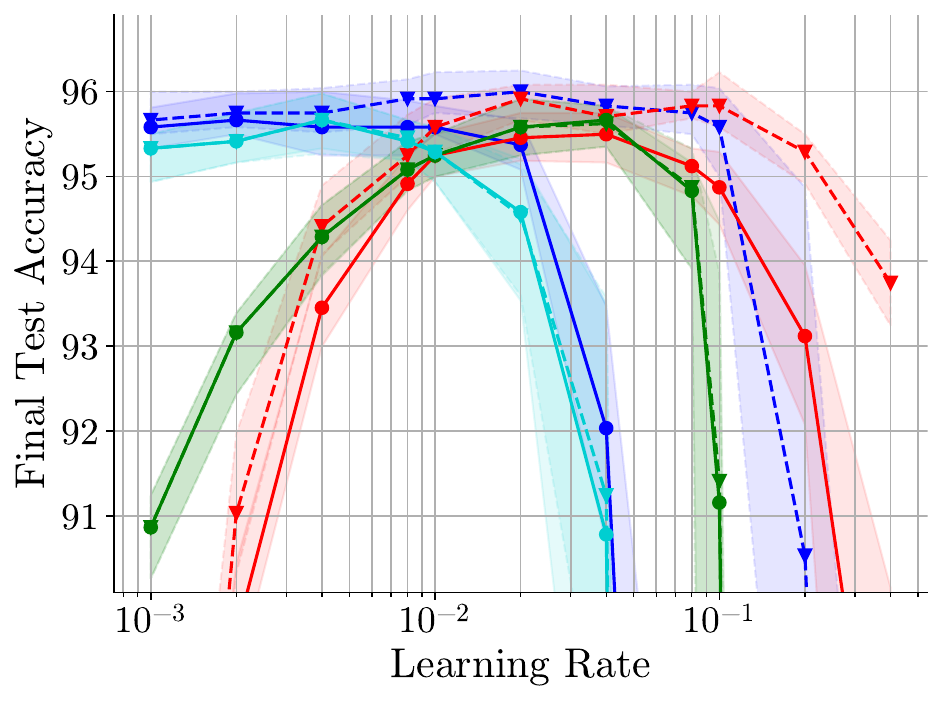}
        \caption{\centering Median final test accuracy versus learning rate}
        \label{fig:tinyMNIST-testacc-sup}
    \end{subfigure}
    \caption{Training loss and test accuracy results for the UCI digits dataset image classification task. Results are aggregated over independent trials corresponding to different random initializations. Medians are plotted as traces, and shaded regions indicate the 25\ts{th}-75\ts{th} percentiles.
     } \label{fig:tinyMNIST-sup}
\end{figure*}

\begin{figure*}[h]
    \centering        \includegraphics[width=\linewidth]{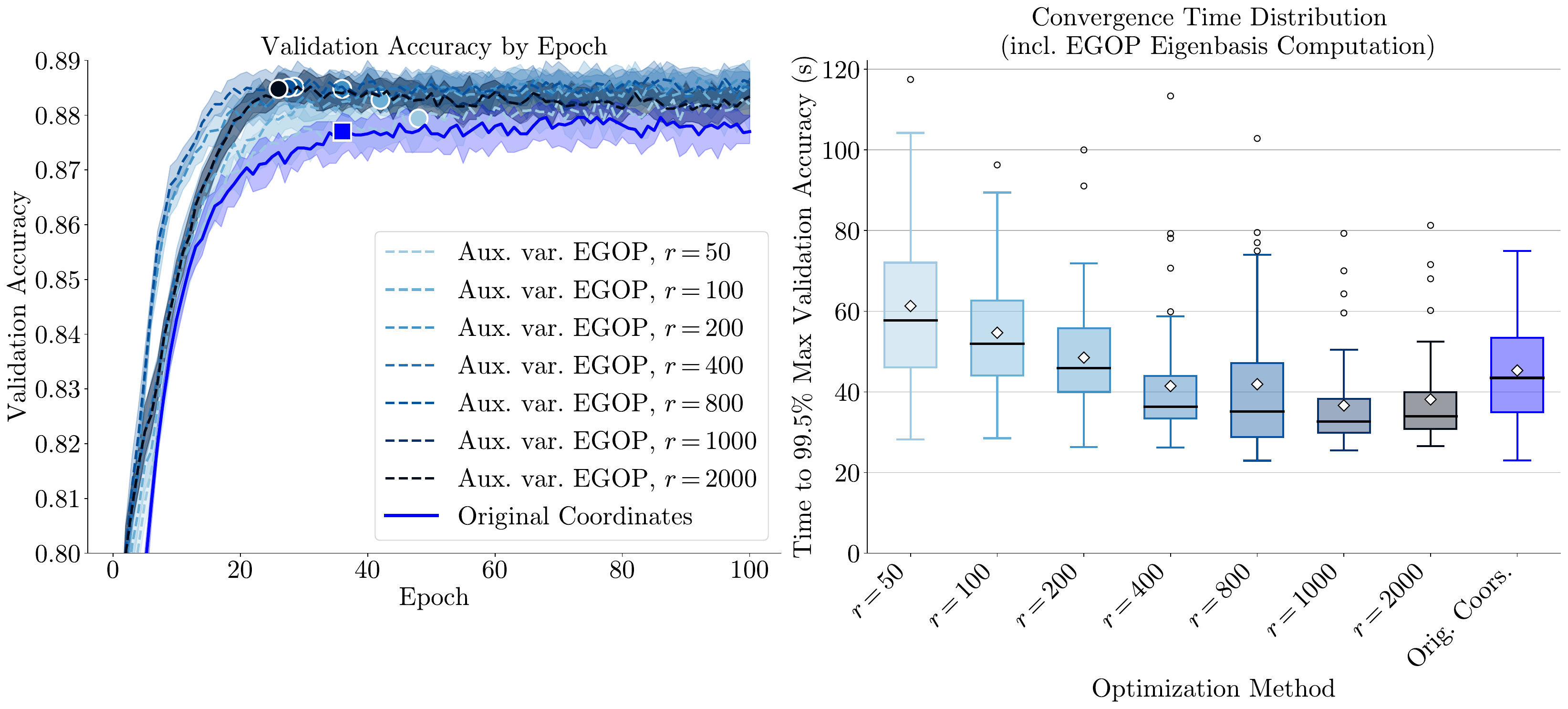}
    \caption{Validation accuracies and timing results for image classification on fashionMNIST. Counterpart to Figure~\ref{fig:fashionMNIST}. (Left) Plotting validation accuracy demonstrates that the improved minimization of training loss leads to improved generalization. (Right) Plotting wall-clock time to convergence for the model in original coordinates, as well as under auxiliary variable EGOP reparameterization, as described in Section~\ref{sec:efficient-heuristics}, for varying values of $r$. We find that the cost-benefit tradeoff favors intermediate values of $r$, e.g. $r=1000$ in this figure. Recall that for this architecture, $r = 1000$ is much smaller than the number of optimization variables in the largest network layer, which contains $78.4k$ weights.}
    \label{fig:auxilliary-cost-benefit-tradeoff}
\end{figure*}

Similarly, Figure~\ref{fig:auxilliary-cost-benefit-tradeoff} expands on the results presented in Figure~\ref{fig:fashionMNIST}. Figure~\ref{fig:auxilliary-cost-benefit-tradeoff} (left) shows that the improved training does not lead to over-fitting, but rather that reparameterization leads to improved accuracy on hold-out data. Figure~\ref{fig:auxilliary-cost-benefit-tradeoff} (right) reports wallclock time to convergence, demonstrating that the improved convergence in epochs shown in Figure~\ref{fig:fashionMNIST} translates to faster convergence in wallclock time. For this experiment, we define convergence as the first epoch at which the model attains $99.5\%$ of its maximum validation accuracy. In Figure~\ref{fig:auxilliary-cost-benefit-tradeoff} (left), we use square a circular markers to indicate the epochs at which the model in original coordinates and reparameterized models respectively achieve this threshold.

\paragraph{Convex Objectives} In addition to training neural networks, which is a primary application of adaptive optimization algorithms, we also study reparameterization for convex optimization. Figure~\ref{fig:log-sum-exp-sup} shows the result of minimizing
\begin{equation}\label{eq:log-sum-exp-objective}
    f(\theta) = \log\left(\sum_{i=1}^n\exp\left((\langle a_i, \theta\rangle -y_i)^2\right)\right),
\end{equation}
where $a_i\in \R^d$ denote vectors of observations and $y_i \defeq \langle a_i, \theta^*\rangle$ for ground truth $\theta^*$. We also plot results from minimizing logistic regression objectives (Figure~\ref{fig:log-reg-sup}) and linear least-squares objectives (Figure~\ref{fig:linear-least-squares-sup}) arising from problems with data matrices $A\in \R^{n\times d}$. For all convex objectives, we induce EGOP spectral decay by choosing matrices $A$ with singular value decay. For these objectives, we estimate the EGOP and perform optimization using noiseless (full-batch) gradients. 

Figure~\ref{fig:sup-cvx-objectives} demonstrates that EGOP reparameterization can improve convergence of adaptive algorithms to global minima of convex objectives.
EGOP-reparameterized Adagrad outperforms its counterpart in original coordinates for all three objectives. Notably, methods with momentum excel for this logistic regression, and reparameterization boosts the performance of Adagrad (which does not use momentum), making it competitive with the momentum-based optimizers. EGOP reparameterization improves Adam's convergence on the log-sum-exp objective and linear least-squares objectives (Figures~\ref{fig:log-sum-exp-sup} and \ref{fig:linear-least-squares-sup}), but has no impact on Adam's performance for logistic regression (Figure~\ref{fig:log-reg-sup}).

\begin{figure*}[h]
    \centering
    \begin{subfigure}[t]{0.34\textwidth}
        \centering
        \includegraphics[width=\linewidth]{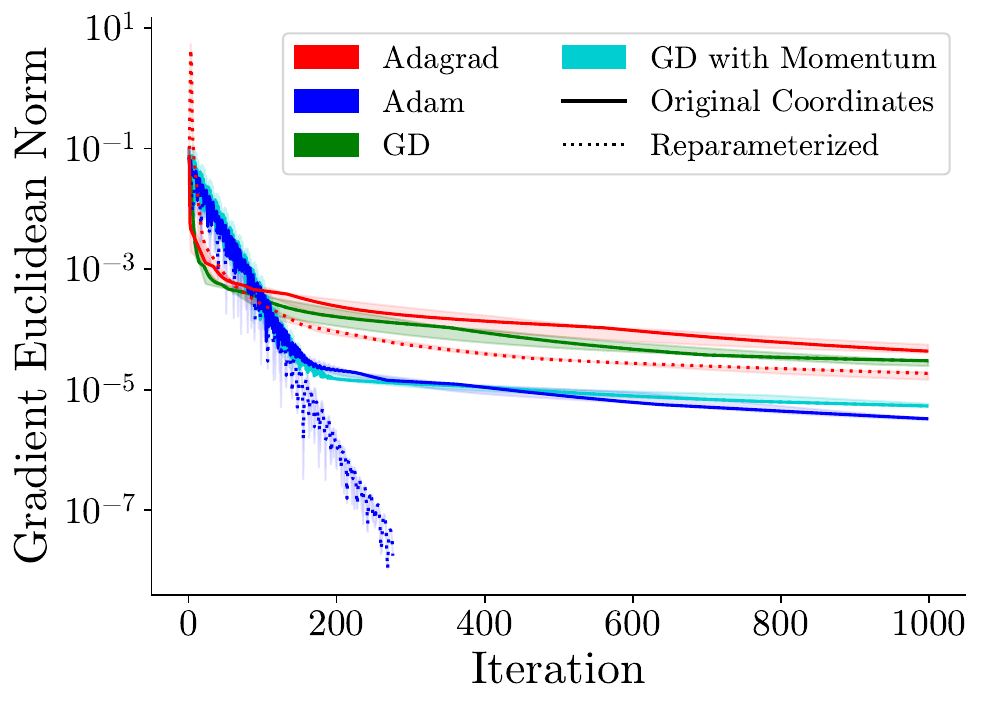}
        \caption{\centering Log-sum-exp with $\alpha = 2$}
        \label{fig:log-sum-exp-sup}
    \end{subfigure}
    \begin{subfigure}[t]{0.315\textwidth} 
        \centering
        \includegraphics[width=\linewidth]{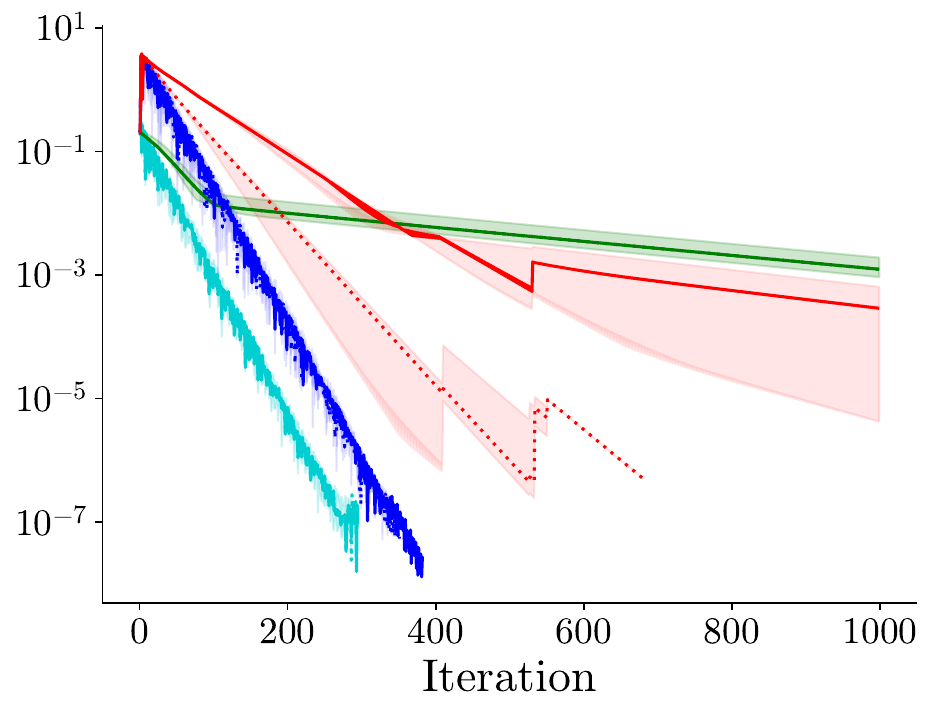}
        \caption{\centering Logistic regression with $\alpha = 3$}
        \label{fig:log-reg-sup}
    \end{subfigure}
    \begin{subfigure}[t]{0.315\textwidth} 
        \centering
        \includegraphics[width=\linewidth]{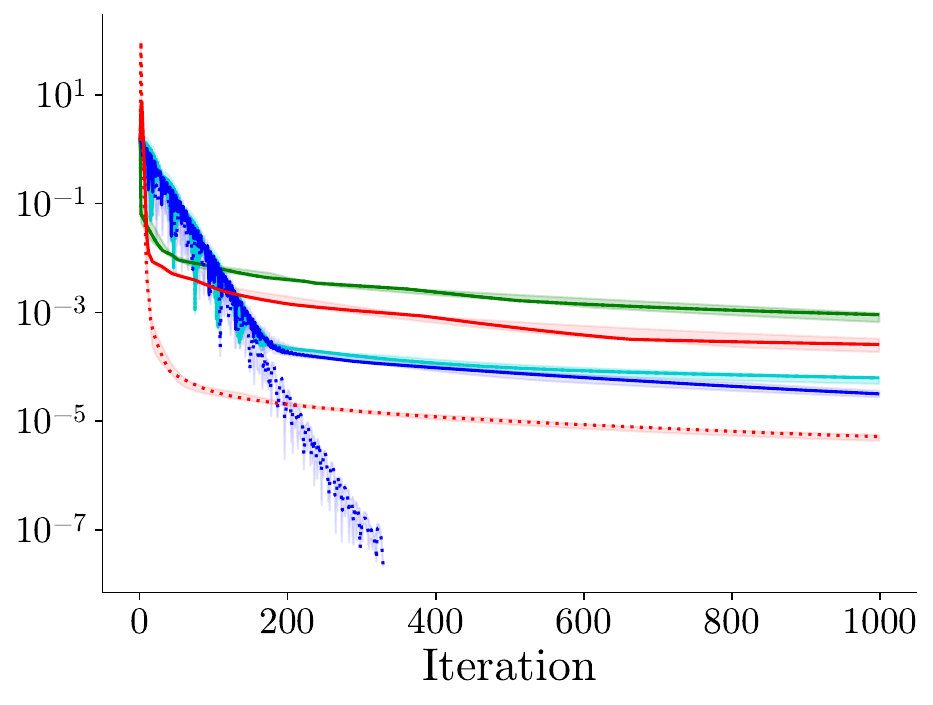}
        \caption{\centering Linear least-squares with $\alpha = 2$}
        \label{fig:linear-least-squares-sup}
    \end{subfigure}
    \caption{Gradient Euclidean norm of solution at $t$\ts{th} iterate. Learning rates tuned separately for methods in original coordinates and under reparameterization. We induce EGOP spectral decay by choice of data matrix $A$ with singular values $\sigma_k(A) = k^{-\alpha}$. As noted in the prose, in some plots the dotted traces coincide with the solid and are thus not visible (Adam in  Figure~\ref{fig:log-reg-sup}, and equivariant methods GD and GD with momentum).}
    \label{fig:sup-cvx-objectives}
\end{figure*}

\subsection{Optimizing with Weight Decay}\label{ssec:L2-reg}

In Section~\ref{sec:convergence-analysis}, we instantiated guarantees for a ball constraint set, and remarked that in practice adaptive optimization algorithms are often employed along with weight decay/L2 regularization, which implicitly constrains the algorithms to some ball in parameter space centered at the origin. We note that our experiments on residual networks and on image classification using fashionMNIST both employ AdamW, an adaptive algorithm with explicit weight decay, thus validating that EGOP reparameterization can also improve convergence when weights are implicitly constrained to some ball.

\section{Experimental details}\label{sec:experimental-details}

Experiments were implemented in python, and both code and data will be made publicly available upon publication. Experiments with neural networks were implemented in Pytorch, and experiments with convex objectives were implemented using auto-differentiation in Jax and optimizers from Optax \cite{paszke2017automatic, jax2018github, deepmind2020jax}. With the exception of experiments performed on residual networks, all experiments can be performed on CPU with a MacBook Pro with an Intel Core i5 processor and 16GB memory. However, all experiments can be run more quickly on GPU, and GPU support is included in the codebase. Experiments with residual networks were performed on an internal cluster, using 1 RTX 6000 Ada GPU, and 8 CPU cores on an AMD EPYC 74F3 processor. The timing results with auxiliary variable reparameterization for ReLU networks in Figures~\ref{fig:fashionMNIST} and \ref{fig:auxilliary-cost-benefit-tradeoff} were performed on an internal cluster using 1 RTX A6000 GPU and 2 CPU cores on an Intel Xeon Silver 4114 processor.

\subsection{Dataset Details}\label{ssec:dataset-details}

\paragraph{UCI handwritten digits dataset} The UCI digits dataset \cite{optical_recognition_of_handwritten_digits_80} contains 5,620 instances, each an $8\times 8$ pixel greyscale image of a handwritten digit of values $0,\dots, 9$. Each instance has a integer label, $0,\dots, 9$. We split the dataset into a training dataset with 3,823 instances, a validation dataset with 598 instances, and a test dataset with 1,199 instances.

\paragraph{fashionMNIST dataset} The fashionMNIST dataset consists of $28\times 28$ pixel grayscale images of clothing, each having a label from one of 10 classes. The full dataset contains 60k training samples and 10k test samples; for the results in Figure~\ref{fig:fashionMNIST}, we use the full training set and we subdivide the test set into a validation set of 2.5k samples and a test set of 7.5k samples \cite{xiao2017fashion}. For the experiments in Figure~\ref{fig:layer-by-layer-spectra}, we restrict to only instances corresponding to the first four classes (labeled ``t-shirt,'' ``trouser,'' ``pullover,'' and ``dress''). This yields a dataset of 28,000 instances, which we subdivide into a training set of 24,000 instances, a validation set of 1,000 instances, and a test set of 3,000 instances. We do this in order to reduce the problem size to a setting where we can visualize the full EGOP eigenspectrum obtained without using heuristics.

\paragraph{ImageNet dataset}
For our large-scale experiments, we use the ImageNet classification dataset \cite{deng2009imagenet}. ImageNet consists of approximately 1.28 million training images and 50,000 validation images spanning 1,000 object categories. Following standard practice, we train on the full training set and report results on the validation set. During training, all models use the conventional preprocessing pipeline of random resized crops to $224\times224$ and horizontal flips. For evaluation, we resize each image so that its shorter side is 256 pixels and use a ten-crop evaluation protocol (four corners, center crop, and their horizontal flips). ImageNet serves as a challenging benchmark to evaluate the scalability and efficiency of our method in deep residual architectures.

\subsection{Architecture Details}\label{ssec:architecture-details}

\paragraph{Architecture for multilayer linear networks} To implement the objective 
\begin{equation}
    f(\theta) = \frobnorm{W_3 W_2 W_1 A - Y}^2 /\nsamples
\end{equation}
we use a 3-layer fully connected network without any nonlinear activations. In accordance with the description in Section~\ref{sec:experimental-results}, the first layer has 10 input nodes and 30 output nodes, the second layer has 50 output nodes, and the last layer has 10 output nodes.

\paragraph{Architecture for UCI handwritten digits}  For the UCI digits dataset we use a 2-layer ReLU network. Its first layer is a  fully-connected linear layer with bias, with 64 input nodes and 32 output nodes, followed by a ReLU nonlinearity. Its second layer is a fully-connected linear layer with bias, with 32 input nodes and 10 output nodes, followed by a log-softmax. We form the function $f(\cdot)$ by taking the negative log-likelihood loss on the training dataset.

\paragraph{Architecture for fashionMNIST dataset} For results reported in Figure~\ref{fig:fashionMNIST}, we use a 2-layer ReLU network. The first layer is a fully connected linear layer with bias, with 784 input nodes and 100 output nodes, followed by a ReLU activation function. The second layer is a fully-connected linear layer with bias, with 100 input nodes and 10 output nodes. We reparameterize each fully connected layer using the auxiliary variable reparameterization method defined in Section~\ref{sec:reduced-egop}. We define the function $f(\cdot)$ as the cross-entropy loss computed on the fashionMNIST training dataset.

For the eigenspectrum results visualized in Figure~\ref{fig:layer-by-layer-spectra}, we reduce the problem size in order to study a regime where computing the full EGOP eigenspectrum is tractable. We restrict to 4 classes and use the train-validation-test split, as described in Section~\ref{ssec:dataset-details}. Additionally, we down-sample the input to our network from 28 × 28 pixels to 14 × 14 pixels using a 2D max-pooling operation as the first layer of our network. The next layer is then a fully-connected linear layer with bias, with 196 input nodes and 20 output nodes, followed by a ReLU activation function. The last layer is a fully-connected linear layer with bias, with 20 input nodes and 4 output nodes. 

\paragraph{Architecture for ImageNet} For the ImageNet experiments, we use the 34\,layer residual network (ResNet-34) architecture presented in the original paper \cite{he2016resnet}. The model begins with a convolutional layer consisting of a $7\times 7$ convolution with 64 channels and stride~2, followed by batch normalization, a ReLU activation, and a $3\times 3$ max pooling layer with stride~2. The network then contains four stages of residual blocks: the first stage has three residual blocks with 64 channels; the second stage has four residual blocks with 128 channels, with the first block using stride~2 for spatial downsampling; the third stage has six residual blocks with 256 channels, again with stride~2 in the first block; and the final stage has three residual blocks with 512 channels, with stride~2 in the first block. Each residual block consists of two $3\times 3$ convolutional layers with batch normalization and ReLU activations, together with a skip connection. After the residual stages, we apply global average pooling and a fully connected layer mapping from 512 features to 1000 output classes. We define the function $f(\cdot)$ as the cross entropy loss computed on the ImageNet training dataset.

\subsection{Algorithm Details}\label{ssec:algorithm-details}

We consider four optimization algorithms: Adagrad, Adam, SGD, and SGD with momentum. We use standard settings for all hyperparameters except learning rate: for Adam, we use $\beta_1 = 0.9$ and $\beta_2 = 0.999$, and for SGD with momentum we use momentum parameter $0.9$, matching that of Adam. When using AdamW, we use default settings  $\beta_1 = 0.9$ and $\beta_2 = 0.999$, and tune the weight decay parameter. 

We tune learning rates for each algorithm. Unless otherwise noted, when choosing the range of learning rates to sweep for tuning, we the following two-step procedure. First we sweep a coarse sweep using powers of 10 (e.g. 1e-3, 1e-2, $\dots$ , 100) to find an upper and lower bound on the learning rates that produce best performance. The metric by which we quantify the ``best performance'' for each different experiment is discussed below in Section~\ref{ssec:details-for-main-body-experimental-results}. We then perform a refined sweep, where we discretize each interval between subsequent power of 10 using doubling. For example, if the coarse sweep identified lower bound 0.01 and upper bound 1.0, we would perform the refined sweep over values [0.01, 0.02, 0.04, 0.08, 0.1, 0.2, 0.4, 0.8, 1.0].

\paragraph{Auxiliary Variables} For the auxiliary variables reparameterization experiments, we use AdamW with $\beta_1 = 0.9$ and $\beta_2 = 0.999$. As described in Section~\ref{sec:reduced-egop}, the auxiliary variables method allows the model to represent parameters in the full $d$-dimensional space rather than being constrained to the reduced $r$-dimensional subspace $\text{span}(V_r)$. We therefore expect auxiliary variables EGOP to achieve better performance than reduced-dimension EGOP. We apply weight decay to both $\theta_r$ and $\theta_d$ parameters, but not to the bias terms. Similarly to the learning rate sweep procedure described in the preceding paragraph, we perform a two-stage coarse and fine grain grid search to select learning rates and weight decay, ultimately sweeping learning rates $[0.00001, 0.00005, 0.0001, 0.0005, 0.001, 0.005, 0.01, 0.05, 0.1]$ and weight decay values $[0, 0.00001, 0.0001, 0.001, 0.01, 0.05, 0.1, 0.2, 0.5]$, training each combination for 30 epochs and selecting the hyperparameters that achieve the highest validation accuracy.

\paragraph{ResNet for ImageNet}
Following the standard ResNet training protocol, we use hyperparameters comparable to those reported in the original work. Although the original ResNet work \cite{he2016resnet} trains using SGD with momentum~$0.9$, our experiments use AdamW. To preserve comparable smoothing of gradient information, we set $\beta_1 = 0.9$, which plays a role analogous to momentum in SGD, together with the standard choice $\beta_2 = 0.999$. We use the same weight decay parameter as the original ResNet work, namely $10^{-4}$. For learning rate selection, we perform a logarithmic sweep over the interval $[10^{-5}, 1]$, training each candidate learning rate for 10~epochs and selecting the value that achieves the highest top-1 validation accuracy under ten-crop evaluation. After selecting the learning rate, we train the final model for 40~epochs using a batch size of 128 and the standard ImageNet data augmentations described in Section~\ref{ssec:dataset-details}. We chose 40~epochs because, under a fixed learning rate without decay, the validation losses consistently plateau by this point, indicating that the model has reached its first stage of convergence.

To compute the EGOP matrix for convolutional layers, we repeatedly reinitialize the network weights and record the gradients obtained from a single forward–backward pass at initialization. Because a single pass produces one gradient sample for each kernel in a convolutional layer, a layer with $c$ output channels yields $c$ gradient samples per pass. We therefore determine the required number of passes by ensuring that every convolutional layer accumulates at least 1000 gradient samples in total. Let $c_{\min}$ denote the smallest number of output channels across all convolutional layers. Since this layer contributes only $c_{\min}$ samples per pass, we perform $\lceil 1000 / c_{\min} \rceil$ forward–backward passes. In each pass, we compute the gradient of the cross-entropy loss with respect to all convolutional and linear weights.

\subsection{Additional Details for Figures in Section~\ref{sec:intro}}\label{ssec:details-for-opener-cartoon}

For Figure~\ref{fig:opener-cartoon} (left), we generate the loss landscape pictured by taking $d=2$ dimensions and $n=100$ samples. We generate a matrix $A$ with singular values $\sigma_k = k^{-\alpha}$ with $\alpha = 1.5$ and random right- and left-singular vectors. We sampled $\theta^* \sim \mathcal{N}(0, \mathbb{I})$ and used $A, \theta^*$ to induce $f(\cdot)$ a log-sum-exp objective, as defined in Section~\ref{sec:experimental-results}. We generated problem instances at random, and chose a random seed that produced a problem where the primary directions of variation for $f(\cdot)$ were clearly un-aligned with the coordinate axes. We allow Adagrad to take 1000 iterations.

We selected an initial point $\theta_0$ that would produce a clear visual distinction between the two coordinate systems in early iterates. To select the learning rate, we first swept over the values $1, 10, 20,\dots, 100$, and examined the suboptimality of the solution produced after 1000 iterates. For each learning rate we conducted 5 random trials on log-sum-exp objectives generated with identical values of $d, n, \alpha$, and with $\theta^*\sim \mathcal{N}(0,\mathbb{I})$; these trials were initialized at points drawn from $\mathbf{N}(0, \norm{\theta_0}^2_2\mathbb{I})$, where $\theta_0$ is the initial point used in Figure~\ref{fig:opener-cartoon} (left). We found that for learning rates 30 through 90, the suboptimality of the solution returned by the algorithm in original coordinates was very comparable across learning rates and very close to zero ($<1e-10$). Thus we selected a learning rate with these properties that was on the lower end (hence 30) because at larger learning rates, the oscillations around the global minimum made it more difficult to visually assess the difference between trajectories before and after reparameterization.

\subsection{Additional Details for Figures in Section~\ref{sec:EGOP-spectral-decay}}\label{ssec:details-for-spectral-decay} 

Figure~\ref{fig:tinyMNIST-global-spectral-decay} plots the eigenspectrum of an empirically estimated EGOP matrix. We use the UCI digits dataset and consider the training subset described in Section~\ref{ssec:dataset-details}. We use the 2-layer ReLU network detailed in Section~\ref{ssec:architecture-details}. 

For Figure~\ref{fig:tinyMNIST-global-spectral-decay} we use full-batch gradients to estimate the EGOP matrix. We take $\rho$ to be a standard normal distribution: for each gradient sample, we form $\theta$ by sampling the entries of all weights and biases i.i.d. from a standard Gaussian. We estimate $\EGOP(f)$ using $M = 10d$, where $d$ is the total number of parameters (sum of all weight and bias entries) for the network. This is a larger number of EGOP samples than we use in later experiments; this is done intentionally with the goal of clearly resolving the spectral decay, rather than having decay appear as an artifact of numerical estimation with few samples.

For Figure~\ref{fig:tinyMNIST-compare-spectral-decay}, we use the same dataset, architecture, objective $f(\cdot)$, and procedure for EGOP estimation, but we use different sampling distributions $\rho$. The ``realistic initialization'' distribution is described in full detail below in Section~\ref{ssec:details-for-main-body-experimental-results}.

For Figure~\ref{fig:layer-by-layer-spectra}, we use the same architecture and objective for the UCI digits dataset. For fashionMNIST, we use the train-validation-test split, as described in Section~\ref{ssec:dataset-details}, and use the architecture described in Section~\ref{ssec:architecture-details}. Wwe use minibatches of size 300 to estimate gradients. For each architecture, we use $M = 5d$ gradient samples to estimate each block EGOP matrix, where $d$ is the number of parameters in the network. This is a larger number of EGOP samples than we use in later experiments; this is done intentionally with the goal of clearly resolving the spectral decay, rather than having decay appear as an artifact of numerical estimation with few samples. We perform block EGOP reparameterization for both networks, where the blocks are defined by the weight matrices of each network. For more details on block reparameterization for neural network weights, see the example in Section~\ref{sec:expanded-heuristics-discussion}.

When drawing gradients to estimate the block EGOP matrices, the distribution over parameters $\rho$ from which we draw is the same as the distribution we later use to initialize the networks during training. For a full description of this initialization distribution, see Section~\ref{ssec:details-for-main-body-experimental-results}.

\subsection{Additional Details for Figures in Section~\ref{sec:experimental-results}}\label{ssec:details-for-main-body-experimental-results}

\paragraph{Linear Feedforward Networks} As described in Section~\ref{sec:experimental-results}, for Figure~\ref{fig:global-reparam-linear-layers} we consider parameters $\theta = [\textrm{vec}(W_1), \textrm{vec}(W_2), \textrm{vec}(W_3)]$ where $W_1 \in \R^{50\times 10}$, $W_2 \in \R^{30\times 50}$, $W_3 \in \R^{10\times 30}$. We seek to minimize loss
\[
    f(\theta) = \frobnorm{W_3 W_2 W_1 A - Y}^2 /\nsamples
\]
where $A\in \R^{10\times \nsamples}$, and $Y = M^* A$ for $M^*\in \R^{10\times 10}$ drawn from a standard Gaussian distribution. We induce spectral decay in $\EGOP(f)$ by generating $A$ with singular values $\sigma_k(A) = k^{-2}$ and random right- and left-singular vectors. For each trial, we generate $20,000$ data samples, which we split into a test set of $10,000$ training data instances, $4,000$ validation instances, and $6,000$ test instances. We define the training loss to be $f(\cdot)$ restricted to the training instances and their labels, and use this function when estimating the EGOP matrix. We use stochastic mini-batches of size 500 when estimating the EGOP and when performing optimization.

For each trial, we generate $A, M^*$, and $Y$ as described above. We then estimate $\EGOP(f)$ using $M = 2d$ samples, where $d$ is the total number of parameters in the network ($d=2,300$). When estimating the EGOP, we let $\rho$ be the same distribution used when initializing networks for training. Specifically, we initialize each weight matrix from a mean-zero Xavier normal distribution, also called Glorot initialization, which is widespread in practice \cite{glorot2010understanding}. The Xavier normal distribution is a Gaussian with standard deviation $\sqrt{2/\FIFO}$, where for a matrix $W \in \R^{n\times m}$, $\FIFO = n+m$. We compute the full eigenvalue decomposition to find the change-of-basis matrix $V$.

We use 1000 epochs during training. See Section~\ref{ssec:algorithm-details} for choice of  hyperparameters and choice of the range of learning rates to sweep for tuning. We measure performance using the median minimum validation loss achieved during training. We perform 10 independent trials at each learning rate. For each algorithm, we choose the learning rate at which the algorithm in original achieved lowest median minimum validation loss. Results of this learning rate sweep are in Figure~\ref{fig:linear-layers-global-reparam-valloss-vs-LR}.

Because Adam in original coordinates exhibited numerical instability due to the near-zero gradient values encountered, we used cosine annealing to decay the learning rate over the coarse of training. We use Pytorch's default implementation of cosine annealing. For ease of comparison, we apply the same learning rate decay schedule to all algorithms (Adagrad, Adam, SGD, and SGD with momentum) in both original and reparameterized coordinates. We use this learning rate decay schedule throughout learning rate tuning, and in all results displayed in Figure~\ref{fig:global-reparam-linear-layers}.

\paragraph{ReLU Networks for Image Classification} We use the UCI digits dataset and fashionMNIST, with the train-validation-split described in Section~\ref{ssec:dataset-details} and architectures described in Section~\ref{ssec:architecture-details}. We use the same (random) partition of datapoints into train, validation, and test datasets for all trials. 

For both networks and objectives, we use minibatches of size 300 to estimate gradients during both EGOP estimation and optimization. We perform block EGOP reparameterization for both networks, where the blocks are defined by the weight matrices of each network. For more details on block reparameterization for neural network weights, see the example in Section~\ref{sec:expanded-heuristics-discussion}. For the UCI digits dataset, we use $M=d$ minibatch gradient samples to estimate the EGOP, where $d$ is the number of parameters in the network, and compute the orthonormal reparameterization matrices for each layer by taking the (full, deterministic) SVD of the matrix of gradient samples for that layer, as this is equivalent to taking the eigenvalue decomposition of the empirical block EGOP matrix. 

For the fashionMNIST experiments reported in Figures~\ref{fig:fashionMNIST} and ~\ref{fig:auxilliary-cost-benefit-tradeoff},  we use pytorch's \texttt{svd\_lowrank} function to implement randomized SVD, which in turn is based off Algorithms 4.1 and 5.1 in Halko et al. \cite{halko2011finding}. This yields $V_r \in \R^{d_1\cdot d_2\times r}$ a matrix with orthonormal columns. We use this matrix to perform auxiliary variable reparameterization, as introduced in Section~\ref{sec:efficient-heuristics}. We collect $M = 0.01d$ minibatch gradient samples, where $d$ is the total number of parameters in the model. For each layer, we form the gradient matrix $G \in \mathbb{R}^{d_1 \cdot d_2 \times M}$ and use randomized SVD to compute the leading $r$ left-singular vectors. See Section~\ref{sec:expanded-heuristics-discussion} for full pseudocode for auxiliary reparameterization. We use the same value of $r$ for all layers and vary $r$ to study the effect of subspace dimension on convergence speed. 

For each objective and network, we use the same distribution over parameters for drawing gradient samples to estimate the EGOP and for initializing the network at training time. For the experiments in Figure~\ref{fig:opener-cartoon} (right), we draw weight entries i.i.d. from a mean-zero Xavier distribution, and bias entries i.i.d. from a uniform distribution with range 
\[
    [-(\FIFO)^{-1/2},(\FIFO)^{-1/2}].
\]
For the experiments in Figure~\ref{fig:fashionMNIST}, we draw weight entries i.i.d. from a standard Gaussian distribution $\mathcal{N}(0, 1)$, and bias entries i.i.d. from a uniform distribution with range 
\[
    [-(\textrm{in-features})^{-1/2}, (\textrm{in-features})^{-1/2}].
\]

All of these initialization distributions stem from conventions that are widespread in practice. The distinction between the weight distributions used with the digits dataset versus the fashionMNIST dataset was the result of using legacy code; we did not tune initializations to different applications. 

We selected hyperparameters and tuned learning rates following the procedures discussed in Section~\ref{ssec:algorithm-details}. We used a constant learning rate for both datasets, as all methods appeared stable. We measure performance using the median maximum validation classification accuracy achieved during training. When tuning learning rates for the UCI digits task, for each algorithm we choose the learning rate at which the algorithm in original coordinates achieved highest median maximum validation accuracy. For the digits dataset, we performed 50 independent trials at each learning rate.

For fashionMNIST, we optimize using AdamW. We performed a single trial for each combination of learning rate and weight decay, and only trained for 30 epochs during tuning. 

Once the combination of learning rate and weight decay was selected, we set the number of epochs over which to train by examining loss and accuracy curves by eye, and choosing a value large enough that all algorithms had converged. For the digits dataset, we trained for 200 epochs, and for fashionMNIST we trained for 100 epochs.

\paragraph{Regression with Errors-in-Variables} As described in Section~\ref{sec:experimental-results}, we optimize
\[
    f(\theta) = \frac{1}{2} \norm{A(M\odot W) - Y}^2_F.
\]
For each trial, the data are generated as follows. We set the problem dimension to $n = d = 10$ and sample the design matrix $A$ so that its singular values obey $\sigma_k = k^{-\alpha}$ with $\alpha = 2.0$. For $\nu \geq 0$, each entry $M_{ij}$ is drawn independently from $\mathcal{N}(1, \nu^2)$. The ground truth parameter $W^*$ is sampled entrywise from a standard normal distribution, and the target $Y$ is set as $Y = A (M \odot W^*)$.

For optimization, we compare the performance of SOAP, Shampoo, base Adam and Adagrad, and EGOP-reparameterized Adam and Adagrad. All parameters are initialized at zero, and each optimization run proceeds for 1000 steps using a cosine annealing learning rate schedule. We employ this learning rate schedule in order to improve SOAP's convergence. On independent trials, a new random instance of $A$, $M$, $W^*$, and $Y$ is generated, on which each optimizer is evaluated. 

We perform a logarithmic learning rate sweep for each optimizer and each $\nu$. In every trial, the training loss trajectory is recorded. We measure performance for each learning rate using the median final training loss across the five trials. The best learning rate for each setting is taken to be the one with the lowest median final loss.

For EGOP reparameterization, we sample 1000 gradients from parameters $W$, where the entries of $W$ are sampled from a standard normal distribution. 

\paragraph{Low-rank Matrix Factorization} As stated in Section~\ref{sec:experimental-results}, we consider minimizing 
\[
    f(\theta) = \frac{1}{2}\norm{A(LR^T) - Y}^2_F.
\]
where $L \in \mathbb{R}^{n \times r}$ and $R \in \mathbb{R}^{m \times r}$, and $A \in \R^{n\times n}$. On each trial, the synthetic ground truth $X^*$ is generated to have rank $r$ and a prescribed condition number $\kappa$. Specifically, we factor $X^*$ as $X^* = U \Sigma V^\top$ with random orthonormal matrices $U \in \mathbb{R}^{n \times r}$, $V \in \mathbb{R}^{m \times r}$, and singular values $\Sigma = \mathrm{diag}(\sigma_1, \ldots, \sigma_r)$, for values forming a linear grid from $1$ to $1/\kappa$. The measurement operator $A$ is also constructed as a random matrix with controlled singular value decay: its singular values follow $\sigma_k = k^{-\alpha}$. Both $X^*$ and $A$ are independently regenerated for each trial. Each observation matrix $Y$ was generated using $Y = A X^* + \mathcal{N}(0,\, \sigma_{\text{obs}}^2)$. All experiments are performed with $n=m=20$, $r=5$, $\kappa = 2$, and $\alpha = 1/2$, and the standard deviation of the observation noise was set to $\sigma_{\text{obs}} = 0.01$.

We use spectral initialization: to form $L_0, R_0$, we solve for $A^\dag Y$ where $A^\dag$ indicates the Moore-Penrose pseudoinverse, perform an SVD, and truncate to rank $r$ to obtain $U_r, \Sigma_r, V_r$, and form left- and right- factors $L_0 = U_r \Sigma_r{-1/2}$,  $R_0 = V_r \Sigma_r^{-1/2}$. SOAP, Shampoo, and base Adam and Adagrad are initialized at $(L_0, R_0)$, and EGOP-reparameterized methods are initialized at the appropriate equivalent points in the coordinates of the EGOP-eigenbasis.

For each optimizer, we perform logarithmic learning rate sweeps, performing five independent trials on each. For each optimizer, we selected the learning rate achieving the lowest median final loss. All parameters are optimized for $1000$ steps starting from this initialization. Figure~\ref{fig:LRMF} reports results aggregated over $5$ independent trials, corresponding to problem instances \textit{independent} from any problems generated during the learning rate sweep. We use a cosine annealing learning rate schedule both in learning rate sweeps and in the results reported in Figure~\ref{fig:LRMF}; this learning rate schedule was selected to stabilize and promote best performance in momentum-based methods (particularly SOAP, Adam, and EGOP-reparameterized Adam), but the overarching trends persist when the learning rate schedule is removed or modified.

\paragraph{Convex objectives} We study three objectives. For each, we generate a matrix $A \in \R^{\nsamples\times d}$ with singular value decay $\sigma_k = k^{-\alpha}$ and random orthonormal right- and left-singular vectors. The values of $\alpha$ are specified in the caption of Figure~\ref{fig:sup-cvx-objectives}. The log-sum-exp objective is defined in Eq.~\ref{eq:log-sum-exp-objective}. The logistic regression objective is defined as
\[
    f(\theta) = \sum_{i=1}^{\nsamples} -y_i \log\left(\frac{1}{1+e^{-\langle a_i, \theta\rangle}}\right) -(1-y_i)\log\left(\frac{e^{-\langle a_i, \theta\rangle}}{1+e^{-\langle a_i, \theta\rangle}}\right)
\]
where $\{a_i\in \R^d\}_{i=1}^{\nsamples}$ are the columns of $A$, and labels $y_i\sim \operatorname{Bernoulli}(\pi(\theta^*)_i)$ where $\theta^* \in \R^d$ is drawn from $\mathcal{N}(0, \mathbb{I})$ and 
\[
    \pi(\theta^*)_i \defeq \frac{e^{-\langle a_i, \theta^*\rangle}}{1+e^{-\langle a_i, \theta^*\rangle}}.
\]
The logistic regression objective is 
\[
    f(\theta) = \frac{1}{2}\norm{A\theta-y}^2_2
\]
where $y=A\theta^*$ and $\theta^*\in \R^d$ is drawn from $\mathcal{N}(0, \mathbb{I})$.

For log-sum-exp and linear least squares, $\nsamples=d=100$. For logistic regression, $\nsamples=100$ and $d=3$. For each objective, on each trial we generate $A, \theta^*$, and thus $f(\cdot)$ randomly and independently following the above procedure. 

For all convex objectives, use deterministic (full-batch) gradients to estimate the EGOP and to optimize. For all objectives, we use $M = 5d$ gradient samples to estimate the EGOP matrix. We take $\rho$ to be a standard Gaussian distribution, and use the same distribution to initialize when optimizing. See Section~\ref{ssec:algorithm-details} for choice of hyperparameters and choice of the range of learning rates to sweep for tuning. We used a constant learning rate for all convex objectives. We measure performance using the median final training loss achieved. For each algorithm, we choose the learning rate at which the algorithm in original achieved lowest median training loss. We perform 5 independent trials at each learning rate.

For each objective, we optimize for 1000 iterations, or until the gradient Euclidean norm drops below $1e-8$. The latter termination condition is comparable to default termination conditions employed by \texttt{scipy.optimize} \cite{2020SciPy-NMeth}.

\section{Expanded Discussion of Heuristics for Scalability}\label{sec:expanded-heuristics-discussion}
In this section, we expand on the heuristics proposed in Section~\ref{sec:efficient-heuristics}.

\paragraph{Block Reparameterization} Here we instantiate block reparameterization for multilayer neural networks. Given an $L$-layer neural network whose $\ell$th layer is parameterized by weight matrix $W_\ell\in \R^{\nin^{(\ell)}\times \nout^{(\ell)}}$ and bias vector $b_\ell\in \R^{\nout^{(\ell)}}$, we consider optimizing $f(\theta) = \textrm{loss}(\theta; X_{\textrm{train}}, y_{\textrm{train}})$ with respect to parameters
$ \theta = [(W_1, b_1),\dots,(W_L, b_L)].$

Rather than forming the full EGOP as described in Algorithm~\ref{alg:meta-algorithm-block}, we consider estimating the \textit{layer} EGOP matrices
\[
    \EGOP^{(\ell)} \defeq \frac{1}{M}\sum_{k=1}^M \nabla_{W_\ell} f(\theta_k)\nabla_{W_\ell} f(\theta_k)^\T
\]
where $\nabla_{W_\ell} f(\theta_k) \in \R^{\nin^{(\ell)}\nout^{(\ell)}}$ is the vector of partial derivatives of $f$ w.r.t. the entries of $W_\ell$ evaluated at $\theta_k$, and the points $\{\theta_k\}_{k=1}^M \sim \rho$ are drawn i.i.d.. 

Given these $L$ empirical layer EGOP matrices, one can obtain $L$ change-of-basis matrices $V^{(\ell)}\in \R^{\nin^{(\ell)}\nout^{(\ell)}\times \nin^{(\ell)}\nout^{(i)}}$, and form the objective
\[
    \tf(\theta) = f([(V^{(1)}W_1, b_1),\dots,(V^{(L)}W_L, b_L)]).
\]

Block EGOP reparameterization requires storing and applying $L$ orthonormal matrices, each of size $\nin^{(\ell)}\nout^{(\ell)}\times \nin^{(\ell)}\nout^{(\ell)}$. For deep neural networks, this can be considerably less expensive than storing and applying the global change-of-basis matrix, which would be of size $\sum_{\ell=1}^L \nin^{(\ell)}\nout^{(\ell)}\times \sum_{\ell=1}^L \nin^{(\ell)}\nout^{(\ell)}$. This may also reduce the sampling cost; each layer EGOP is of smaller dimension than the global EGOP and thus has a smaller number of eigenvectors to estimate. Thus the layer EGOP eigenbases may be estimated with fewer total gradient samples.

Another benefit to block EGOP reparameterization is that it offers an easy way to reduce cost by choosing to only reparameterize a subset of layers. Thus for example if a network has a subset of layers which are too wide to efficiently reparameterize, one can choose to only reparameterize the narrower layers. 

For networks where the first layer involves a linear transformation of the input data, reparameterizing only the first layer corresponds to an orthogonal transformation of the data. Thus pre-computing the EGOP eigenbasis and applying this matrix to the input data up-front would allow one to reparameterize the first layer, without requiring one to store and apply the change-of-basis matrix during training.

\subsection{Efficient Heuristics for Large-Scale Neural Networks}\label{appendix:large-neural-networks}

One major application area for adaptive gradient methods is the training of massive multi-layer neural networks. In this section, we present computational heuristics that enable EGOP-reparameterization to scale to massive architectures.

\subsubsection{Reparameterizing Massive Fully-Connected Layers}\label{sec:reduced-egop}

For large fully-connected layers, storing and applying the full reparameterization matrix $V \in \R^{d \times d}$ can be prohibitively expensive. We describe two memory-efficient variants that compute and retain only the top-$r$ eigenvectors of the EGOP matrix, forming a reduced basis $V_r \in \R^{d \times r}$ where $r \ll d$. The first approach constrains parameters to the subspace spanned by $V_r$, while the second uses auxiliary variables to maintain full expressivity while still storing only $V_r$.

\paragraph{Reduced-dimension EGOP} The primary motivation for reduced-dimension reparameterization is to reduce the memory cost of storing the reparameterization matrix $V$. Storing the full orthonormal basis $V \in \R^{d \times d}$ requires $O(d^2)$ memory, while storing only the top-$r$ eigenvectors $V_r \in \R^{d \times r}$ requires $O(dr)$ memory. Since the memory savings from reducing $V$ scale as $d^2 - dr = d(d-r)$ while the savings from reducing the parameter count scale as $d - r$, shrinking $V$ to $V_r$ provides a factor of $d$ larger benefit than shrinking $\theta$ to $\theta_r$. Additionally, computing only the top-$r$ eigenvectors via randomized SVD is significantly faster than computing the full eigendecomposition.

Rather than the full reparameterization $f(\theta) = f(V\theta)$ with $\theta \in \R^d$ and $V \in \R^{d \times d}$, reduced-dimension EGOP uses
$$
f(\theta) = f(V_r \theta_r)
$$
where $\theta_r \in \R^r$ and $V_r \in \R^{d \times r}$ contains the top-$r$ eigenvectors of the empirical EGOP matrix $\widehat{P}$. This constrains $\theta$ to lie in the $r$-dimensional subspace $\text{span}(V_r) \subset \R^d$.

\paragraph{Computing the reduced basis}
In practice, we compute $V_r$ by first sampling $M$ gradients $\{g_i\}_{i=1}^M$ at random initializations $\{\theta_i\}_{i=1}^M \sim \rho$, forming the gradient matrix $G \in \R^{d \times M}$ with columns $g_i$. We then apply randomized SVD to $G$ to obtain the top-$r$ left singular vectors, which form $V_r$, without explicitly forming the full $\widehat{P} = \frac{1}{M}GG^\T$ matrix. 

If the EGOP matrix $V$ has exact rank $r$, then $V_r$ spans the full column space of $V$ and we lose no expressivity by switching to $V_r$. However, in practice, the EGOP matrix is typically full-rank with rapidly decaying eigenvalues rather than having exact rank $r$. Thus reduced-dimension reparameterization necessarily introduces a projection error of $\|\theta - V_r V_r^\T \theta\| > 0$ for general $\theta \in \R^d$, limiting the expressivity of the reparameterized model to the $r$-dimensional subspace $\text{span}(V_r)$. We introduce auxiliary variables EGOP below to overcome this limitation.

\paragraph{Auxiliary variables EGOP}
Rather than completing the basis by appending a random orthonormal basis for $\text{span}(V_r)^\perp$, we introduce auxiliary parameters to represent the orthogonal complement. The key idea is to decompose parameters into two orthogonal components via the reparameterized objective:
$$
\tilde{f}(\theta_r, \theta_d) = f(V_r \theta_r + (I - V_r V_r^T) \theta_d)
$$
where $\theta_r \in \R^r$ represents coordinates in the dominant $r$-dimensional EGOP subspace, and $\theta_d \in \R^d$ are auxiliary variables whose projection $(I - V_r V_r^\T) \theta_d$ represents the residual component orthogonal to this subspace. The full parameter vector in $\R^d$ is thus reconstructed according to \eqref{eq:aux-reconstruction}.

\paragraph{Exact reconstruction} 
To initialize the auxiliary variables model equivalently to an original model at $\theta_0$, we set:
\begin{align*}
\theta_{r} &= V_r^T \theta_0 \in \R^r \\
\theta_{d} &= (I - V_r V_r^T) \theta_0 \in \R^d
\end{align*}
This initialization guarantees that the reparameterized model exactly reconstructs $\theta_0$:
$$
V_r \theta_{r} + (I - V_r V_r^T) \theta_{d} = \theta_0
$$
ensuring functional equivalence: $f(\theta_0) = \tilde{f}(\theta_{r}, \theta_{d})$ for all inputs.

\begin{algorithm}[H]
    \caption{EGOP Auxiliary Variables Linear Layer Forward Pass}\label{alg:aux-forward-optimized}
    \begin{algorithmic}[1]
        \State {\bfseries Input:} $x \in \R^{n \times d_{\text{in}}}$ (batch input), $\theta_r \in \R^r$ (reduced params), $\theta_d \in \R^d$ (auxiliary params), $V \in \R^{d \times r}$ (EGOP basis)
        \State {\bfseries Layer dims:} $d_{\text{out}}, d_{\text{in}}$ where $d = d_{\text{out}} \times d_{\text{in}}$
        \State \Comment{Compute $\theta = V_r \theta_r + (I - V_r V_r^\T) \theta_d$ using only 2 matrix multiplications}
        \State $\theta_{\text{full}} \gets \theta_d + V (\theta_r - V^\T \theta_d)$
        \State $W \gets \text{reshape}(\theta_{\text{full}}, (d_{\text{out}}, d_{\text{in}}))$ \Comment{Reshape to weight matrix}
        \State $y \gets x W^T + b$
        \State {\bfseries Return:} $y \in \R^{n \times d_{\text{out}}}$
    \end{algorithmic}
\end{algorithm}

\subsubsection{Reparameterizing Convolutional Network Layers}\label{ssec:CNN-sup-details}

To evaluate the scalability and generalization of EGOP reparameterization beyond fully connected architectures, we extend the method to convolutional neural networks (CNNs) and residual networks (ResNets). We study EGOP reparameterization on larger and deeper models such as AlexNet and ResNet trained on ImageNet, and compare performance against well tuned baseline optimizers. This extension allows us to assess whether EGOP can improve optimization in modern large scale visual recognition pipelines.

\paragraph{Background on convolutional layers}
A convolutional layer consists of a collection of learnable filters (also called kernels). Each filter is a three-dimensional tensor of shape $k\times k\times C_{\text{in}}$, where $k$ is the spatial size and $C_{\text{in}}$ is the number of input channels. A layer with $C_{\text{out}}$ output channels therefore contains $C_{\text{out}}$ such filters, each producing one output channel by convolving the filter with the input feature map. Stacking the $C_{\text{out}}$ resulting feature maps yields the layer’s activation tensor.

\paragraph{Motivation for EGOP in convolutional networks}
A direct application of EGOP to convolutional layers would require constructing and applying a separate reparameterization matrix for every filter, which is prohibitive for modern CNNs. However, filters within the same layer are drawn from the same initialization distribution, so their first-step gradient statistics are identically distributed; consequently, their EGOP matrices at early stages of training may also be similar. This motivates our shared reparameterization strategy: we compute a single EGOP matrix per convolutional layer and reuse it across all filters within that layer. This reduces computational cost while preserving the benefits of EGOP, enabling its use in deep architectures such as AlexNet and ResNet.

\paragraph{Shared Gradient Structure at Initialization} 
A key observation that enables an efficient extension of EGOP to CNNs is that, for most practical initialization distributions (e.g. Kaiming, Xavier, Glorot normal/uniform, filters within the same convolutional layer are identically distributed and so are their gradients. This property allows us to substantially reduce the computational cost of the reparameterization. During initialization, each forward/backward pass yields multiple gradients simultaneously rather than a single one, which decreases the number of required passes proportionally to the number of filters. Once these gradients are collected, we compute a single orthogonal matrix \( V \) that captures the shared structure of the layer. This \( V \) is then reused across all filters, leading to a dramatic reduction in memory footprint and computational speedup, due to the low number of required forward/backward passes required for gradient sampling.

\paragraph{Kronecker Product Structure from Reusing $V$} As referenced in Section~\ref{sec:efficient-heuristics}, reusing $V$ across all filters as described above is equivalent to employing a Kronecker-product reparameterization matrix. Moreover, the effective Kronecker product imposed by this construction is different from that employed by SOAP/Shampoo, and the effective EGOP change-of-basis matrix applied to each convolutional filter does \textit{not} have any structural constraint. We clarify this point below.

In the EGOP reparameterization procedure described in the preceeding paragraphs, for each convolutional layer we form $V \in \R^{(k^2 C_{\text{in}
})\times (k^2 C_{\text{in}
})}$ and apply this change-of-basis to each (vectorized) filter in the convolutional layer, denoted $\theta_j$ for $j = 1\dots C_{\text{out}}$, as
\[
    V \vecop(\theta_{j}).
\]
This is equivalent to employing the following Kronecker-product structured change-of-basis
\[
    (\mathbb{I}_{C_{\text{out}}} \otimes V) 
    \begin{bmatrix}
        \vecop(\theta_1)\\
        \vdots\\
        \vecop(\theta_{C_{\text{out}}})
    \end{bmatrix}.
\]

We now compare with the Kronecker-product structured change-of-basis employed by SOAP/Shampoo. For a convolutional layer with the above-described dimensions, SOAP and Shampoo form four orthogonal matrices
\[
    Q_h \in \R^{k\times k}, \quad Q_{w} \in \R^{k \times k}, \quad Q_{\text{in}} \in \R^{C_{\text{in}} \times C_{\text{in}}}, \quad Q_{\text{out}} \in \R^{C_{\text{out}} \times C_{\text{out}}}.
\]
The SOAP/Shampoo forward pass is then equivalent to performing
\[
    (Q_{\text{out}} \otimes Q_{w} \otimes Q_h  \otimes Q_{\text{in}}) \begin{bmatrix}
        \vecop(\theta_1)\\
        \vdots\\
        \vecop(\theta_{C_{\text{out}}})
    \end{bmatrix}.
\]
Comparing the transformation employed in EGOP reparameterization with that employed by SOAP/Shampoo, we see that the effective transformation applied to each individual filter by SOAP/Shampoo is constrained to have a Kronecker product structure, $Q_{w} \otimes Q_h  \otimes Q_{\text{in}}$, whereas the transformation $V$ employed in EGOP reparameterization is a generic matrix. SOAP/Shampoo employ a more general mapping over output channels, namely $Q_{\text{out}}$, compared to $\mathbb{I}_{C_{\text{out}}}$ employed in EGOP reparameterization. However, as explained in the preceding paragraphs, the distribution over different filters is identical at initialization, and in practice, we find employing the identity mapping over output channels in EGOP reparameterization yields performance competitive with that of SOAP/Shampoo.

\paragraph{Algorithmic Implementation}
The convolutional reparameterization follows the same principle as in the linear case, with additional tensor reshaping steps to accommodate convolutional filters. For each layer, we simultaneously collect gradients at initialization and flatten them into a single matrix. A full SVD is performed for each layer, and the resulting right singular vectors \(V\) define a shared orthogonal basis that captures the dominant gradient directions. Each filter is then flattened, preconditioned by \(V\), reshaped back to its convolutional form, and used in the standard forward pass. For full pseudocode, see \cref{alg:cnn-reparam-multiinit}.

\begin{algorithm}[h!]
\caption{EGOP Reparameterization for Convolutional Layers}
\label{alg:cnn-reparam-multiinit}
\begin{algorithmic}[1]
\State {\bfseries Input:} CNN forward map $f$ with convolutional layers 
$\{W^{(l)}\}_{l=1}^L$, number of initializations $K$, sampling distribution $\rho$
\State \Comment{1. Sample random initializations and collect gradients}
\State Sample $\{\theta_k\}_{k=1}^K \sim \rho$ i.i.d.
\For{$k = 1,\dots,K$}
    \State Initialize network parameters with $\theta_k$
    \State Perform one forward/backward pass
    \State Store gradients $\{\nabla W^{(l,k)}\}_{l=1}^L$
\EndFor

\State
\State \Comment{2. Form per-layer gradient matrices}
\For{each convolutional layer $l = 1,\dots,L$}
    \State \Comment{Flatten each kernel gradient $\nabla W^{(l,k)}_j \in 
    \R^{c_{\text{in}} \times k_h \times k_w}$ into a vector}
    \State \Comment{Stack across initializations $k$ and output channels $j$}
    \State $G^{(l)} \gets 
    \text{stack}\!\left(
    \text{flatten}(\nabla W^{(l,k)}_j)
    \right)_{k=1,\dots,K;\,j=1,\dots,c_{\text{out}}}$
    \Comment{$G^{(l)} \in \R^{(c_{\text{in}} k_h k_w) \times (K c_{\text{out}})}$}
\EndFor

\State
\State \Comment{3. Compute EGOP eigenbasis per layer}
\For{each layer $l = 1,\dots,L$}
    \State $V^{(l)} \gets \texttt{left\_singular\_vectors}(G^{(l)})$
\EndFor

\State
\State \Comment{4. Define reparameterized forward map}
\State Define $\tilde f(\{W^{(l)}_j\}_{j=1}^L) = 
f\!\left(
\left\{
\text{reshape}\!\left(
V^{(l)}\, \text{flatten}(W^{(l)}_j),
(c_{\text{in}}, k_h, k_w)
\right)
\right\}_{j=1}^L
\right)$

\State
\State \Comment{5. Optimize the reparameterized model}
\State Optimize $\tilde f$ using an adaptive optimizer of choice

\State {\bfseries Return:} Reparameterized forward map $\tilde f$
\end{algorithmic}
\end{algorithm}

\section{Kronecker Product Constraints in SOAP/Shampoo}\label{sec:compare-w-SOAP-Galore}

    In this section, we provide a detailed comparison of the change of basis employed by EGOP reparameterization and those performed by SOAP and Shmpoo \cite{vyas2024soap, gupta2018shampoo}. We instantiate an example of EGOP reparameterization for an objective whose parameters can be viewed as matrices, as this highlights some of the conceptual differences between the methods.

    In many key applications of adaptive algorithms, including training neural networks, the parameters over which optimization occurs are matrix-valued. As a simple illustration, consider optimizing a single-layer linear fully-connected network. Let $\nin$ denote the number of input features to the layer, and $\nout$ denote the number of output features. This network then has $\nin\cdot\nout$ parameters, which can be expressed either as a vector, $\theta\in \R^{\nin\nout}$, or as a matrix, denoted $\mat(\theta)\in\R^{\nin\times \nout}$.
    
    Denote the training data and labels by $A \in \R^{\nin\times\nsamples}$ and $Y \in \R^{\nout \times \nsamples}$, and consider minimizing loss function
    \[
        f(\theta) = \frac{1}{2}\frobnorm{\mat(\theta)^\T A - Y}^2.
    \]
    Similarly, the vector-valued gradients $\nabla f(\theta) \in \R^{\nin \nout}$ can also be viewed as matrices, $\mat(\nabla f(\theta)) \in \R^{\nin\times \nout}$.
    
    In the method proposed in this work, we consider gradients to be vector-valued when forming the EGOP. Thus for this single-layer objective, $\EGOP(f) \in \R^{\nin\nout \times \nin\nout}$. We emphasize this vector-view of gradients for clarity, because the EGOP matrix has distinct eigenvectors from the expectation of the matrix product, $\mat(\nabla f(\theta))\mat(\nabla f(\theta))^{\T}$. The related methods discussed in the introduction, including SOAP and GaLore, consider transformations by the matrices
    \begin{align*}
        Q_L &= \operatorname{eigenvectors}(\mathbb{E} [\mat\left(\nabla f(\theta)\right)\mat\left(\nabla f(\theta)\right)^{\T}])\\
        Q_R &= \operatorname{eigenvectors}(\mathbb{E} [\mat\left(\nabla f(\theta)\right)^{\T}\mat\left(\nabla f(\theta)\right)])
    \end{align*}
    and closely related transformations \cite{maes2024understanding,vyas2024soap, zhao2024galore}. In general the eigenvectors in $Q_L, Q_R$ correspond to different transformations than the eigenvectors of the EGOP. In particular, letting $V$ denote the eigenbasis of the EGOP formed from vector-valued gradients, $V\neq Q_L\otimes Q_R$. Moreover, the class of orthonormal matrices obtainable from the EGOP eigenbasis is strictly more general than the class $Q_L \otimes Q_R$ for pairs of orthogonal matrices $Q_L, Q_R$, as formalized below:
    \begin{lemma}
        For any orthogonal matrices $Q_L\in\R^{\nin\times \nin},Q_R\in\R^{\nout\times\nout}$ and any objective function $f:\R^{\nin\nout}\rightarrow \R$, there exists orthogonal $V\in \R^{\nin\nout\times\nin\nout}$ such that
        \[
            \forall \theta \in \R^{\nin\nout}, \quad Q_L^\T \mat(\nabla f(\theta)) Q_R = \mat(V\nabla f(\theta)).
        \]
        However, there exist values of $\nin, \nout$, objective functions $f:\R^{\nin\nout}\rightarrow \R$, and orthogonal matrices $V\in \R^{\nin\nout\times \nin\nout}$ such that no orthogonal matrices $Q_L\in\R^{\nin\times \nin},Q_R\in\R^{\nout\times\nout}$satisfy
        \[
            \forall \theta \in \R^{\nin\nout}, \quad Q_L^\T \mat(\nabla f(\theta)) Q_R = \mat(V\nabla f(\theta)).
        \]
    \end{lemma}
    \begin{proof}
        First we show that for any $Q_L, Q_R$, there exists suitable $V$ satisfying the property. We begin by noting that for any matrices $A, B, C$ of compatible dimension,
        \[
            \matvec(ABC) = (C^{\T} \otimes A) \matvec(B).
        \]
        Thus for any $Q_L, Q_R$,
        \[
            \matvec(Q_L^{\T} \mat(\nabla f(\theta)) Q_R) = (Q_R^\T \otimes Q_L^\T)\nabla f(\theta)
        \]
        where the second equality uses the fact that trivially $\matvec(\mat(\nabla f(\theta))) = \nabla f(\theta)$. The Kronecker product of orthogonal matrices is orthogonal, so choice of $V = (Q_R^\T \otimes Q_L^\T)$ satisfies the desired property.

        We now show there exist orthogonal matrices $V\in \R^{\nin\nout\times \nin\nout}$ such that no orthogonal matrices $Q_L\in\R^{\nin\times \nin},Q_R\in\R^{\nout\times\nout}$ satisfy
        \[
           Q_R \otimes Q_L^{\T} = V.
        \]
        This is a consequence of the fact that not all orthogonal matrices admit Kronecker product factorizations. Here we give one specific construction.
        
        Let $\vec{1}_d$ denote the all-ones vector in $\R^d$. Consider $V$ with leading column $$v_1 = \vec{1}_{\nin\nout}/\sqrt{\nin \nout}$$ and second column $$v_2 = (\nin\nout)^{-1/2}\cdot\operatorname{concatenate}(\vec{1}_{\nin\nout/2}, -\vec{1}_{\nin\nout/2})$$. Given the entries of $v_1$ are identical, this implies all the entries in the first column of $Q_R$ are identical, and thus that the first $\nin$ columns of $V$ comprise concatenated copies of $Q_L^T$. However this contradicts the fact that the entries of $v_2$ are identical in magnitude and nonzero but have positive sign for the first $\nin\nout/2$ entries and negative sign for the rest. Thus no orthogonal matrices $V$ with such first and second columns can be decomposed into $Q_R \otimes Q_L^T$, so in particular there exists some vector $z\in \R^{\nin\nout}$ such that
        \[
            Q_R \otimes Q_L^{\T} z \neq V z.
        \]
        Thus for any $f(\cdot)$ such that there exists $\theta$ with $\nabla f(\theta)=z$, it holds that for this value
        \[
            Q_L^{\T} \mat(\nabla f(\theta)) Q_R = \mat\left((Q_R \otimes Q_L^{\T})\nabla f(\theta)\right)\neq  \mat(V\nabla f(\theta)).
        \]
    \end{proof}

\section{Extended Discussion of Related Works}\label{sec:extended-related-works}

In this section, we expand on the overview of related work presented in Section~\ref{ssec:related-work}.

\paragraph{Geometric sensitivity of adaptive methods} A large body of research has been devoted to understanding the settings in which the per-coordinate learning rates of adaptive algorithms confer a benefit over more basic methods, such as (S)GD. Recently there has been renewed interest in distinguishing the properties of adaptive algorithms versus SGD because several empirical studies suggest that adaptive methods outperform SGD when training transformer models \cite{kunstner2024heavy,zhang2024transformers}. Traditional analyses of Adagrad in the context of online convex optimization establish regret bounds which can be either better or worse than those enjoyed by SGD by up to a factor of $\sqrt{d}$, for $d$ the number of problem parameters \cite{pmlr-v247-chen24e,duchi2011adaptive}. More recent research has studied the rates at which adaptive methods converge to stationary points. Several works study convergence guarantees, measured in terms of the $\ell_2$ norm of the gradient, on smooth non-convex objectives \cite{defossez2020simple,ward2020adagrad}. These results show that Adagrad, Adam and related variants achieve rates matching those enjoyed by SGD, while having the key distinction that Adagrad and its variants do not require a-priori knowledge of the objective function's Lipschitz constant \cite{defossez2020simple,ward2020adagrad}.

In order to shed light on the question of when adaptive algorithms enjoy provably stronger guarantees than SGD, a recent line of work studies convergence under more refined geometric assumptions, with particular emphasis on assumptions that are \textit{not} rotationally invariant \cite{jiang2024convergence, liu2024adagrad, xie2024adamexploitsellinftygeometryloss}. Xie et al. \cite{xie2024adamexploitsellinftygeometryloss} establish convergence guarantees in terms of the $L_\infty$ smoothness constant of the objective and also in terms of the related Hessian 1-norm. They show that rotationally invariant geometric assumptions do not suffice to capture settings when Adam out-performs SGD through experiments examining the sensitivity of Adam to orthonormal rotations \cite{xie2024adamexploitsellinftygeometryloss}. 

Jiang et al. \cite{jiang2024convergence} and Liu et al. \cite{liu2024adagrad} study convergence of Adagrad on objectives that satisfy coordinate-wise smoothness. Both works prove similar convergence guarantees for Adagrad in terms of $\ell_1$ gradient norm, rather than the $\ell_2$ gradient norm measure more widely studied; Jiang et al. \cite{jiang2024convergence} establish convergence guarantees showing that SGD remains worst-case optimal even in the setting of coordinate-wise smoothness when measuring stationarity with the $\ell_2$ gradient norm, motivating the need to measure stationarity with the $\ell_1$ norm in order to prove separation between Adagrad and SGD. Using this $\ell_1$ stationarity measure, Jiang et al. \cite{jiang2024convergence} establish provable separation between SGD's and Adagrad's convergence to stationary points of non-convex objectives. They show that when objective functions exhibit certain geometric properties, measured by their coordinate-wise smoothness constants, Adagrad's upper bounds are lower than SGD's lower bounds by a factor of $d$ \cite{jiang2024convergence}. Our analysis builds on that of Jiang et al. \cite{jiang2024convergence} and Liu et al. \cite{liu2024adagrad}, as a consequence of our results is that the EGOP reparameterization proposed in this work acts to transform objectives into the setting identified by Jiang et al. \cite{jiang2024convergence} where Adagrad's convergence guarantees compare most favorably with SGD's. 

Ling et al. \cite{ling2022vectoradam} develop a rotationally equivariant extension of Adam, termed VectorAdam, in an effort to reduce axis-aligned artifacts that arise when using adaptive methods to solve geometric optimization problems such as differentiable rendering. This method targets applications when the problem parameters $\theta\in \R^{r\cdot n}$ represent a collection of $r$ different vectors in $\R^n$. VectorAdam uses the squared gradient norm of each $n$-dimensional vector to rescale the learning rates of all entries in each of the $r$ vectors comprising $\theta$, making the algorithm equivariant to transformations of the form $Q\mat{\theta}$, where $\mat{\theta}\in \R^{n\times r}$ is the reshaping of $\theta\in \R^{r\cdot n}$ and $Q \in \R^{n\times n}$ is orthonormal \cite{ling2022vectoradam}.

\paragraph{Change-of-basis for Optimization}

Several recent works propose that when using Adam and its variants to train neural networks,  different choices of orthonormal transformation can reduce the computational cost associated with these algorithms and improve their performance \cite{maes2024understanding,vyas2024soap, zhao2024galore}. Gupta et al. \cite{gupta2018shampoo} introduced Shampoo, an efficient preconditioning method for optimizing tensor spaces. Vyas et al. \cite{vyas2024soap} formalized a connection between Shampoo and a variant of Adagrad, leading to a new proposed method called SOAP. Designed for training neural networks, SOAP computes an orthonormal reparameterization based on the singular vectors of the matrix-valued network gradients and performs optimization in this basis. Vyas et al. \cite{vyas2024soap} empirically examine the performance of SOAP and find that it outperforms both Adam and Shampoo in LLM pre-training. Zhao et al. \cite{zhao2024galore} propose GaLore, a method that simultaneously performs reparameterization and dimensionality reduction. GaLore computes a similar orthogonal basis to that used in SOAP, but instead of a full-dimensional change-of-basis GaLore retains only leading basis vectors in order to reduce dimension \cite{zhao2024galore}. Maes et al. \cite{maes2024understanding} empirically study Adam's rotational sensitivity and examine the power of existing geometric assumptions (such as those leveraged by Xie et al. \cite{xie2024adamexploitsellinftygeometryloss}) to explain Adam's performance when training transformer architectures. They propose an orthonormal reparameterization, similar to those used by SOAP and GaLore, and show empirically that this can improve Adam's performance \cite{maes2024understanding}.

SOAP and GaLore both call for periodically re-computing the change-of-basis matrices. In their experiments, the fewest re-computations performed by Vyas et al. \cite{vyas2024soap} was once every 100 batches, and they show that the performance gap between SOAP and Adam narrows as the number of batches between re-computations increases. In contrast, our we show that with our proposed choice of basis, a single up-front computation of the change-of-basis suffices to improve the performance of both Adam and Adagrad in a variety of settings, as shown in Section~\ref{sec:experimental-results}.

Outside of the domain neural network training, several works have considered data-driven methods for performing dimensionality reduction when optimizing objectives with low-rank EGOP matrices, including the works of Cartis et al. \cite{cartis2024learning} and Cosson et al. \cite{cosson2023low}. For functions with low-dimensional active subspaces, Cartis et al. \cite{cartis2024learning} study using the EGOP eigenbasis to reparameterize and reduce the dimension of optimization problems. They demonstrate that this approach yields computational speedups for loss functions whose EGOP matrix is low-rank. Cosson et al. \cite{cosson2023low} develop gradient descent algorithms which leverage the EGOP eigenbasis. Their method first estimates the EGOP and computes the corresponding leading $r$-eigenvectors, and then performs gradient descent steps along the directions of these $r$ vectors. In the setting of exactly low-rank functions, they prove convergence results illustrating that this approach improves the dimensional dependency of gradient descent.

\end{document}